\documentclass[11pt,oneside]{article}

\usepackage[numbers]{natbib}

\usepackage[utf8]{inputenc}
\usepackage[T1]{fontenc}
\usepackage{amsmath}
\usepackage{amssymb,amsfonts}
\usepackage{bm}

\usepackage[paper=a4paper, hscale=0.75, vscale=0.79]{geometry}

\usepackage[dvipsnames]{xcolor}

\usepackage{graphicx}
\graphicspath{{./}}
\DeclareGraphicsExtensions{.pdf,.png,.jpg,.eps}
\usepackage{subfig}

\newcommand*\dif{\mathop{}\!\mathrm{d}} 

\DeclareMathOperator{\atan}{atan}

\DeclareMathOperator{\exx}{\mathbf{E}}

\DeclareMathOperator{\prr}{\mathbf{P}}

\DeclareMathOperator{\vaa}{var}
\DeclareMathOperator*{\argmin}{arg\,min}

\def\FF{\mathcal{F}}

\def\HH{\mathcal{H}}

\def\LL{\mathcal{L}}

\def\OO{\mathcal{O}}

\def\RR{\mathbb{R}}

\def\VV{\mathcal{V}}

\def\XX{\mathcal{X}}
\def\YY{\mathcal{Y}}
\def\ZZ{\mathcal{Z}}

\newcommand{\overbar}[1]{\mkern 1.5mu\overline{\mkern-1.5mu#1\mkern-1.5mu}\mkern 1.5mu}

\newcommand*{\defeq}{\mathrel{\vcenter{\baselineskip0.5ex \lineskiplimit0pt     
                     \hbox{\scriptsize.}\hbox{\scriptsize.}}}
                     =}

\newcommand{\term}[1]{\textcolor{BlueViolet}{\textit{{#1}}}}

\def\clarke{\textup{C}} 
\def\cond{\,|\,} 
\def\ddist{\mu} 
\def\dev{\rho} 
\DeclareMathOperator{\dist}{dist} 
\DeclareMathOperator{\dom}{dom} 
\DeclareMathOperator{\env}{env} 
\def\frechet{\textup{F}} 
\DeclareMathOperator{\inter}{int} 
\DeclareMathOperator{\jrisk}{r} 
\DeclareMathOperator{\jsmooth}{\widetilde{r}} 
\DeclareMathOperator{\loss}{L} 
\DeclareMathOperator{\mpara}{M} 
\DeclareMathOperator{\proj}{\Pi} 
\DeclareMathOperator{\prox}{prox} 
\DeclareMathOperator{\risk}{R} 

\usepackage{amsthm}
\theoremstyle{definition} \newtheorem{defn}{Definition}
\theoremstyle{plain} \newtheorem{prop}[defn]{Proposition}
\theoremstyle{plain} \newtheorem{thm}[defn]{Theorem}
\theoremstyle{plain} \newtheorem{lem}[defn]{Lemma}
\theoremstyle{plain} 
\theoremstyle{remark} \newtheorem{rmk}[defn]{Remark}
\theoremstyle{remark} \newtheorem{ex}[defn]{Example}

\usepackage[pdfauthor={Matthew J. Holland},%
colorlinks=true,%
linkcolor=blue,%
citecolor=blue]{hyperref}
\hypersetup{pdftitle={Learning with risks based on M-location}}

\usepackage{url}
\usepackage{booktabs}
\usepackage{amsfonts}
\usepackage{microtype}
\usepackage{algorithm}
\usepackage{algpseudocode}
\usepackage{enumitem}
\makeatletter
\def\namedlabel#1#2{\begingroup
    #2%
    \def\@currentlabel{#2}%
    \phantomsection\label{#1}\endgroup
}
\makeatother

\title{\textbf{Learning with risks based on M-location}}
\author{
  Matthew J.~Holland\thanks{Please direct correspondence to \texttt{matthew-h@ar.sanken.osaka-u.ac.jp}.}\\
  Osaka University
}
\date{} 

\begin{document}

\maketitle

\begin{abstract}
In this work, we study a new class of risks defined in terms of the location and deviation of the loss distribution, generalizing far beyond classical mean-variance risk functions. The class is easily implemented as a wrapper around any smooth loss, it admits finite-sample stationarity guarantees for stochastic gradient methods, it is straightforward to interpret and adjust, with close links to M-estimators of the loss location, and has a salient effect on the test loss distribution.
\end{abstract}

\tableofcontents

\section{Introduction}\label{sec:intro}

In machine learning, the important yet ambiguous notion of ``good off-sample generalization'' (or ``small test error'') is typically formalized in terms of minimizing the expected value of a random loss $\exx_{\ddist}\loss(h)$, where $h$ is a candidate decision rule and $\loss(h)$ is a random variable on an underlying probability space $(\Omega,\FF,\ddist)$. This setup based on \emph{average} off-sample performance has been famously called the ``general setting of the learning problem'' by \citet{vapnik1999NSLT}, and is central to the decision-theoretic formulation of learning in the influential work of \citet{haussler1992a}. This is by no means a purely theoretical concern; when average performance dictates the ultimate objective of learning, the data-driven \emph{feedback} used for training in practice will naturally be designed to prioritize average performance \citep{bottou2016a,johnson2014a,leroux2013a,shalev2013a}. Take the default optimizers in popular software libraries such as PyTorch or TensorFlow; virtually without exception, these methods amount to efficient implementations of empirical risk minimization. While the minimal expected loss formulation is clearly an intuitive choice, the tacit emphasis on average performance represents an important and non-trivial value judgment, which may or may not be appropriate for any given real-world learning task.

To make this value judgment an explicit part of the machine learning workflow, in this work we consider a generalized class of risk functions, designed to give the user much greater flexibility in terms of how they choose to evaluate performance, while still allowing for theoretical performance guarantees. One core statistical concept is that of the \term{M-location} of the loss distribution under a candidate $h$, defined by
\begin{align}
\mpara(h) \defeq \argmin_{\theta \in \RR} \, \exx_{\ddist}\dev\left(\frac{\loss(h)-\theta}{\sigma}\right).
\end{align}
Here $\dev:\RR \to [0,\infty)$ is a function controlling how we measure deviations, and $\sigma > 0$ is a scaling parameter. Since the loss distribution $\ddist$ is unknown, clearly $\mpara(h)$ is an ideal, unobservable quantity. If we replace $\ddist$ with the empirical distribution induced by a sample $(\loss_1,\ldots,\loss_n)$, then for certain special cases of $\dev$ we get an \term{M-estimator} of the location of $\loss(h)$, a classical notion dating back to \citet{huber1964a}, which justifies our naming. Ignoring integrability concerns for the moment, note that in the special case of $\dev(u) = u^{2}$, we get the classical risk $\mpara(h) = \exx_{\ddist}\loss(h)$, and in the case of $\dev(u) = |u|$, we get $\mpara(h) = \inf\{u: \ddist\{ \loss(h) \leq u \} \geq 0.5 \}$, namely the median of the loss distribution. This rich spectrum of evaluation metrics makes the notion of casting learning problems in terms of minimization of M-locations (via their corresponding M-estimators) very conceptually appealing. However, while the statistical properties of the minima of M-estimators in special cases are understood \citep{brownlees2015a}, the optimization involved is both costly and difficult, making the task of designing and studying $\mpara(\cdot)$-minimizing learning \emph{algorithms} highly intractable. With these issues in mind, we study a closely related alternative which retains the conceptual appeal of raw M-locations, but is more computationally congenial.

\paragraph{Our approach}
With $\sigma$ and $\dev$ as before, our generalized risks will be defined implicitly by
\begin{align}\label{eqn:mrisk_intuitive}
\min_{\theta \in \RR} \left[ \theta + \eta\exx_{\ddist}\dev\left(\frac{\loss(h)-\theta}{\sigma}\right) \right],
\end{align}
where $\eta > 0$ is a weighting parameter that controls the balance of priority between location and deviation. A more formal definition will be given in section \ref{sec:mrisk} (see equations (\ref{eqn:dev_handy})--(\ref{eqn:mrisk_defn_general})), including concrete forms for $\dev(\cdot)$ that are conducive to both fast computation and meaningful learning guarantees. In addition, we will see (cf.~Proposition \ref{prop:risk_props_axioms}) that after minimization, this objective can be written as
\begin{align*}
[\mpara(h) - c_{\mpara}] + \eta \exx_{\ddist}\dev\left(\frac{\loss(h)-[\mpara(h) - c_{\mpara}]}{\sigma}\right),
\end{align*}
where the shift term $c_{\mpara} > 0$ can be simply characterized by the equality
\begin{align*}
\exx_{\ddist}\dev^{\prime}\left(\frac{\loss(h)-[\mpara(h) - c_{\mpara}]}{\sigma}\right) = \frac{\sigma}{\eta},
\end{align*}
noting that $c_{\mpara} \to 0_{+}$ as $\eta \to \infty$. By utilizing smoothness properties of loss functions typical to machine learning problems (e.g., squared error, cross-entropy, etc.), even though the generalized risks need not be convex, they can be shown to satisfy weak notions of convexity, which still admit finite-sample guarantees of near-stationary for stochastic gradient-based learning algorithms (details in section \ref{sec:learning}). This approach has the additional benefit that implementation only requires a single wrapper around any given loss which can be set prior to training, making for easy integration with frameworks such as PyTorch and TensorFlow, while incurring negligible computational overhead.

\paragraph{Our contributions}

The key contribution here is a new concrete class of risk functions, defined and analyzed in section \ref{sec:mrisk}. These risks are statistically easy to interpret, their empirical counterparts are simple to implement in practice, and as we prove in section \ref{sec:learning}, their design allows for standard stochastic gradient-based algorithms to be given competitive finite-sample excess risk guarantees. We also verify empirically (section \ref{sec:empirical}) that the proposed feedback generation scheme has a demonstrable effect on the test distribution of the loss, which as a side-effect can be easily leveraged to outperform traditional ERM implementations, a result which is in line with early insights from \citet{breiman1999a} and \citet{reyzin2006a} regarding the impact of the loss distribution on generalization. More broadly, the choice of \emph{which risk to use} plays a central role in pursuit of increased transparency in machine learning, and our results represent a first step towards formalizing this process.

\paragraph{Relation to existing literature}

With respect to alternative notions of ``risk'' in machine learning, perhaps the most salient example is conditional value-at-risk (CVaR) \citep{prashanth2020a,holland2021c}, namely the expected loss conditioned on it exceeding a quantile at a pre-specified probability level. CVaR allows one to encode a sensitivity to extremely large losses, and admits convexity when the underlying loss is convex, though the conditioning often leaves the effective sample size very small. Other notions such as cumulative prospect theory (CPT) scores have also been considered \citep{bhat2020a,leqi2019a,khim2020a}, but the technical difficulties involved with computation and analysis arguably outweigh the conceptual benefits of learning using such scores. These proposals can all be interpreted as ``location'' parameters of the underlying loss distribution; our risks take the form of a sum of a location and a deviation term, where the location is a shifted M-location, as described above. The basic notion of combining location and deviation information in evaluation is a familiar concept; the mean-variance objective $\exx_{\ddist}\loss(\cdot)+\vaa_{\ddist}\loss(\cdot)$ dates back to classical work by \citet{markowitz1952a}; our proposed class includes this as a special case, but generalizes far beyond it. Mean-variance and other risk function classes are studied by \citet{ruszczynski2006a,ruszczynski2006b}, who give minimizers a useful dual characterization, though our proposed class is not captured by their work (see also Remark \ref{rmk:non_monotone}). We note also that the recent (and independent) work of \citet{lee2020a} considers a form which is similar to (\ref{eqn:mrisk_intuitive}) in the context of empirical risk minimizers; the critical technical difference is that their formulation is restricted to $\dev$ which is monotonic, an assumption which enforces convexity. The special case of mean-variance is also treated in depth in more recent work by \citet{duchi2019a}, who consider stochastic learning algorithms for doing empirical risk minimization with variance-based regularization. Finally, we note that our technical analysis in section \ref{sec:learning} makes crucial use of weak convexity properties of function compositions, an area of active research in recent years \citep{duchi2018a,drusvyatskiy2019a,davis2019b}. Since our proposed objective can be naturally cast as a composition taking us from parameter space to a Banach space and finally to $\RR$, leveraging the insights of these previous works, we extend the existing machinery to handle learning over Banach spaces, and give finite-sample guarantees for arbitrary Hilbert spaces. More details are provided in section \ref{sec:learning}, plus the appendix.

\paragraph{Notation and terminology}

To give the reader an approximate idea of the technical level of this paper, we assume some familiarity with probability spaces, the notion of sub-gradients and the sub-differential of convex functions, as well as special classes of vector spaces like Banach and Hilbert spaces, although the main text is written with a wide audience in mind. Strictly speaking, we will also deal with sub-differentials of non-convex functions, but these technical concepts are relegated to the appendix, where all formal proofs are given. In the main text, to improve readability, we write $\partial f(x)$ to denote the sub-differential of $f$ at $x$, regardless of the convexity of $f$. When we refer to a function being $\lambda$-\term{smooth}, this refers to its gradient being $\lambda$-Lipschitz continuous, and \emph{weak} smoothness just requires such continuity on directional derivatives; all these concepts are given a detailed introduction in the appendix. Throughout this paper, we use $\exx[\cdot]$ for taking expectation, and $\prr$ as a general-purpose probability function. For indexing, we will write $[k] \defeq \{1,\ldots,k\}$. Distance of a vector $v$ from a set $A$ will be denoted by $\dist(v;A) \defeq \inf\{\|v-v^{\prime}\|: v^{\prime} \in A \}$.

\section{A concrete class of risk functions}\label{sec:mrisk}

The risks described by (\ref{eqn:mrisk_intuitive}) are fairly intuitive as-is, but a bit more structure is needed to ensure they are well-defined and useful in practice. To make things more concrete, let us fix $\dev$ as
\begin{align}\label{eqn:dev_handy}
\dev(u) \defeq u \atan(u) - \frac{\log(1+u^2)}{2}, \quad u \in \RR.
\end{align}
This function is handy in that it behaves approximately quadratically around zero, and it is both $\pi/2$-Lipschitz and strictly convex on the real line.\footnote{To see this, note that $\dev^{\prime}(u) = \atan(u)$ and $\dev^{\prime}(\RR) \subset (-\pi/2,\pi/2)$, and $\dev^{\prime\prime}(u) = 1/(1+u^2) > 0$.} Fixing this particular choice of $\dev$ and letting $Z$ be a random variable (any $\FF$-measurable function), we interpolate between mean- and median-centric quantities via the following class of functions, indexed by $\sigma \in [0,\infty]$:
\begin{align}\label{eqn:jrisk_defn}
\jrisk_{\sigma}(Z,\theta) \defeq \theta + \eta \exx_{\ddist} \dev_{\sigma}\left( Z-\theta \right), \text{ where } \dev_{\sigma}(u) \defeq
\begin{cases}
\left|u\right|, & \text{if } \sigma = 0\\
\dev\left(u/\sigma\right), & \text{if } 0 < \sigma < \infty\\
u^2, & \text{if } \sigma = \infty.
\end{cases}
\end{align}
With this class of ancillary functions in hand, it is natural to define
\begin{align}\label{eqn:mrisk_defn_general}
\risk_{\sigma}(Z) \defeq \inf_{\theta \in \RR} \, \jrisk_{\sigma}(Z,\theta)
\end{align}
to construct a class of \term{risk functions}. In the context of learning, we will use this risk function to derive a generalized \term{risk}, namely the composite function $h \mapsto \risk_{\sigma}(\loss(h))$. As a special case, clearly this includes risks of the form (\ref{eqn:mrisk_intuitive}) given earlier. Visualizations of these functions are given in the supplementary appendix. Minimizing $\risk_{\sigma}(\loss(h))$ in $h$ is our formal learning problem of interest.

Before considering learning algorithms, we briefly cover the basic properties of the functions $\jrisk_{\sigma}$ and $\risk_{\sigma}$. Without restricting ourselves to the specialized context of ``losses,'' note that if $Z$ is any square-$\ddist$-integrable random variable, this immediately implies that $|\jrisk_{\sigma}(Z,\theta)| < \infty$ for all $\theta \in \RR$, and thus $\risk_{\sigma}(Z) < \infty$. Furthermore, the following result shows that it is straightforward to set the weight $\eta$ to ensure $\risk_{\sigma}(Z) > -\infty$ also holds, and a minimum exists.
\begin{prop}[Well-defined risks]\label{prop:well_defined}
Assuming that $\exx_{\ddist}Z^2 < \infty$, set $\eta$ based on $\sigma \in [0,\infty]$ as follows: if $\sigma = 0$, take $\eta > 1$; if $0 < \sigma < \infty$, take $\eta > 2\sigma/\pi$; if $\sigma = \infty$, take any $\eta > 0$. Under these settings, the function $\theta \mapsto \jrisk_{\sigma}(Z,\theta)$ is bounded below and takes its minimum on $\RR$. Thus, for each square-$\ddist$-integrable $Z$, there always exists a (non-random) $\theta_{Z} \in \RR$ such that
\begin{align}\label{eqn:mrisk_loc_dev}
\risk_{\sigma}(Z) = \theta_{Z} + \eta \exx_{\ddist}\dev_{\sigma}\left(Z-\theta_{Z}\right).
\end{align}
Furthermore, when $\sigma > 0$, this minimum $\theta_{Z}$ is unique.
\end{prop}
\begin{rmk}[Mean-median interpolation]\label{rmk:mean_median}
In order to ensure that risk modulation via $\sigma \in [0,\infty]$ smoothly transitions from a median-centric ($\sigma = 0$ case) to a mean-centric ($\sigma = \infty$ case) location, the parameter $\eta$ plays a key role. Noting that for any $u \in \RR$, for $\dev$ defined by (\ref{eqn:dev_handy}) we have $2\sigma^2 \dev(u/\sigma) \to u^2$ as $\sigma \to \infty$, and thus for large values of $\sigma > 0$ it is natural to set $\eta = 2\sigma^2$. On the other end of the spectrum, since $\sigma \log(1+(u/\sigma)^2) \to 0_{+}$ whenever $\sigma \to 0_{+}$, it is thus natural to set $\eta = \sigma/\atan(\infty) = 2\sigma/\pi$ when $\sigma > 0$ is small. Strictly speaking, in light of the conditions in Proposition \ref{prop:well_defined}, to ensure $\risk_{\sigma}$ is finite one should take $\eta > 2\sigma/\pi$.
\end{rmk}
What can we say about our risk functions $\risk_{\sigma}$ in terms of more traditional statistical risk properties? The form of $\risk_{\sigma}$ given in (\ref{eqn:mrisk_loc_dev}) has a simple interpretation as a weighted sum of ``location'' and ``deviation'' terms. In the statistical risk literature, the seminal work of \citet{artzner1999a} gives an axiomatic characterization of location-based risk functions that can be considered \term{coherent}, while \citet{rockafellar2006a} characterize functions which capture the intuitive notion of ``deviation,'' and establish a lucid relationship between coherent risks and their deviation class. The following result describes key properties of the proposed risk functions, in particular highlighting the fact that while our location terms are monotonic, our risk functions are non-traditional in that they are non-monotonic.
\begin{prop}[Non-monotonic risk functions]\label{prop:risk_props_axioms}
Let $\ZZ$ be a Banach space of square-$\ddist$-integrable functions. For any $\sigma \in [0,\infty]$, let $\eta > 0$ be set as in Proposition \ref{prop:well_defined}. Then, the functions $\jrisk_{\sigma}:\ZZ \times \RR \to \RR$ and $\risk_{\sigma}:\ZZ \to \RR$ satisfy the following properties:
\begin{itemize}
\item Both $\jrisk_{\sigma}$ and $\risk_{\sigma}$ are continuous, convex, and sub-differentiable.

\item The location in (\ref{eqn:mrisk_loc_dev}) is monotonic (i.e., $Z_1 \leq Z_2$ implies $\theta_{Z_1} \leq \theta_{Z_2}$) and translation-equivariant (i.e., $\theta_{Z+a} = \theta_{Z} + a$ for any $a \in \RR$), for any $0 < \sigma \leq \infty$.

\item The deviation in (\ref{eqn:mrisk_loc_dev}) is non-negative and translation-invariant, namely for any $a \in \RR$, we have $\exx_{\ddist}\dev_{\sigma}(Z+a-\theta_{Z+a}) = \exx_{\ddist}\dev_{\sigma}(Z-\theta_{Z})$, for any $0 < \sigma \leq \infty$.

\item The risk $\risk_{\sigma}$ is not monotonic, i.e., $\ddist\{Z_1 \leq Z_2\}=1$ need not imply $\risk_{\sigma}(Z_1) \leq \risk_{\sigma}(Z_2)$.
\end{itemize}
In particular, the risk $h \mapsto \risk_{\sigma}(\loss(h))$ need not be convex, even if $\loss(\cdot)$ is.
\end{prop}
\begin{rmk}\label{rmk:non_monotone}
Since our risk function $\risk_{\sigma}$ is not monotonic, standard results in the literature on optimizing generalized risks do not apply here. We remark that our proposed risk class does not appear among the comprehensive list of examples given in the works of \citet{ruszczynski2006a,ruszczynski2006b}, aside from the special case of $\sigma = \infty$ with $\eta = 1$. While the continuity and sub-differentiability of any risk function which is convex \emph{and} monotonic is well-known for a large class of Banach spaces \citep[Sec.~3]{ruszczynski2006a}, in Proposition \ref{prop:risk_props_axioms} we obtain such properties without monotonicity by using square-$\ddist$-integrability combined with properties of our function class $\dev_{\sigma}$.
\end{rmk}
Since our principal interest is the case where $Z = \loss(h)$, the key takeaways from this section are that while the proposed risk $h \mapsto \risk_{\sigma}(\loss(h))$ is well-defined and easy to estimate given a random sample $\loss_1(h),\ldots,\loss_n(h)$, the learning task is non-trivial since $\risk_{\sigma}(\loss(\cdot))$ is not differentiable (and thus non-smooth) when $\sigma = 0$, and for any $\sigma \in [0,\infty]$ need not be convex, even when the underlying loss is both smooth and convex. Fortunately, smoothness properties of the losses typically used in machine learning can be leveraged to overcome these technical barriers, opening a path towards learning guarantees for practical algorithms. This is the topic of the next section.

\section{Learning algorithm analysis}\label{sec:learning}

\begin{algorithm}[t!]
\caption{Projected sub-gradient method with randomized output.}
\label{algo:sgd}
\begin{algorithmic}
\State \textbf{inputs:} initial point $(h_0,\theta_0) \in C \subset \HH \times \RR$, step sizes $(\alpha_t)$, and max iterations $n$.
\medskip
\For{$t \in \{0,\ldots,n-1\}$}
 \State Sample $G_t$ via (\ref{eqn:sgd_feedback}).
 \State Update $(h_t,\theta_t) \mapsto (h_{t+1},\theta_{t+1})$ via (\ref{eqn:sgd_routine}).
\EndFor
\medskip
\State Sample $T \in \{0,\ldots,n-1\}$ with probabilities $\prr\{T=t\} = \alpha_{t}/(\sum_{k=0}^{n-1}\alpha_{k})$, $t \in [n-1]$.
\medskip
\State \textbf{return:} $\displaystyle (\overbar{h}_{[n]},\overbar{\theta}_{[n]}) \defeq (h_{T},\theta_{T})$.
\end{algorithmic}
\end{algorithm}

Thus far, we have only been concerned with \emph{ideal} quantities $\risk_{\sigma}$ and $\jrisk_{\sigma}$ used to define the ultimate formal goal of learning. In practice, the learner will only have access to noisy, incomplete information. In this work, we focus on iterative algorithms based on stochastic gradients, largely motivated by their practical utility and ubiquity in modern machine learning applications. For the rest of the paper, we overload our risk definitions to enhance readability, writing $\jrisk_{\sigma}(h,\theta) \defeq \jrisk_{\sigma}(\loss(h),\theta)$ and $\risk_{\sigma}(h) \defeq \risk_{\sigma}(\loss(h))$. First note that we can break down the underlying joint objective as $\jrisk_{\sigma}(h,\theta) = \exx_{\ddist}(f_2 \circ F_1)(h,\theta)$, where we have defined
\begin{align}\label{eqn:jrisk_composition}
F_1(h,\theta) \defeq (\loss(h),\theta), \qquad f_2(u,\theta) \defeq \theta + \eta \dev_{\sigma}(u-\theta).
\end{align}
From the point of view of the probability space $(\Omega,\FF,\ddist)$, the function $F_1$ is random, whereas $f_2$ is deterministic; our use of upper- and lower-case letters is just meant to emphasize this. Given some initial value $(h_0,\theta_0) \in \HH \times \RR$, one naively hopes to construct an efficient stochastic gradient algorithm using the update
\begin{align}\label{eqn:sgd_routine}
(h_{t+1},\theta_{t+1}) & = \proj_{C} \left[ (h_t,\theta_t) - \alpha_{t} G_t \right],
\end{align}
where $\alpha_t > 0$ is a step-size parameter, $\proj_{C}[\cdot]$ denotes projection to some set $C \subset \HH \times \RR$, and the stochastic feedback $G_t$ is just a composition of sub-gradients, namely
\begin{align}\label{eqn:sgd_feedback}
G_t & \in \partial f_2(\loss(h_t),\theta_t) \circ \partial F_1(h_t,\theta_t).
\end{align}
We call this approach ``naive'' since it is exactly what we would do if we knew \textit{a priori} that the underlying objective was convex and/or smooth.\footnote{The convex and smooth case is the classical setting \citep{nemirovsky1983a}; see \citet{shalev2014a} for a modern textbook introduction. When convex but non-smooth, see \citet{shamir2013a}. When smooth but non-convex, see \citet{ghadimi2013a}.} The precise learning algorithm studied here is summarized in Algorithm \ref{algo:sgd}. Fortunately, as we describe below, this naive procedure actually enjoys lucid non-asymptotic guarantees, on par with the smooth case.

\paragraph{How to measure algorithm performance?}

Before stating any formal results, we briefly discuss the means by which we evaluate learning algorithm performance. Since the sequence $(\risk_{\sigma}(h_t))$ cannot be controlled in general, a more tractable problem is that of finding a \term{stationary point} of $\jrisk_{\sigma}$, namely any $(h^{\ast},\theta^{\ast})$ such that $0 \in \partial \jrisk_{\sigma}(h^{\ast},\theta^{\ast})$. However, it is not practical to analyze $\dist(0;\partial \jrisk_{\sigma}(h_t,\theta_t))$ directly, due to a lack of continuity. Instead, we consider a smoothed version of $\jrisk_{\sigma}$:
\begin{align}
\jsmooth_{\sigma,\beta}(h,\theta) \defeq \inf_{h^{\prime},\theta^{\prime}} \left[ \jrisk_{\sigma}(h^{\prime},\theta^{\prime}) + \frac{1}{2\beta}\|(h,\theta)-(h^{\prime},\theta^{\prime})\|^{2} \right].
\end{align}
This is none other than the \term{Moreau envelope} of $\jrisk_{\sigma}$, with weighting parameter $\beta > 0$. A familiar concept from convex analysis on Hilbert spaces \citep[Ch.~12 and 24]{bauschke2017CMH}, the Moreau envelope of non-smooth functions satisfying weak convexity properties has recently been shown to be a very useful metric for evaluating stochastic optimizers \citep{davis2019b,drusvyatskiy2019a}. Our basic performance guarantees will first be stated in terms of the gradient of the smoothed function $\jsmooth_{\sigma,\beta}$. We will then relate this to the joint risk $\jrisk_{\sigma}$ and subsequently the risk $\risk_{\sigma}$.

\paragraph{Guarantees based on joint risk minimization}

Within the context of the stochastic updates characterized by (\ref{eqn:sgd_routine})--(\ref{eqn:sgd_feedback}), we consider the case in which $\HH$ is any Hilbert space. All Hilbert spaces are reflexive Banach spaces, and the stochastic sub-gradient $G_t \in (\HH \times \RR)^{\ast}$ (the dual of $\HH \times \RR$) can be uniquely identified with an element of $\HH \times \RR$, for which we use the same notation $G_t$. Denoting the partial sequence $G_{[t]} \defeq (G_0,\ldots,G_t)$, we formalize our assumptions as follows:
\begin{itemize}
\item[\namedlabel{asmp:background}{\textup{A1}}.] For all $h \in \HH$, the random loss $\loss(h)$ is square-$\ddist$-integrable, locally Lipschitz, and weakly $\lambda$-smooth, with a gradient satisfying $\exx_{\ddist}|\loss^{\prime}(h;\cdot)|^{2} < \infty$.

\item[\namedlabel{asmp:hilbert}{\textup{A2}}.] $\HH$ is a Hilbert space, and $C \subset \HH \times \RR$ is a closed convex set.

\item[\namedlabel{asmp:feedback_cond_indep}{\textup{A3}}.] The feedback (\ref{eqn:sgd_feedback}) satisfies $\exx[G_t \cond G_{[t-1]}] = \exx_{\ddist}G_t$ for all $t > 0$.\footnote{The expectation on the left-hand side is with respect to the joint distribution of $G_{[t]}$ conditioned on $G_{[t-1]}$.}

\item[\namedlabel{asmp:feedback_mnt_bd}{\textup{A4}}.] For some $0 < \kappa < \infty$, the second moments are bounded as $\exx_{\ddist}\|G_t\|^{2} \leq \kappa^{2}$ for all $t$.
\end{itemize}
The following is a performance guarantee for Algorithm \ref{algo:sgd} in terms of the smoothed joint risk.
\begin{thm}[Nearly-stationary point of smoothed objective]\label{thm:learn_stationary}
If $0 < \sigma < \infty$, set smoothing parameter $\gamma = (1+\eta\pi/(2\sigma))\max\{1,\lambda\}$. Otherwise, if $\sigma = 0$, set $\gamma = (1+\eta)\max\{1,\lambda\}$. Under these $\sigma$-dependent settings and assumptions \ref{asmp:background}--\ref{asmp:feedback_mnt_bd}, let $(\overbar{h}_{[n]},\overbar{\theta}_{[n]})$ denote the output of Algorithm \ref{algo:sgd}, with $\jrisk_{C}^{\ast} \defeq \inf\{\jrisk_{\sigma}(h,\theta): (h,\theta) \in C\}$ denoting the minimum over the feasible set and $\Delta_0 \defeq \jsmooth_{\sigma,\beta}(h_0,\theta_0) - \jrisk_{C}^{\ast}$ denoting the initialization error. Then, for any choice of $n>1$, $\eta > 0$, and $\beta < 1/\gamma$, we have that
\begin{align*}
\exx\|\jsmooth_{\sigma,\beta}^{\prime}(\overbar{h}_{[n]},\overbar{\theta}_{[n]})\|^{2} \leq \left(\frac{1}{1-\beta\gamma}\right) \frac{\Delta_0 + \gamma\kappa^{2}\sum_{t=0}^{n-1}\alpha_{t}^{2}/2}{\sum_{t=0}^{n-1}\alpha_t},
\end{align*}
where expectation is taken over all the feedback $G_{[n-1]}$.
\end{thm}
\begin{rmk}[Sample complexity]\label{rmk:sample_complexity}
Let us briefly describe a direct take-away from Theorem \ref{thm:learn_stationary}. If $\Delta_0$, $\gamma$, and $\kappa$ are known (upper bounds will of course suffice), then constructing step sizes as $\alpha_t^2 \geq \Delta_0/(n\gamma\kappa^2)$, if we set $\beta = 1/(2\gamma)$, it follows immediately that
\begin{align*}
\exx\|\jsmooth_{\sigma,\beta}^{\prime}(\overbar{h}_{[n]},\overbar{\theta}_{[n]})\|^2 \leq \sqrt{\frac{2\gamma\kappa^{2}\Delta_0}{n}}.
\end{align*}
Fixing some desired precision level of $\sqrt{\exx\|\jsmooth_{\sigma,\beta}^{\prime}(\overbar{h}_{[n]},\overbar{\theta}_{[n]})\|^2} \leq \varepsilon$, the sample complexity is $\OO(\varepsilon^{-4})$. This matches guarantees available in the smooth (but non-convex) case \citep{ghadimi2013a}, and suggests that the ``naive'' strategy implemented by Algorithm \ref{algo:sgd} in fact comes with a clear theoretical justification.
\end{rmk}

\paragraph{Implications in terms of the original objective}

The results described in Theorem \ref{thm:learn_stationary} and Remark \ref{rmk:sample_complexity} are with respect to a smoothed version of the joint risk function $\jrisk_{\sigma}$. Linking these facts to insights in terms of the original proposed risk $\risk_{\sigma}$ can be done as follows. Assuming we take $n \geq 2\gamma\kappa^{2}\Delta_0/\varepsilon^{4}$ to achieve the $\varepsilon$-precision discussed in Remark \ref{rmk:sample_complexity}, the immediate conclusion is that  the algorithm output is $(\varepsilon/2\gamma)$-close to a $\varepsilon$-nearly stationary point of $\jrisk_{\sigma}$. More precisely, we have that there exists an ideal point $(h_{n}^{\ast},\theta_{n}^{\ast})$ such that
\begin{align}
\exx\left[ \dist\left(0;\partial\jrisk_{\sigma}(h_{n}^{\ast},\theta_{n}^{\ast})\right) \right] \leq \varepsilon, \text{ and } \exx\left\| (\overbar{h}_{[n]},\overbar{\theta}_{[n]})-(h_{n}^{\ast},\theta_{n}^{\ast}) \right\| = \frac{\varepsilon}{2\gamma}.
\end{align}
The above fact follows from basic properties of the Moreau envelope (cf.~Appendix \ref{sec:bg_prox}). These non-asymptotic guarantees of being close to a ``good'' point extend to the function values of the risk $\risk_{\sigma}$ since we are close to a candidate $h_{n}^{\ast}$ whose risk value can be no worse than
\begin{align*}
\exx\left[\risk_{\sigma}(h_{n}^{\ast})\right] \leq \exx\left[\jrisk_{\sigma}(h_{n}^{\ast},\theta_{n}^{\ast})\right] \leq \exx\left[\jrisk_{\sigma}(\overbar{h}_{[n]},\overbar{\theta}_{[n]})\right].
\end{align*}
We remark that these learning guarantees hold for a class of risks that are in general non-convex and need not even be differentiable, let alone satisfy smoothness requirements.

\paragraph{Key points in the proof of Theorem \ref{thm:learn_stationary}}

Here we briefly highlight the key sub-results involved in proving Theorem \ref{thm:learn_stationary}; please see Appendix \ref{sec:proofs_learning} for all the details. The key structure that we require is a smooth loss, reflected in assumption \ref{asmp:background}. This along with the Lipschitz property of our function $\dev_{\sigma}$ for all $0 \leq \sigma < \infty$ allows us to prove that the underlying objective $\jrisk_{\sigma}$ is \emph{weakly convex}, where $\HH$ can be any Banach space (Proposition \ref{prop:learn_weak_convexity}); this generalizes a result of \citet[Lem.~4.2]{drusvyatskiy2019a} from Euclidean space to any Banach space. This alone is not enough to obtain the desired result. Note that the assumption \ref{asmp:feedback_cond_indep} is very weak, and trivially satisfied in most traditional machine learning settings (e.g., where losses are based on a sequence of iid data points). The question of whether the feedback is \emph{unbiased} or not, i.e., whether $\exx_{\ddist}G_t$ is in the sub-differential of $\jrisk_{\sigma}$ at step $t$ or not, is something that needs to be formally verified. In Proposition \ref{prop:learn_unbiased} we show that as long as the gradient has a finite expectation, this indeed holds for the feedback generated by (\ref{eqn:sgd_feedback}), when $\HH$ is any Banach space. With the two key properties of a weakly convex objective and unbiased random feedback in hand, we can leverage the techniques used in \citet[Thm.~3.1]{davis2019b} for proximal stochastic gradient methods applied to weakly convex functions on $\RR^{d}$, extending their core argument to the case of any Hilbert space. Combining this technique with the proof of weak convexity and unbiasedness lets us obtain Theorem \ref{thm:learn_stationary}.

\section{Empirical analysis}\label{sec:empirical}

In this section we introduce representative results for a series of experiments designed to investigate the quantitative and qualitative repercussions of modulating the underlying risk function class.\footnote{Repository with software and demos: \url{https://github.com/feedbackward/mrisk}}

\paragraph{Sanity check in one dimension}

As a natural starting point, we use a toy example to ensure that Algorithm \ref{algo:sgd} takes us where we expect to go for a particular risk setting. Consider a loss on $\RR$ with the form $\loss(h) = h \loss_{\text{wide}} + (1-h)\loss_{\text{thin}}$, where $\loss_{\text{wide}}$ and $\loss_{\text{thin}}$ are random variables independent of $h$ and each other. As a simple example, we use a folded Normal distribution for both, namely $\loss_{\ast}=|\text{Normal}(a_{\ast},b_{\ast}^{2})|$, where $a_{\text{wide}}=0$, $a_{\text{thin}}=2.0$, $b_{\text{wide}}=1.0$, and $b_{\text{thin}}=0.1$. For simplicity, we fix $\alpha_t = 0.001$ throughout, and each step uses a mini-batch of size $8$. Regarding the risk settings, we look in particular at the case of $\sigma = \infty$, where we modify the setting of $\eta = 2^k$ over $k=0,1,\ldots,7$. Results averaged over 100 trials are given in Figure \ref{fig:narrow_wide}. By modifying $\eta$, we can control whether the learning algorithm ``prefers'' candidates whose losses have a high degree of dispersion centered around a good location, or those whose losses are well-concentrated near a weaker location.

\begin{figure}[t]
\centering
\includegraphics[width=0.65\textwidth]{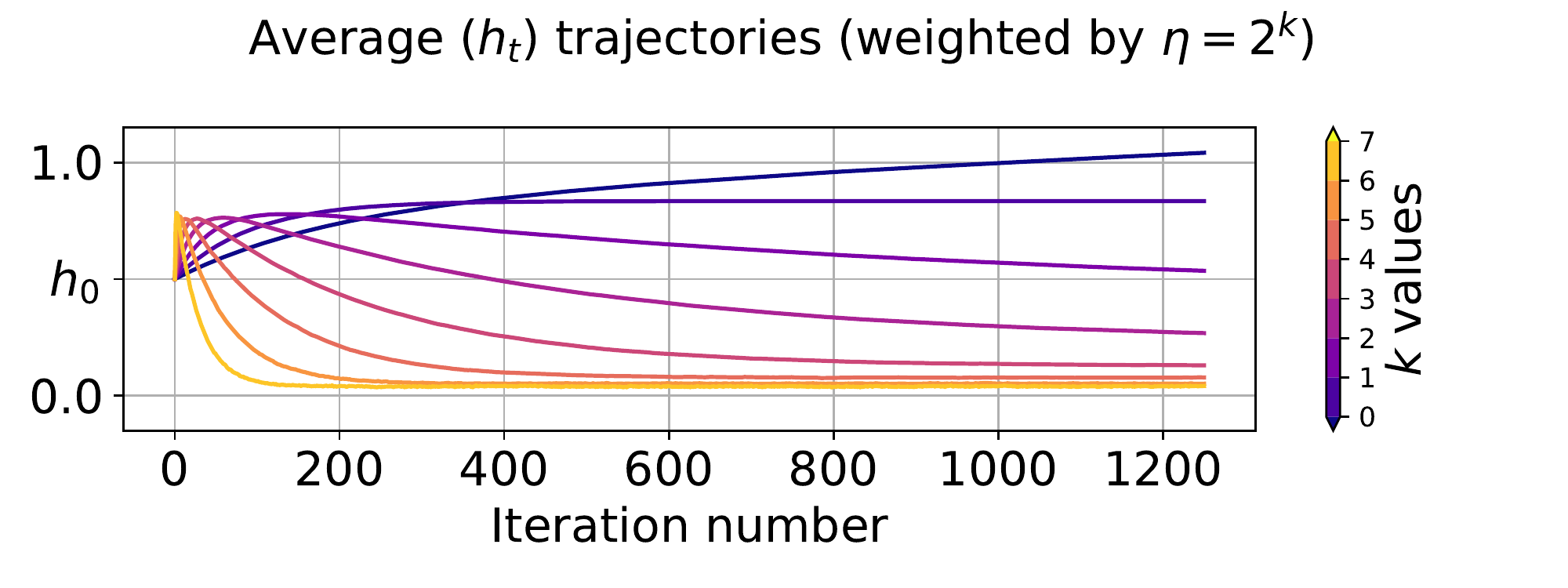}\,\includegraphics[width=0.35\textwidth]{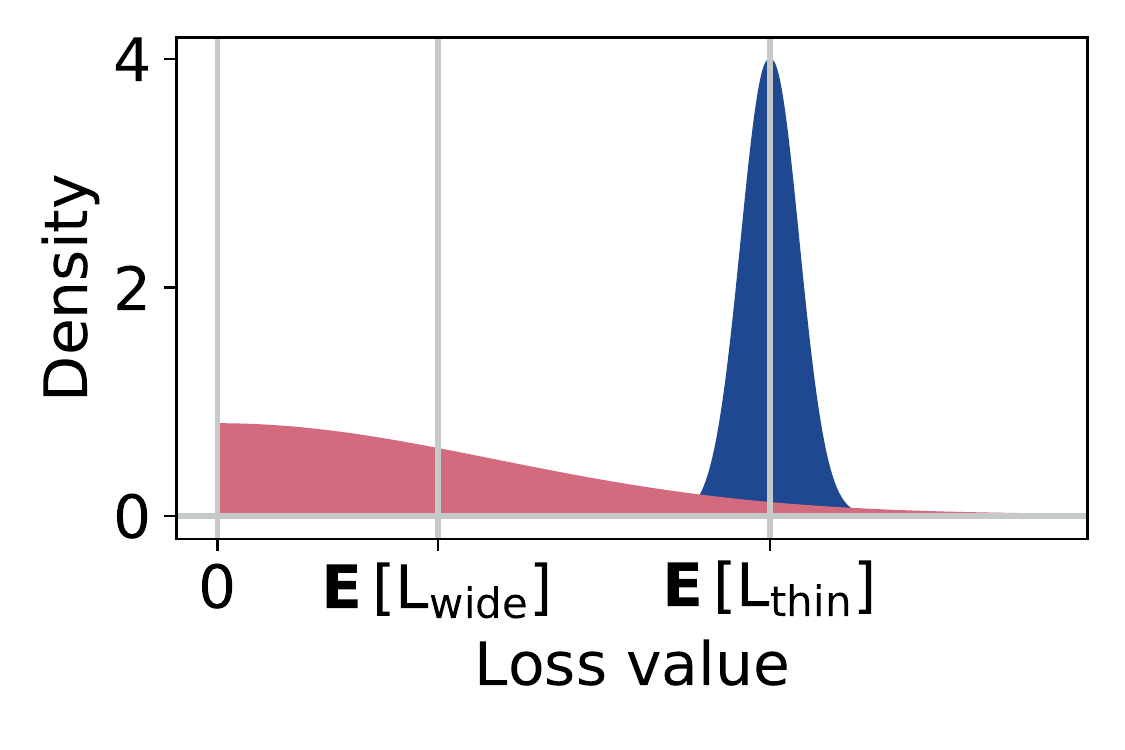}
\caption{A simple toy example using $\loss(h) = h \loss_{\text{wide}} + (1-h)\loss_{\text{thin}}$. Trajectories shown are the sequence $(h_t)$ generated by running (\ref{eqn:sgd_routine}) on $\RR^{2}$, with $h_0 = 0.5$ and $\theta_0 = 0.5$, averaged over all trials. Densities of $\loss_{\text{wide}}$ (red) and $\loss_{\text{thin}}$ (blue) are also plotted, with additional details in the main text.}
\label{fig:narrow_wide}
\end{figure}

\paragraph{Impact of risk choice on linear regression}

Next we consider how the key choice of $\sigma$ (and thus the underlying risk $\risk_{\sigma}$) plays a role on the behavior of Algorithm \ref{algo:sgd}. As another simple, yet more traditional example, consider linear regression in one dimension, where $Y = w_{0}^{\ast} + w_{1}^{\ast}X + \epsilon$, where $X$ and $\epsilon$ are independent zero-mean random variables, and $(w_{0}^{\ast},w_{1}^{\ast}) \in \RR^{2}$ are unknown to the learner. Using squared error $\loss(h) = (Y - h(X))^{2}$, we run Algorithm \ref{algo:sgd} again with mini-batches of size $8$ and $\alpha_t = 0.001$ fixed throughout, over a range of $\sigma \in [0, \infty]$ settings, for the same number of iterations as in the previous experiment. The initial value $(h_0,\theta_0)$ is initialized at zero plus uniform noise on $[-0.05,0.05]$. We also consider multiple noise distributions; as a concrete example, letting $N = \text{Normal}(0,(0.8)^2)$, we consider both $\varepsilon = N$ (Normal case) and $\varepsilon = \mathrm{e}^{N} - \exx\mathrm{e}^{N}$ (log-Normal case). In Figure \ref{fig:linreg_1d}, we plot the learned regression lines (averaged over 100 trials) for each choice of $\sigma$ and each noise setting. By modulating the target risk function, we can effectively choose between a self-imposed bias (smaller slope, lower intercept here), and a sensitivity to outlying values.

\begin{figure}[t]
\centering
\includegraphics[width=0.5\textwidth]{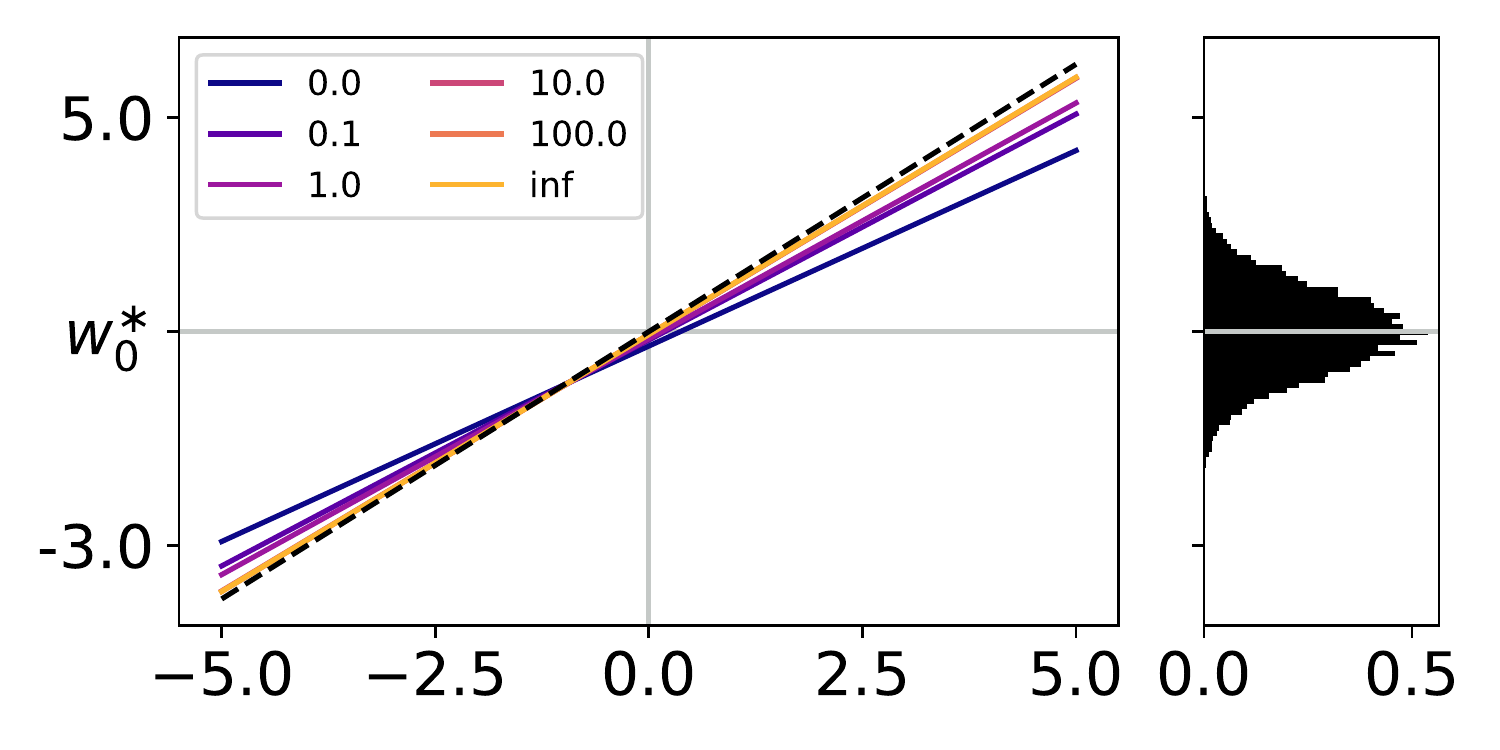}\includegraphics[width=0.5\textwidth]{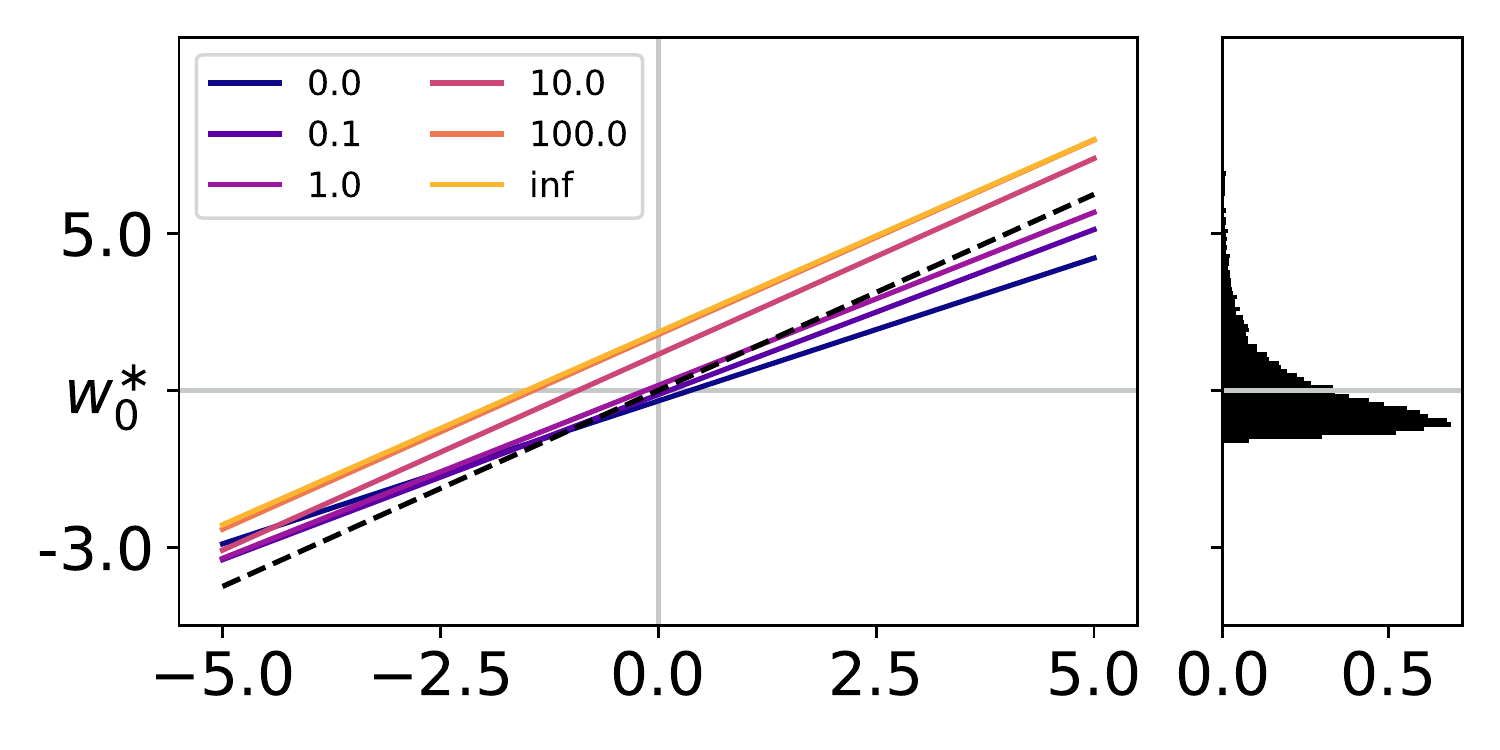}
\caption{Learned regression lines (solid; colors denote $\sigma \in [0,\infty]$) are plotted along with the true model $(w_{0}^{\ast},w_{1}^{\ast}) = (1.0,1.0)$ (dashed; black). Histograms are of independent samples of $w_{0}^{\ast} + \varepsilon$. The left plots are the Normal case, and right plots are the log-Normal case.}
\label{fig:linreg_1d}
\end{figure}

\paragraph{Tests using real-world data}

Finally, we consider an application to some well-known benchmark datasets for classification. At a high level, we run Algorithm \ref{algo:sgd} for multi-class logistic regression for 10 independent trials, where in each trial we randomly shuffle and re-split each full dataset (88\% training, 12\% testing), and randomly re-initialize the model weights identically to the previous paragraph, again with mini-batches of size $8$, and step sizes fixed to $\alpha_t = 0.01/\sqrt{d}$, where $d$ is the number of free parameters. Additional background on the datasets is given in appendix \ref{sec:empirical_supp}. The key question of interest is \emph{how the test loss distribution changes} as we modify the learner's feedback to optimize a range of risks $\risk_{\sigma}$. In Figure \ref{fig:real_data}, we see a stark difference between doing traditional empirical risk minimization (ERM, denoted ``off'') and using $\risk_{\sigma}$-based feedback, particularly for moderately large values of $\sigma$. The logistic losses are concentrated much more tightly (visible in the bottom row histograms), and this also leads to a better classification error (visible in the top row plots), an interesting trend that we observed across many distinct datasets.

\begin{figure}[t]
\centering
\includegraphics[width=0.25\textwidth]{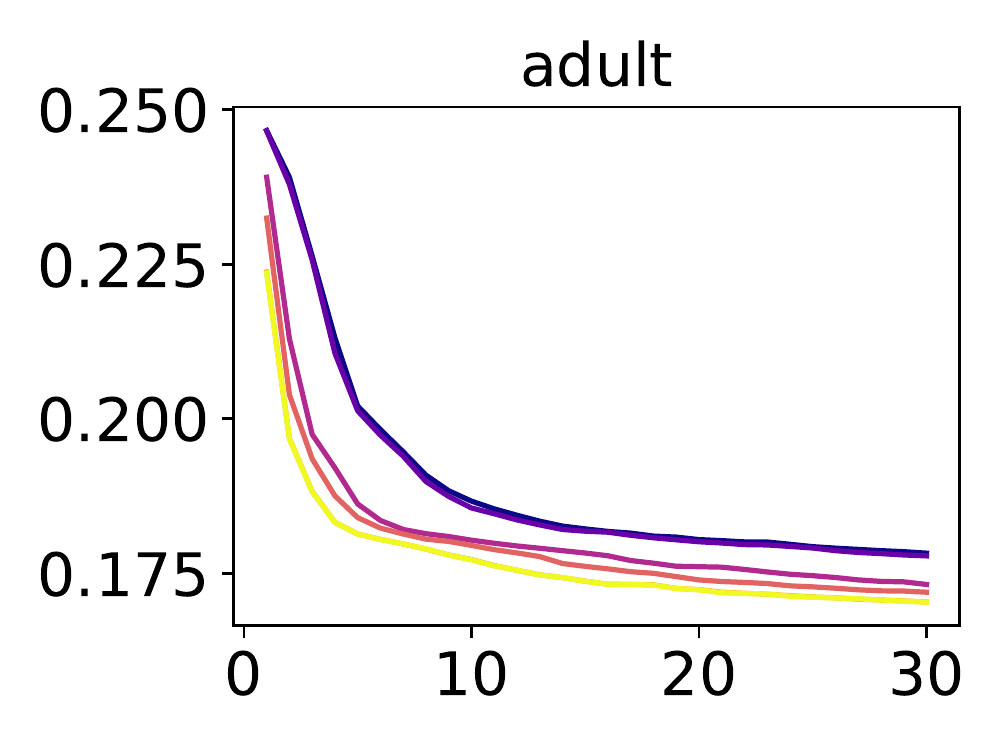}\includegraphics[width=0.25\textwidth]{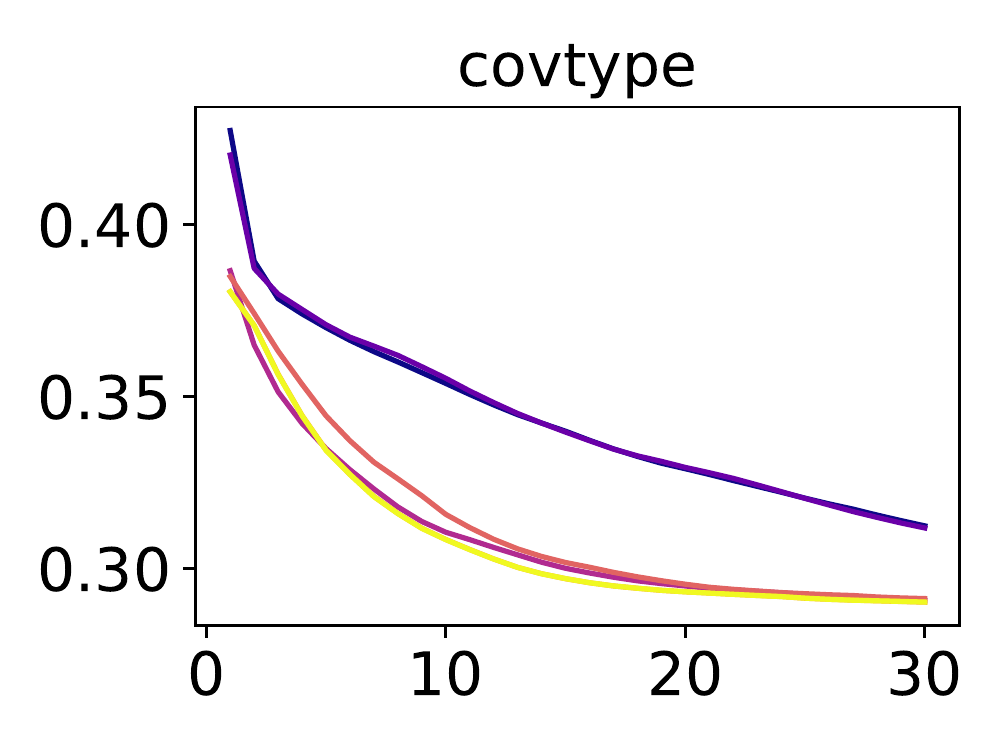}\includegraphics[width=0.25\textwidth]{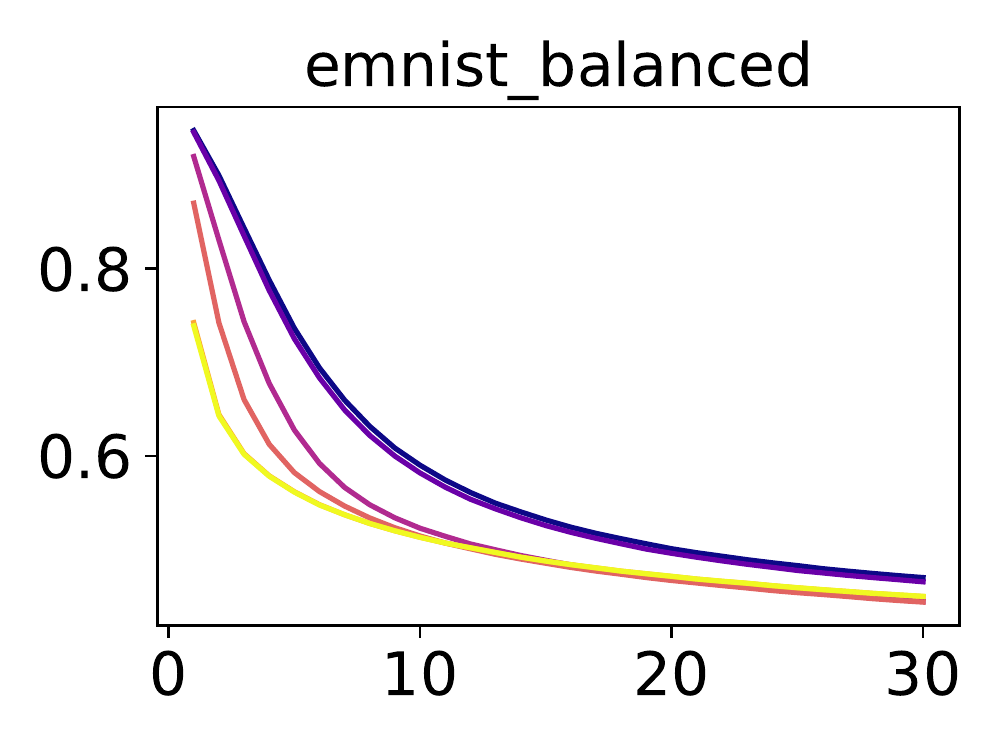}\includegraphics[width=0.25\textwidth]{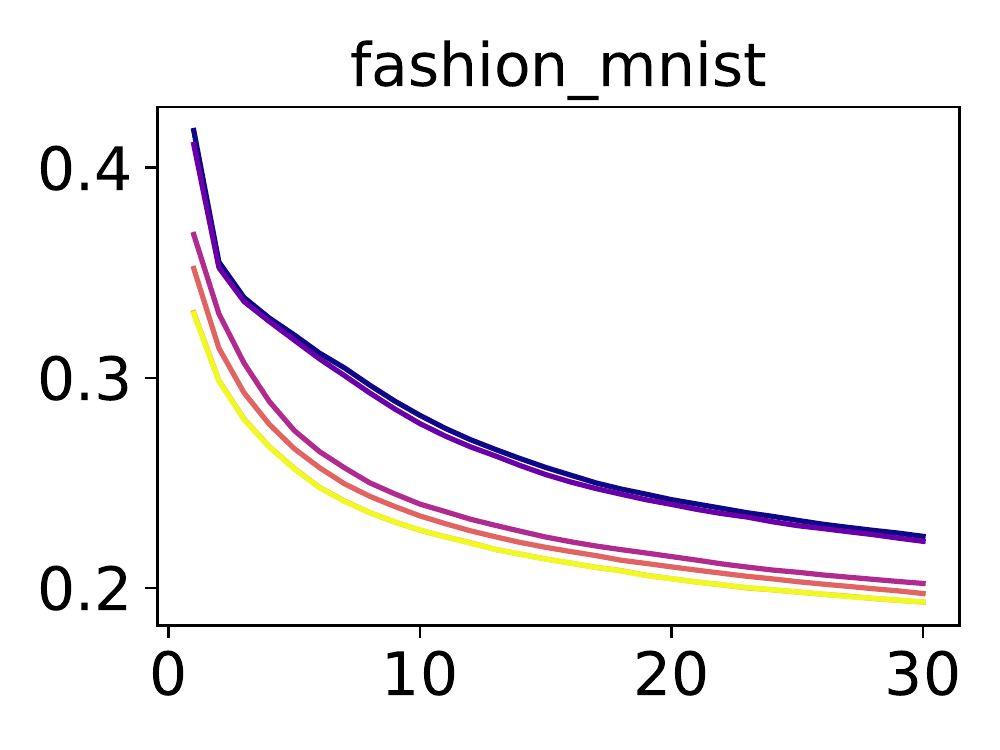}\\
\includegraphics[width=0.7\textwidth]{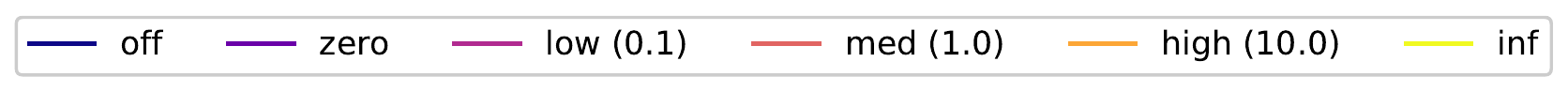}\\
\includegraphics[width=1.0\textwidth]{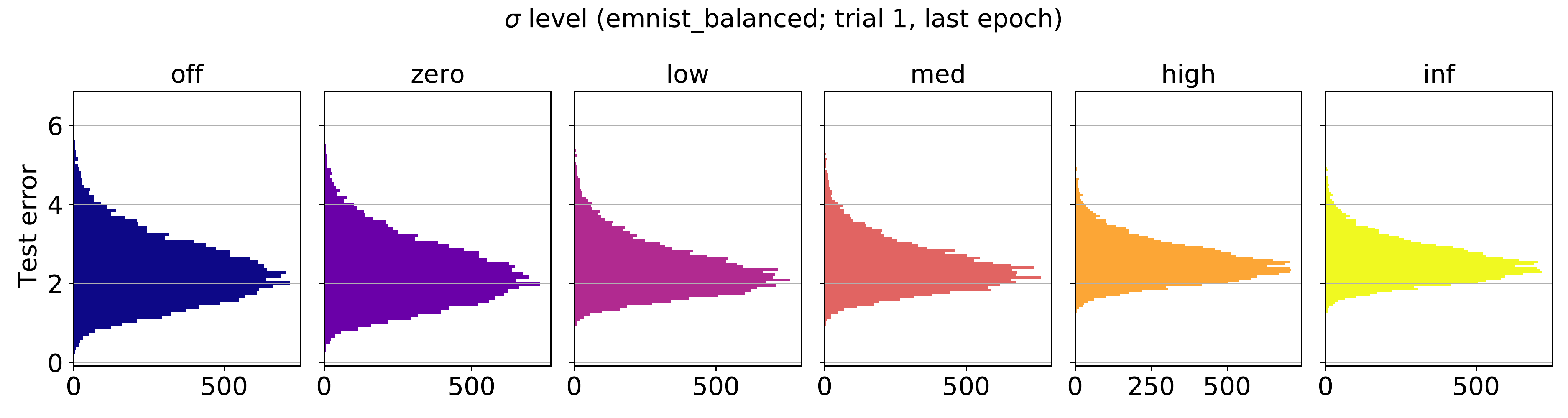}
\caption{Top row: average test error (zero-one loss) as a function of epochs, for four datasets and five $\sigma$ levels, plus traditional ERM (denoted ``off''). Bottom row: histograms of the test error (logistic loss) incurred after the final epoch for one trial under the \texttt{emnist\_balanced} dataset.}
\label{fig:real_data}
\end{figure}

\pagebreak

\appendix

\section{Overview of appendix contents}

Our appendix is comprised of several sections, ordered as follows:
\begin{itemize}
\item[\ref{sec:bg}] Background and setup
\item[\ref{sec:proofs}] Detailed proofs
\item[\ref{sec:helpers}] Helper results
\item[\ref{sec:empirical_supp}] Empirical supplement
\end{itemize}
As with the main paper, we handle theoretical topics before diving into empirical topics. Section \ref{sec:bg} gives a very detailed background including numerous formal definitions, supporting lemmas, and discussion on results used later in the detailed proofs (section \ref{sec:proofs}) for the main paper's results. Additional numerical test results are at the very end of section \ref{sec:empirical_supp}.

To provide additional visual intuition for the reader, we include at the start of this appendix several figures related to $\dev$ defined in (\ref{eqn:dev_handy}), $\dev_{\sigma}$ defined in (\ref{eqn:jrisk_defn}), and the resulting risk functions. In Figure \ref{fig:dev_examples} we plot $\dev$ and its derivatives, plus $\dev_{\sigma}$ for a wide variety of $\sigma \in [0,\infty]$ values. Additional details are given in the figure caption. In Figure \ref{fig:composition}, we show how specific choices of standard loss functions lead two different forms of the function composition $h \mapsto \loss(h) \mapsto \eta \dev_{\sigma}(\loss(h)-\theta)$.

\begin{figure}[t]
\centering
\includegraphics[width=0.33\textwidth]{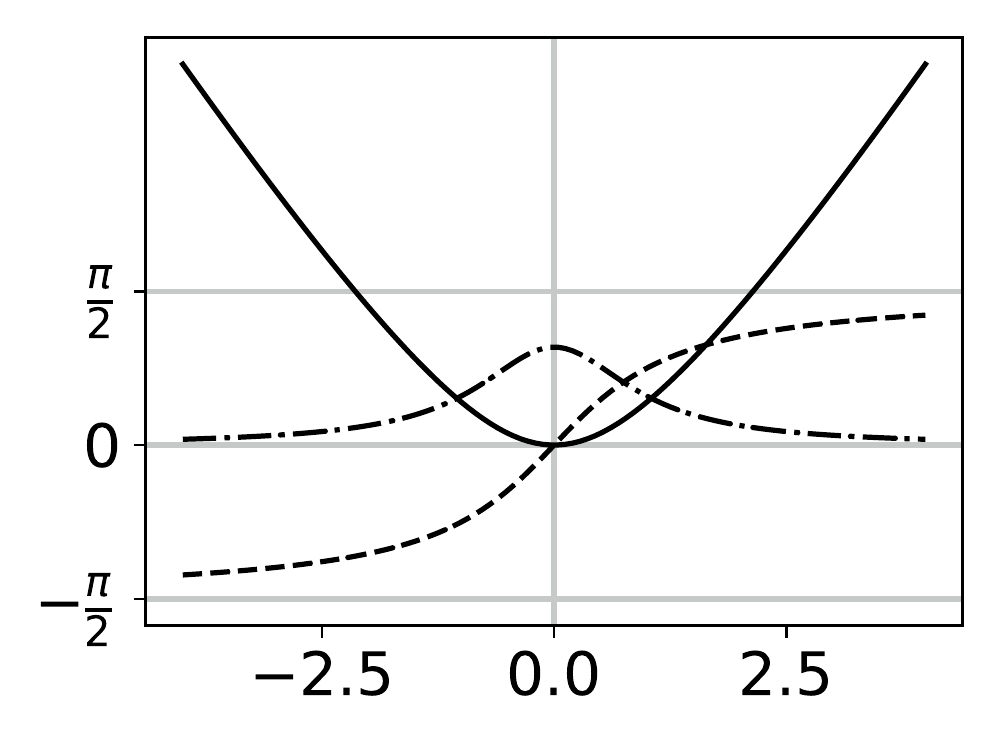}\includegraphics[width=0.33\textwidth]{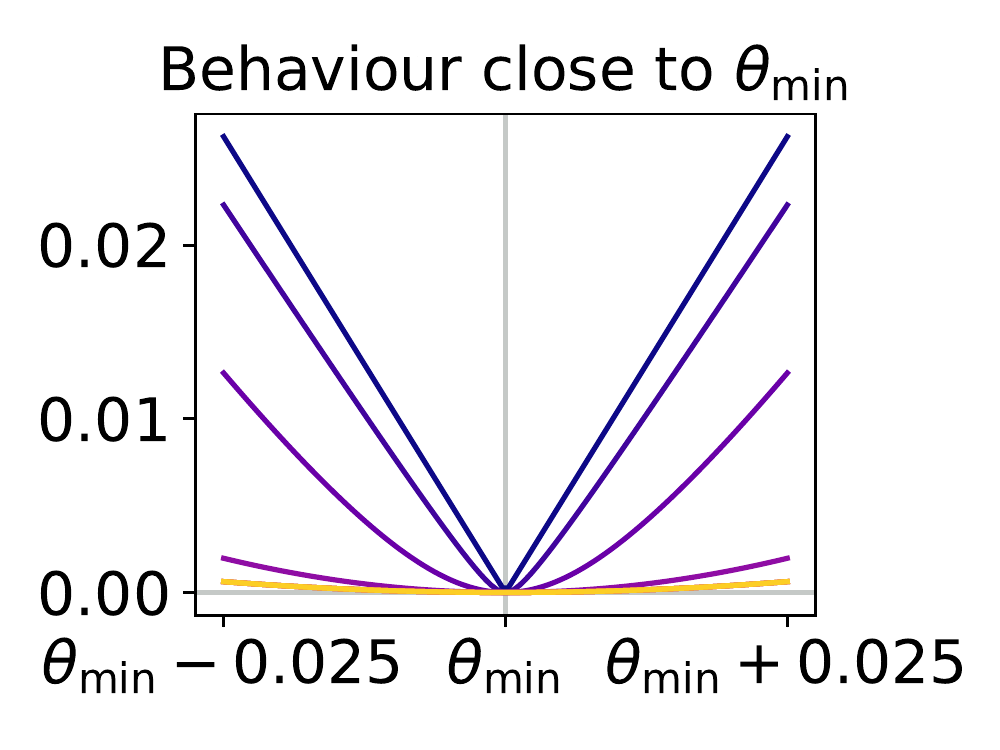}\includegraphics[width=0.33\textwidth]{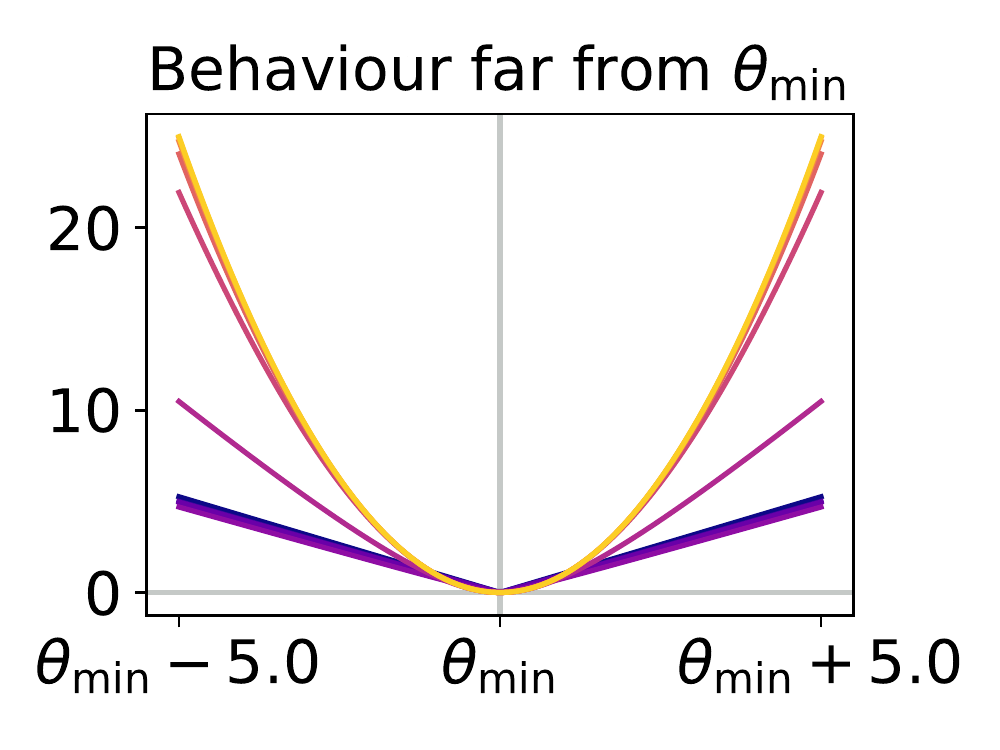}\\
\hspace{4.75cm}\includegraphics[width=0.6\textwidth]{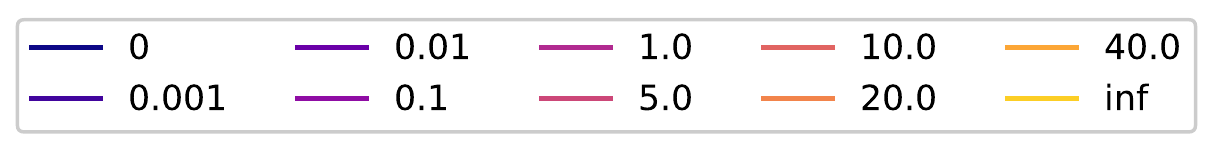}
\caption{Left: graphs of $\dev$ defined in (\ref{eqn:dev_handy}) (solid line), $\dev^{\prime}$ (dashed line), and $\dev^{\prime\prime}$ (dot-dash line). Middle--right: these are graphs of $\theta \mapsto \eta\dev_{\sigma}(1.0-\theta)$, where the minimizer is $\theta_{\min}=1.0$, the colors correspond to different $\sigma$ values, and $\eta$ is set following Remark \ref{rmk:mean_median}. That is, for $0 < \sigma < 1.0$, set $\eta = \sigma/\atan(\infty)$. For $\sigma = 0$, set $\eta = 1.05$. For $1.0 \leq \sigma < \infty$, set $\eta = 2\sigma^{2}$. For $\sigma = \infty$, set $\eta = 1.0$.}
\label{fig:dev_examples}
\end{figure}

\begin{figure}[t]
\centering
\includegraphics[width=0.5\textwidth]{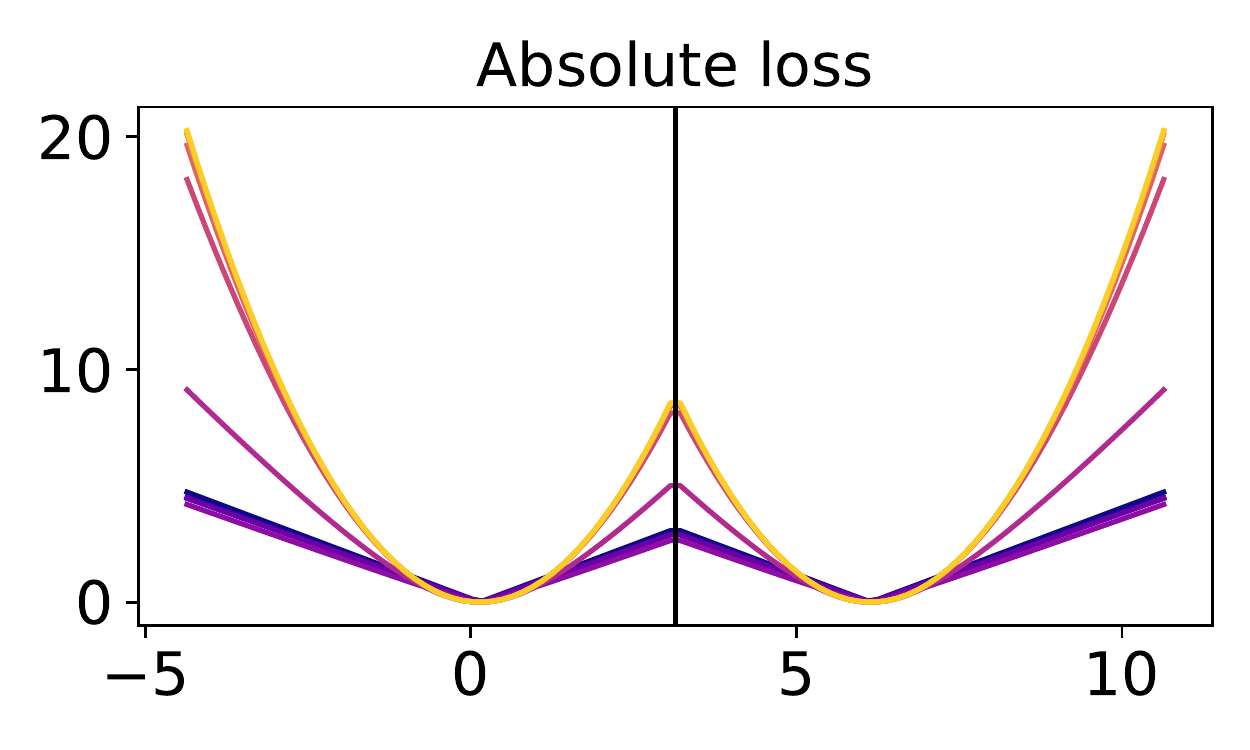}\includegraphics[width=0.5\textwidth]{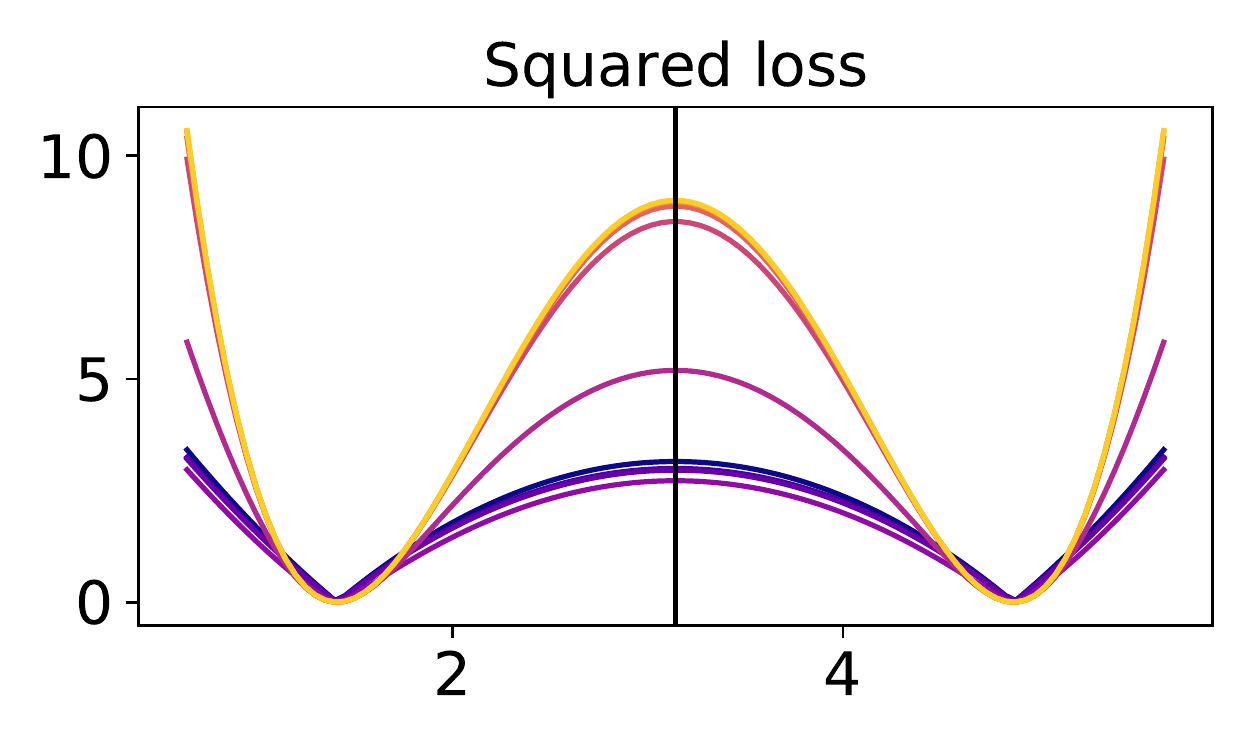}\\
\includegraphics[width=0.5\textwidth]{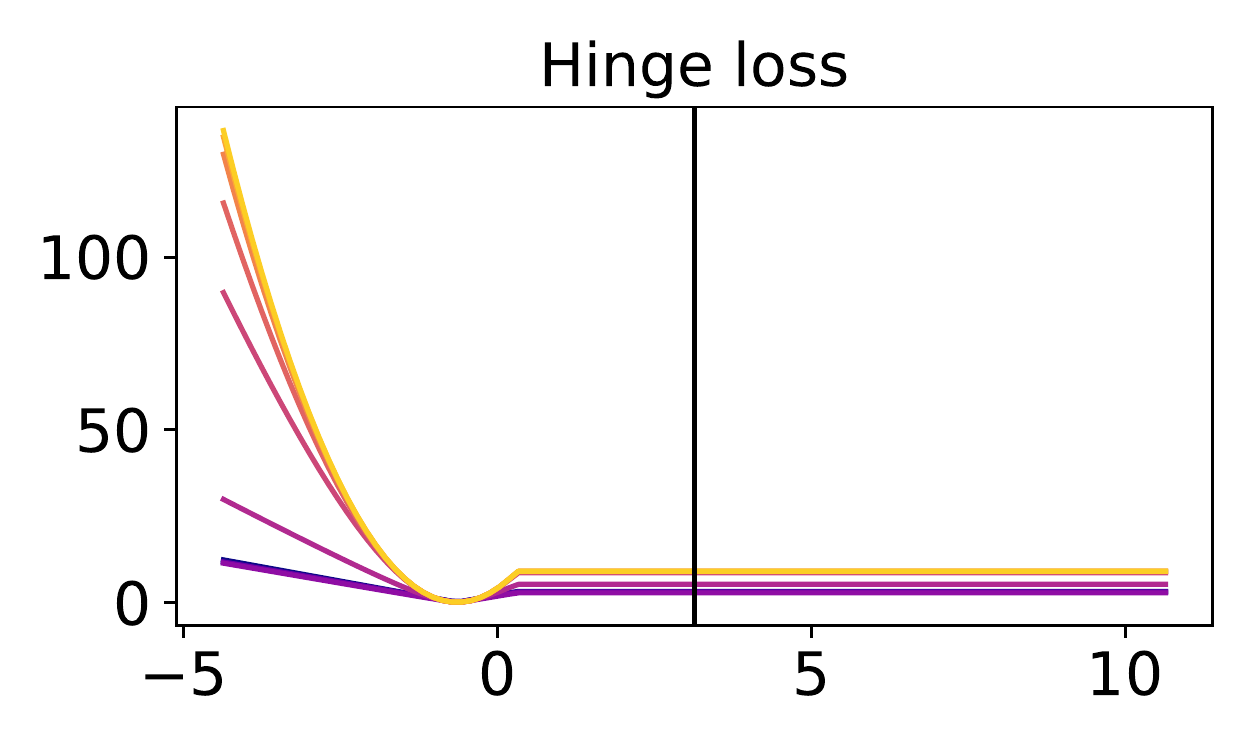}\includegraphics[width=0.5\textwidth]{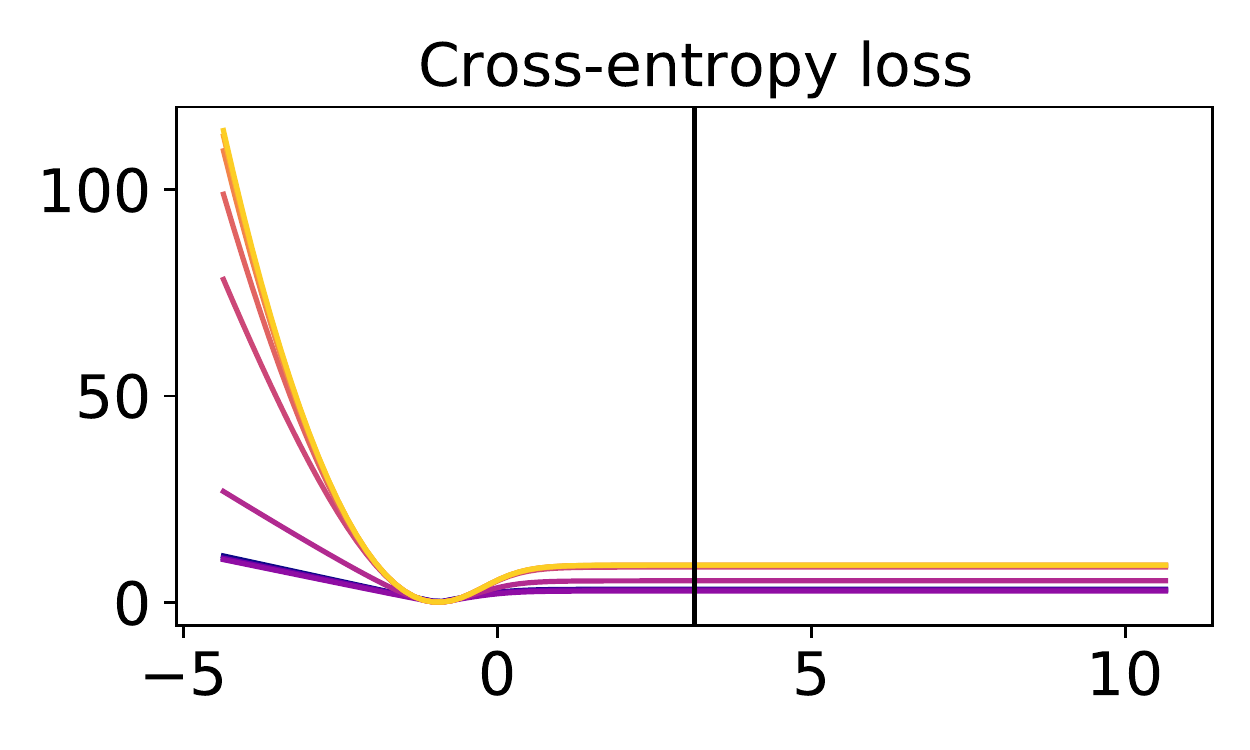}\\
\includegraphics[width=0.6\textwidth]{legend_dev}
\caption{Graphs of $h \mapsto \eta \dev_{\sigma}(\loss(h)-\theta)$, with $h \in \RR$, over different choices of $\sigma$ and $\loss$, with $\theta = 3.0$ fixed, and $\eta$ set as in Figure \ref{fig:dev_examples}. Colors correspond to $\sigma$, and each plot corresponds to a choice of $\loss(\cdot)$. Absolute: $\loss(h) = |h-h^{\ast}|$. Squared: $\loss(h) = (h-h^{\ast})^{2}$. Hinge: $\loss(h) = \max\{0,1-hh^{\ast}\}$. Cross-entropy: $\loss(h) = \log(1+\exp(-hh^{\ast}))$. In all cases, we have fixed $h^{\ast} = \pi$.}
\label{fig:composition}
\end{figure}

\section{Background and setup}\label{sec:bg}

\subsection{Preliminaries}\label{prop:bg_prelim}

\paragraph{General notation (probability)}
Underlying all our analysis is a probability space $(\Omega,\FF,\ddist)$.\footnote{For basic measure-theoretical facts supporting our main arguments, we use \citet{ash2000a} as a well-established and accessible reference. We will cite the exact results that pertain to our arguments in the main text as they become necessary.} All random variables, unless otherwise specified, will be assumed to be $\FF$-measurable functions with domain $\Omega$. Integration using $\ddist$ will be denoted by $\exx_{\ddist}Z \defeq \int_{\Omega}Z(\omega)\,\ddist(\dif \omega)$, and $\prr$ will be used as a generic probability function, typically representing $\ddist$ itself, or the product measure induced by a sample of random variables on $(\Omega,\FF,\ddist)$. We use $\LL_{2} \defeq \LL_{2}(\Omega,\FF,\ddist)$ to denote the set of all square-$\ddist$-integrable functions.\footnote{Strictly speaking, this is the set of all equivalence classes of square-$\ddist$-integrable functions, where $f \in \LL_{2}(\Omega,\FF,\ddist)$ represents all functions that are equal $\ddist$-almost everywhere.}

\paragraph{General notation (normed spaces)}
Let $\VV$ denote an arbitrary vector space. When we call $\VV$ a normed (linear) space, we are referring to $(\VV,\|\cdot\|)$, where $\|\cdot\|: \VV \to \RR$ denotes the relevant norm. For any normed space $\VV$, we shall denote by $\VV^{\ast}$ the usual dual space of $\VV$, namely all continuous linear functionals defined on $\VV$. The space $\VV^{\ast}$ is equipped with the norm $\|v^{\ast}\| \defeq \sup\{ v^{\ast}(u): \forall\,u \in \VV, \|u\| \leq 1 \}$. We shall use the notation $\langle \cdot,\cdot \rangle$ to represent the ``coupling'' function between $\VV$ and $\VV^{\ast}$, that is for any $u \in \VV$ and $v^{\ast} \in \VV^{\ast}$, we will write $\langle u, v^{\ast} \rangle \defeq v^{\ast}(u)$. For any sequence $(x_n)$ of elements $x_1,x_2,\ldots \in \VV$, we denote convergence of $(x_n)$ to some element $x^{\prime}$ by $x_n \to x^{\prime}$. When we take limits and do not specify a particular sequence, for example writing $x \to x^{\prime}$, then this refers to any sequence (of elements from $\VV$) that converges to $x^{\prime}$. In the special case of real-valued sequences (where $\VV \subset \RR$), if we write $x_n \to x^{\prime}_{+}$ (respectively $x_n \to x^{\prime}_{-}$), this refers to all sequences from above (resp.~below), i.e., any convergent sequence such that $x_n \geq x^{\prime}$ (resp.~$x_n \leq x^{\prime}$) for all $n$. We denote the open ball of radius $r>0$ centered at $x_0 \in \VV$ by $B(x_0;r) \defeq \{x \in \VV: \|x_0-x\| < r\}$. We denote the extended real line by $\overbar{\RR}$. On normed space $\VV$, we denote the interior of a set $U \subset \VV$ by $\inter U$ (all $x \in U$ such that $B(x;\delta) \subset U$ for some $\delta$).

\paragraph{General terminology}
On any normed linear space $\VV$, a set $A \subset \VV$ is said to be \term{compact} if for any sequence of elements in $A$, there exists a sub-sequence which converges on $A$. We denote the \term{effective domain} of an extended real-valued function $f$ by $\dom f \defeq \{x: f(x) < \infty\}$. We call a convex function $f:\VV \to \overbar{\RR}$ \term{proper} if $f > -\infty$ and $\dom f \neq \emptyset$. We say that $f$ is \term{coercive} if $\|x\| \to \infty$ implies $f(x) \to \infty$.\footnote{For example, the function $f(x)=x^{2}$ is coercive, but $f(x)=\exp(x)$ is not.} For a function $f: \XX \to \YY$, with $\XX$ and $\YY$ being normed spaces, we say $f$ is (locally) \term{Lipschitz} at $x_0 \in \XX$ if there exists $\delta > 0$ and $\lambda > 0$ such that $x_1,x_2 \in B(x_0;\delta)$ implies $\|f(x_1)-f(x_2)\| \leq \lambda\|x_1-x_2\|$. We say $f$ is $\lambda$-Lipschitz on $\XX$ if this property holds with a common coefficient $\lambda$ for all $x_0 \in \XX$.

\paragraph{Semi-continuous functions}
We say that a function $f$ is \term{lower semi-continuous}\footnote{Nice references on semi-continuity: \citet[Appendix 2]{ash2000a}, \citet[Ch.~2]{luenberger1969Book}, \citet[Sec.~2.1]{barbu2012ConvOptBanach}, \citet[Ch.~1]{penot2012CWOD}.} (LSC) at a point $x$ if for each $\varepsilon > 0$, there exists $\delta > 0$ such that $\|x - x^{\prime}\| < \delta$ implies $f(x^{\prime}) > f(x) - \varepsilon$. If $-g$ is LSC, then we say $g$ is \term{upper} semi-continuous (USC). The property that $f$ is LSC at a point $x$ is equivalent\footnote{\citet[Thm.~A2.2]{ash2000a}, \citet[Lem.~1.18]{penot2012CWOD}.} to the property that for any sequence $x_n \to x$, we have
\begin{align}\label{eqn:lsc_liminf_condition}
f(x) \leq \liminf_{n \to \infty} f(x_n).
\end{align}
Ordinary continuity is equivalent to being both USC and LSC, but the added generality of these weaker notions of continuity is often useful.

\paragraph{Differentiability}
We start by introducing some common notions of directional differentiability at a high level of generality.\footnote{We follow basic notation and terminology of the authoritative text by \citet{penot2012CWOD}.} Let $\XX$ and $\YY$ be normed linear spaces, $U \subset \XX$ an open set, and $f:\XX \to \YY$ a function of interest. The \term{radial derivative} of $f$ at $x \in U$ in direction $u$ is defined
\begin{align}\label{eqn:ddv_gateaux}
f_{r}^{\prime}(x;u) \defeq \lim\limits_{\alpha \to 0_{+}} \frac{f(x + \alpha u)-f(x)}{\alpha}.
\end{align}
A slight modification to this gives us the (Hadamard) \term{directional derivative} of $f$ at $x \in U$ in direction $u$:
\begin{align}\label{eqn:ddv_hadamard}
f^{\prime}(x;u) \defeq \lim\limits_{(\alpha,u^{\prime})\to(0_{+},u)} \frac{f(x + \alpha u^{\prime})-f(x)}{\alpha}
\end{align}
When $f_{r}^{\prime}(x;u)$ exists for all directions $u$, we say that $f$ is \term{radially differentiable} at $x$. Identically, when $f^{\prime}(x;u)$ exists for all directions $u$, we say that $f$ is \term{directionally differentiable} at $x$. When the map $u \mapsto f_{r}^{\prime}(x;u)$ is continuous and linear, we say that $f$ is \term{Gateaux differentiable} at $x$. When the map $u \mapsto f^{\prime}(x;u)$ is continuous and linear, we say $f$ is \term{Hadamard differentiable} at $x$. If $f$ is Hadamard differentiable, then it is Gateaux differentiable. The converse does not hold in general, but if $f$ is Lipschitz on a neighborhood of $x \in U$, then radial differentiability and directional differentiability (at $x$) are equivalent.\footnote{\citet[Prop.~2.25]{penot2012CWOD}.}

When we simply say that a function $f:\XX \to \YY$ is \term{differentiable} at $x \in U$, we mean that there exists a function $f^{\prime}(x)(\cdot):\XX \to \YY$ that is linear, continuous, and which satisfies
\begin{align}\label{eqn:deriv_frechet}
\lim\limits_{\|u\| \to 0} \frac{\|f(x+u)-f(u)-f^{\prime}(x)(u)\|}{\|u\|} = 0.
\end{align}
This property is often referred to as \term{Fr\'{e}chet differentiability}. When $f$ is differentiable at $x$, the map $f^{\prime}(x)$ is uniquely determined.\footnote{See for example \citet[Ch.~7]{luenberger1969Book} or \citet[Ch.~2]{penot2012CWOD}.} In the special case where $\YY \subset \RR$, the linear functional represented by $f^{\prime}(x) \in \XX^{\ast}$ is called the \term{gradient} of $f$ at $x$. Differentiability is also closely related to directional differentiability; if $f$ is Gateaux differentiable on $U$ and the map $x \mapsto f^{\prime}(x;\cdot)$ is continuous at $x$, then $f$ is differentiable at $x$.\footnote{\citet[Prop.~2.51]{penot2012CWOD}.}

\paragraph{Sub-differentials}
Let $\VV$ be any normed linear space. If $f:\VV \to \overbar{\RR}$ is a (proper) convex function, the \term{sub-differential} of $f$ at $x \in \dom f$ is defined
\begin{align}\label{eqn:subdif_convex}
\partial f(x) & \defeq \left\{ v^{\ast} \in \VV^{\ast}: f(u)-f(x) \geq \langle u-x, v^{\ast} \rangle, u \in \VV \right\}\\
\nonumber
& = \left\{ v^{\ast} \in \VV^{\ast}: f^{\prime}_{r}(x;u) \geq \langle u, v^{\ast} \rangle, u \in \VV \right\}.
\end{align}
The second characterization of $\partial f(x)$, given using the radial derivative (\ref{eqn:ddv_gateaux}), is useful and intuitive.\footnote{See \citet[Thm.~3.22]{penot2012CWOD} for this fact.} Some authors refer to this as the \term{Moreau-Rockafellar sub-differential} to emphasize the context of convex analysis.

More generally, however, if $f$ is not convex, then the strong global property used to define the MR sub-differential is so restrictive that most interesting functions are left out. A more general notion is that of the \term{Fr\'{e}chet sub-differential}.\footnote{We follow \citet[Sec.~4.1]{penot2012CWOD} for terms and notation here.} Denoted $\partial_{\frechet} f(x)$, the Fr\'{e}chet sub-differential of $f$ at $x$ is the set of all bounded linear functionals $v^{\ast} \in \VV^{\ast}$ such that for any $\varepsilon > 0$, there exists $\delta > 0$ such that
\begin{align}\label{eqn:subdif_frechet}
\|x-u\| \leq \delta \implies f(u)-f(x) \geq \langle u-x, v^{\ast}\rangle - \varepsilon\|x-u\|.
\end{align}
This local requirement is much weaker than the condition characterizing the MR-sub-differential, and clearly we have $\partial f(x) \subset \partial_{\frechet} f(x)$. When $f$ is assumed to be locally Lipschitz, another class of sub-differentials is often useful. Define the \term{Clarke directional derivative} of $f$ at $x$ in the direction $u$ by
\begin{align}
f_{\clarke}^{\prime}(x;u) \defeq \limsup_{(\alpha,x^{\prime}) \to (0_{+},x)} \frac{f(x^{\prime} + \alpha u
) - f(x^{\prime})}{\alpha}.
\end{align}
The corresponding \term{Clarke sub-differential} is defined as
\begin{align}\label{eqn:subdif_clarke}
\partial_{\clarke} f(x) \defeq \{ v^{\ast} \in \VV^{\ast}: f_{\clarke}^{\prime}(x;u) \geq \langle u, v^{\ast} \rangle, u \in \VV \}.
\end{align}
In the special case where $f$ is convex, all the sub-differentials coincide, i.e., $\partial f(x) = \partial_{\frechet} f(x) = \partial_{\clarke} f(x)$.\footnote{This follows from \citet[Prop.~5.34]{penot2012CWOD}. See also \citet[Sec.~4.1.1, Exercise 1]{penot2012CWOD}.} We say that a function $f$ is \term{sub-differentiable} at $x$ if its sub-differential (in any sense) at $x$ is non-empty. Finally, a remark on notation when using set-valued functions like $x \mapsto \partial_{\clarke} f(x)$. When we write something like ``we have $\langle u, \partial_{\clarke} f(x) \rangle \geq g(u)$,'' it is the same as writing ``we have $\langle u, v^{\ast} \rangle \geq g(u)$ for all $v^{\ast} \in \partial_{\clarke} f(x)$.'' This kind of notation will be used frequently.

\subsection{Generalized convexity}\label{sec:bg_convexity}

Let $\XX$ be a normed linear space. Take an open set $U \subset \XX$ and fix some point $x_0 \in U$. For a function $f:\XX \to \overbar{\RR}$ and parameter $\gamma \in \RR$, say that there exists $\delta > 0$ such that for all $x,x^{\prime} \in B(x_0;\delta)$ and $\alpha \in (0,1)$, we have
\begin{align}\label{eqn:weak_convexity_defn}
f(\alpha x + (1-\alpha)x^{\prime}) \leq \alpha f(x) + (1-\alpha)f(x^{\prime}) + \frac{\gamma \alpha(1-\alpha)}{2}\|x - x^{\prime}\|^{2}.
\end{align}
When $\gamma \geq 0$, we say $f$ is \term{$\gamma$-weakly convex} at $x_0$. When $\gamma \leq 0$, we say $f$ is \term{$(-\gamma)$-strongly convex} at $x_0$. When (\ref{eqn:weak_convexity_defn}) holds for all $x_0 \in U$, we say that $f$ is $\gamma$-weakly/strongly convex on $U$. The special case of $\gamma = 0$ is the traditional definition of convexity on $U$.\footnote{See for example \citet[Ch.~3]{nesterov2004ConvOpt}.}

The ability to construct a quadratic lower-bounding function for $f$ is closely related to notions of weak/strong convexity. Consider the following condition: given $\gamma \in \RR$, there exists $\delta > 0$ such that for all $x,x^{\prime} \in B(x_0;\delta)$ we have
\begin{align}\label{eqn:weak_convexity_quad_lowerbd}
f(x^{\prime}) \geq f(x) + \langle x^{\prime}-x, \partial_{\clarke}f(x) \rangle - \frac{\gamma}{2}\|x-x^{\prime}\|^{2}.
\end{align}
Here $\partial_{\clarke}f$ denotes the Clarke sub-differential of $f$, defined by (\ref{eqn:subdif_clarke}). Let us assume henceforth that $\XX$ is Banach, $f$ is locally Lipschitz, and $\partial_{\clarke}f(x)$ is non-empty for all $x \in U$.\footnote{\citet[Prop.~5.3]{penot2012CWOD}.} For any $\gamma \in \RR$, it is straightforward to show that $\text{(\ref{eqn:weak_convexity_defn})} \implies \text{(\ref{eqn:weak_convexity_quad_lowerbd})}$ holds.\footnote{See for example \citet[Thm.~3.1]{daniilidis2005a}; in particular their proof of $\text{(i)} \implies \text{(iii)}$. Their result is stated for $\XX = \RR^{d}$ and locally Lipschitz $f$, but the proof easily generalizes to Banach spaces. See also the remarks following their proof about how the local Lipschitz condition can be removed.} Since (\ref{eqn:weak_convexity_quad_lowerbd}) gives us a lower bound on both $f(x)$ and $f(x^{\prime})$ for any $x$ and $x^{\prime}$ close enough to $x_0$, adding up the inequalities immediately implies
\begin{align}\label{eqn:weak_convexity_hypomonotone}
\langle x-x^{\prime}, \partial_{\clarke}f(x) - \partial_{\clarke}f(x^{\prime}) \rangle \geq -\gamma\|x-x^{\prime}\|^{2}.
\end{align}
When $\XX$ is Banach and $f$ is locally Lipschitz, it is straightforward to show that $\text{(\ref{eqn:weak_convexity_hypomonotone})} \implies (\ref{eqn:weak_convexity_defn})$ is valid.\footnote{Just apply the argument for $\text{(ii)} \implies \text{(i)}$ employed by \citet[Thm.~3.1]{daniilidis2005a}, and strengthen their argument by using a more general form of Lebourg's mean value theorem \citep[Thm.~5.12]{penot2012CWOD}.} As such, for Banach spaces and locally Lipschitz functions, we have that the conditions (\ref{eqn:weak_convexity_defn}), (\ref{eqn:weak_convexity_quad_lowerbd}), and (\ref{eqn:weak_convexity_hypomonotone}) are all equivalent for the general case of $\gamma \in \RR$.

Let us consider one more closely related property on the same open set $U \subset \XX$:
\begin{align}\label{eqn:weak_convexity_regularize}
x & \mapsto f(x) + \frac{\gamma}{2}\|x\|^{2} \text{ is convex on } U.
\end{align}
In the special case where $\XX$ is a real Hilbert space and the norm $\|\cdot\|$ is induced by the inner product as $\|\cdot\| = \sqrt{\langle \cdot,\cdot \rangle}$, then for any $x,x^{\prime} \in U$ and $\alpha \in \RR$, the equality
\begin{align}\label{eqn:innerprod_norm_helper}
\alpha\|x\|^{2} + (1-\alpha)\|x^{\prime}\|^{2} = \|\alpha x + (1-\alpha)x^{\prime}\|^{2} + \alpha(1-\alpha)\|x-x^{\prime}\|^{2}
\end{align}
is easily checked to be valid.\footnote{\citet[Cor.~2.15]{bauschke2017CMH}.} In this case, the equivalence $\text{(\ref{eqn:weak_convexity_defn})} \iff \text{(\ref{eqn:weak_convexity_regularize})}$ follows from direct verification using (\ref{eqn:innerprod_norm_helper}).\footnote{See also \citet[Lem.~2.1]{davis2019b} for a similar result when $\XX = \RR^{d}$ and $f$ is LSC.}

The facts above are summarized in the following result:
\begin{prop}[Characterization of generalized convexity]\label{prop:weak_convexity_characterized}
Consider a function $f:\XX \to \overbar{\RR}$ on normed linear space $\XX$. When $\XX$ is Banach and $f$ is locally Lipschitz, then with respect to open set $U \subset \XX$ we have the following equivalence:
\begin{align*}
\text{(\ref{eqn:weak_convexity_defn})} \iff \text{(\ref{eqn:weak_convexity_quad_lowerbd})} \iff \text{(\ref{eqn:weak_convexity_hypomonotone})}.
\end{align*}
When $\XX$ is Hilbert, this equivalence extends to (\ref{eqn:weak_convexity_regularize}).
\end{prop}

\subsection{Function composition on normed spaces}\label{sec:bg_composition}

Next we consider the properties of compositions involving functions which are smooth and/or convex. Let $\XX$ and $\YY$ be normed linear spaces. Let $g: \XX \to \YY$ and $h: \YY \to \overbar{\RR}$ be the maps used in our composition, and denote by $f \defeq h \circ g$ the composition, i.e., $f(x) = h(g(x))$ for each $x \in \XX$. Our goal will be to present sufficient conditions for the composition $f$ to be weakly convex on an open set $U \subset \XX$, in the sense of (\ref{eqn:weak_convexity_defn}). If we assume simply that $h$ is convex, fixing any point $x_0 \in U$ such that $h$ is sub-differentiable at $g(x_0)$, it follows that for any choice of $x \in \XX$ we have
\begin{align}\label{eqn:weak_convexity_0}
f(x) = h(g(x)) & \geq h(g(x_0)) + \langle g(x)-g(x_0), \partial h(g(x_0)) \rangle.
\end{align}
Let us further assume that $h$ is $\lambda_0$-Lipschitz, and $g$ is \term{smooth} in the sense that it is differentiable on $U$ and the map $x \mapsto g^{\prime}(x)$ is $\lambda_1$-Lipschitz. For readability, denote the derivative $g^{\prime}(x_0): \XX \to \YY$ by $g_{0}^{\prime}(\cdot) \defeq g^{\prime}(x_0)(\cdot)$. Taking any choice of $v_0 \in \partial h(g(x_0))$, we can write
\begin{align}
\nonumber
\langle g(x)-g(x_0), v_0 \rangle & = \langle g_{0}^{\prime}(x-x_0), v_0 \rangle + \langle g(x)-g(x_0)-g_{0}^{\prime}(x-x_0), v_0 \rangle\\
\nonumber
& \geq \langle g_{0}^{\prime}(x-x_0), v_0 \rangle - \|g(x)-g(x_0)-g_{0}^{\prime}(x-x_0)\|\|v_0\|\\
\nonumber
& \geq \langle g_{0}^{\prime}(x-x_0), v_0 \rangle - \frac{\lambda_1}{2}\|x-x_0\|^{2}\|v_0\|\\
\label{eqn:weak_convexity_1}
& \geq \langle g_{0}^{\prime}(x-x_0), v_0 \rangle - \frac{\lambda_0\lambda_1}{2}\|x-x_0\|^{2}
\end{align}
The first inequality follows from the definition of the norm for linear functionals and the fact that $\partial h(g(x_0)) \subset \YY^{\ast}$. The second inequality follows from a Taylor approximation for Banach spaces (Proposition \ref{prop:smooth_taylor_1st}), using the smoothness of $g$. The final equality follows from the fact that for convex functions, the Lipschitz coefficient implies a bound on all sub-gradients, see (\ref{eqn:lip_convex_subgrad_bound}). To deal with the remaining term, note that we can write
\begin{align}\label{eqn:weak_convexity_2}
\langle g_{0}^{\prime}(x-x_0), \partial h(g(x_0)) \rangle = \langle x-x_0, (g_{0}^{\prime})^{\ast}(\partial h(g(x_0))) \rangle = \langle x-x_0, \partial h(g(x_0)) \circ g_{0}^{\prime} \rangle.
\end{align}
To explain the notation here, we use $(\cdot)^{\ast}$ to denote the adjoint, namely $A^{\ast}(y^{\ast}) \defeq y^{\ast} \circ A$, induced by any continuous linear map $A: \XX \to \YY$, defined for each $y^{\ast} \in \YY^{\ast}$.\footnote{Written explicitly, for any $u \in \XX$, we have $\langle u, A^{\ast}(y^{\ast}) \rangle = A^{\ast}(y^{\ast})(u) = y^{\ast}(Au) = \langle Au, y^{\ast} \rangle$. If $A: \XX \to \YY$ is linear and continuous, this implies that the adjoint $A^{\ast}$ is a continuous linear map from $\YY^{\ast}$ to $\XX^{\ast}$. For more general background: \citet[Ch.~6]{luenberger1969Book}, \citet[Ch.~1]{penot2012CWOD}.} The special case we have considered here is where $Au = g_{0}^{\prime}(u)$, noting that differentiability means that the map $u \mapsto g_{0}^{\prime}(u)$ is continuous and linear. Recalling the desired form of (\ref{eqn:weak_convexity_quad_lowerbd}), we need to establish a connection with $\partial_{\clarke}f(x_0)$. If we further assume that $g$ is locally Lipschitz, then we have
\begin{align}\label{eqn:weak_convexity_3}
\partial_{\clarke}f(x_0) \subset \partial_{\clarke}h(g(x_0)) \circ g_{0}^{\prime} = \partial h(g(x_0)) \circ g_{0}^{\prime},
\end{align}
where the equality follows from the coincidence of sub-differentials in the convex case, and the key inclusion follows from direct application of a generalized chain rule.\footnote{See for example \citet[Thm.~5.13(b)]{penot2012CWOD}. This inclusion holds as long as $g$ is strictly differentiable \citep[Defn.~2.54]{penot2012CWOD}, a property implied by the smoothness we have assumed.} Taking these facts together yields the following result.
\begin{prop}[Weak convexity for composite functions]\label{prop:weak_convexity_comp}
Let $\XX$ and $\YY$ be Banach spaces. Let $g:\XX \to \YY$ be locally Lipschitz and $\lambda_1$-smooth on an open set $U \subset \XX$. Let $h:\YY \to \overbar{\RR}$ be convex and $\lambda_0$-Lipschitz on $g(U) \subset \YY$. Furthermore, let $g(U) \subset \dom h$. Then, the composite function $f \defeq h \circ g$ is $\gamma$-weakly convex on $U$, for $\gamma = \lambda_0\lambda_1$.
\end{prop}
\begin{proof}
The desired result just requires us to piece together the key facts we have outlined in the main text. Local Lipschitz properties for $g$ and $h$ imply that $f$ is locally Lipschitz, and thus $\partial_{\clarke}f(x_0) \neq \emptyset$ for all $x_0 \in U$. Using (\ref{eqn:weak_convexity_3}), we have that $\partial h(g(x_0)) \neq \emptyset$ as well. With this in mind, linking up (\ref{eqn:weak_convexity_0})--(\ref{eqn:weak_convexity_2}), under the assumptions stated, for each $x,x_0 \in U$ we have
\begin{align}\label{eqn:weak_convexity_4}
f(x) \geq f(x_0) + \langle x-x_0, \partial h(g(x_0)) \circ g_{0}^{\prime} \rangle - \frac{\lambda_0\lambda_1}{2}\|x-x_0\|^{2}.
\end{align}
Using the inclusion (\ref{eqn:weak_convexity_3}) with (\ref{eqn:weak_convexity_4}), we have (\ref{eqn:weak_convexity_quad_lowerbd}) for $\gamma = \lambda_0\lambda_1$, and the desired result holds since (\ref{eqn:weak_convexity_quad_lowerbd}) implies (\ref{eqn:weak_convexity_defn}).
\end{proof}
\begin{rmk}
We note that Proposition \ref{prop:weak_convexity_comp} extends a result of \citet[Lem.~4.2]{drusvyatskiy2019a} from the case where $\XX$ and $\YY$ are finite-dimensional Euclidean spaces, to the general Banach space setting considered here. For the classical case of Euclidean spaces, exact chain rules are well-known \citep[Ch.~10.B]{rockafellar1998VA}.
\end{rmk}

\subsection{Proximal maps of weakly convex functions}\label{sec:bg_prox}

For normed linear space $\XX$ and function $f:\XX \to \overbar{\RR}$, the \term{Moreau envelope} $\env_{\beta}f$ and \term{proximal mapping} (or \term{proximity operator}) $\prox_{\beta}f$ are respectively defined for each $x \in \XX$ as follows:
\begin{align}
\label{eqn:env_defn}
\env_{\beta}f(x) & \defeq \inf\left\{ f(x^{\prime}) + \frac{1}{2\beta}\|x-x^{\prime}\|^{2} : x^{\prime} \in \XX \right\}\\
\label{eqn:prox_defn}
\prox_{\beta}f(x) & \defeq \argmin_{x^{\prime} \in \XX} \left[ f(x^{\prime}) + \frac{1}{2\beta}\|x-x^{\prime}\|^{2} \right].
\end{align}
Here $\beta > 0$ is a parameter. In the case where $f$ is convex, the basic properties of the proximal map and envelope are well-understood, particularly when $\XX$ is a Hilbert space.\footnote{See for example \citet[Ch.~12 and 24]{bauschke2017CMH}. For Banach spaces, modified notions of ``proximity'' measured using Bregman divergences have also been developed \citep{bauschke2003a,soueycatt2020a}. See also \citet{jourani2014a} for more analysis of the Moreau envelope in more general spaces.} These insights extend readily to the setting of weak convexity. Under the assumption that $\XX$ is Hilbert, let $f$ be $\gamma$-weakly convex on $\XX$. Trivially we can write
\begin{align*}
f(x^{\prime}) + \frac{1}{2\beta}\|x-x^{\prime}\|^{2} = \left( f(x^{\prime}) + \frac{\gamma}{2}\|x-x^{\prime}\|^{2} \right) + \frac{1}{2}\left(\frac{1}{\beta}-\gamma\right)\|x-x^{\prime}\|^{2}.
\end{align*}
If we write $f_{\gamma,x}(u) \defeq f(u) + (\gamma/2)\|x-u\|^{2}$ and $\beta_\gamma \defeq (\beta^{-1}-\gamma)^{-1}$ for readability, then as long as $\beta_\gamma > 0$ we have for all $x \in \XX$ that $\env_{\beta}f(x) = \env_{\beta_\gamma}f_{\gamma,x}(x)$ and $\prox_{\beta}f(x) = \prox_{\beta_\gamma}f_{\gamma,x}(x)$. By leveraging Proposition \ref{prop:weak_convexity_characterized} under the Hilbert space assumption, we have that for any $x \in \XX$, the function $f_{\gamma,x}(\cdot)$ is convex. This means that as long as $\beta_\gamma > 0$, namely whenever $\gamma < \beta^{-1}$, all the standard results available for the case of convex functions can be brought to bear on the problem.\footnote{For example, see \citet[Sec.~12.4]{bauschke2017CMH}.} Of particular importance to us is the fact that when $f$ is LSC and $\gamma$-weakly convex, the Moreau envelope is differentiable, with gradient
\begin{align}\label{eqn:env_gradient_form}
(\env_{\beta}f)^{\prime}(x) = \frac{1}{\beta}\left(x - \prox_{\beta}f(x) \right),
\end{align}
well-defined for all $\beta < \gamma^{-1}$ and $x \in \XX$.\footnote{One can just apply standard arguments such as given by \citet[Prop.~12.30]{bauschke2017CMH}, while utilizing the weak convexity property described. See also \citet[Lem.~4.3]{drusvyatskiy2019a}, \citet[Lem.~2.2]{davis2019b}, and \citet[Thm.~4.4]{poliquin1996a}.} We will be interested in finding \term{stationary points}  of $f$, namely those $x \in \XX$ such that $0 \in \partial_{\clarke}f(x)$. From the basic properties of the envelope and proximal mapping, for $\gamma$-weakly convex $f$ we have
\begin{align}
\dist(0;\partial_{\clarke}f(\prox_{\beta}f(x))) \leq \|(\env_{\beta}f)^{\prime}(x)\|.
\end{align}
That is, for any point $x \in \XX$, the point $\prox_{\beta}f(x) \in \XX$ is approximately stationary. The degree of precision is controlled by the gradient of $\env_{\beta}f$ evaluated at $x$. In addition, it follows immediately from (\ref{eqn:env_gradient_form}) that
\begin{align}\label{eqn:env_gradient_norm}
\|x - \prox_{\beta}f(x)\| = \beta\|(\env_{\beta}f)^{\prime}(x)\|.
\end{align}
Since one trivially also has $f(\prox_{\beta}f(x)) \leq f(x)$, the norm of the gradient of $\env_{\beta}f$ evaluated at $x$ also tells us how far we are from a point (namely $\prox_{\beta}f(x) \in \XX$) which is no worse than $x$ in terms of function value. These basic facts directly motivate the use of the Moreau envelope norm to quantify algorithm performance.\footnote{This is highlighted in works such as \citet{drusvyatskiy2019a} and \citet{davis2019b}.}

\section{Detailed proofs}\label{sec:proofs}

\subsection{Proofs for section \ref{sec:mrisk}}\label{sec:proofs_mrisk}

\begin{lem}[Lower semi-continuity]\label{lem:dev_lsc}
Let $\ZZ$ be a linear space of $\FF$-measurable random variables, and let $\dev:\RR \to \overbar{\RR}$ be any non-negative LSC function that is Borel-measurable. Then we have that the functional $(Z,\theta) \mapsto \exx_{\ddist}\dev(Z-\theta)$ is also LSC.
\end{lem}
\begin{proof}[Proof of Lemma \ref{lem:dev_lsc}]
Let $(Z_k)$ and $(\theta_k)$ respectively be convergent sequences on $\ZZ$ and $\RR$. As we take $k \to \infty$, say $Z_k \to Z_{\ast}$ pointwise, for some $Z \in \ZZ$, and $\theta_k \to \theta_{\ast} \in \RR$. Since by assumption $\dev$ is LSC on $\RR$, using (\ref{eqn:lsc_liminf_condition}) we have (again, pointwise) that
\begin{align*}
\dev\left(Z_{\ast}-\theta_{\ast}\right) \leq \liminf_{k \to \infty} \dev\left(Z_k-\theta_k\right).
\end{align*}
Writing $\dev_k \defeq \dev(Z_k-\theta_k)$ for each $k \geq 1$ and $\dev_{\ast} = \dev(Z_{\ast}-\theta_{\ast})$, it follows that
\begin{align}\label{eqn:dev_lsc_1}
\exx_{\ddist}\dev_{\ast} \leq \exx_{\ddist} \left(\liminf_{k \to \infty} \dev_k\right) \leq \liminf_{k \to \infty} \left(\exx_{\ddist}\dev_k\right).
\end{align}
The former inequality follows from monotonicity of the integral, and the latter inequality follows from an application of Fatou's inequality, which is valid since $\dev_k \geq 0$.\footnote{\citet[Lem.~1.6.8]{ash2000a}.} Taking both ends of (\ref{eqn:dev_lsc_1}) together, since the choice of sequences $(Z_k)$ and $(\theta_k)$ were arbitrary, it follows again from the equivalence (\ref{eqn:lsc_liminf_condition}) that the functional $(Z,\theta) \mapsto \exx_{\ddist}\dev((Z-\theta)/\sigma)$ is LSC on $\ZZ \times \RR$.
\end{proof}

\begin{lem}[Basic integration properties]\label{lem:integration}
Let $\exx_{\ddist}Z^{2} < \infty$ hold, and take any $\theta \in \RR$. Then the following properties of integrals based on $\dev_{\sigma}$ defined in (\ref{eqn:mrisk_defn_general}) hold:
\begin{itemize}
\item For all $\sigma \in [0,\infty]$, we have $0 \leq \exx_{\ddist}\dev_{\sigma}(Z-\theta) < \infty$.

\item For $0 < \sigma \leq \infty$, $\dev_{\sigma}(\cdot)$ is differentiable, and we have $\exx_{\ddist}|\dev_{\sigma}^{\prime}(Z-\theta)| < \infty$.

\item For $0 < \sigma \leq \infty$, $\dev_{\sigma}^{\prime}(\cdot)$ is differentiable, and we have $0 \leq \exx_{\ddist}\dev_{\sigma}^{\prime\prime}(Z-\theta) < \infty$.
\end{itemize}
Furthermore, the Leibniz integration property holds for both derivatives, that is
\begin{align}\label{eqn:leibniz_sigma}
\frac{d}{d\theta}\exx_{\ddist}\dev_{\sigma}(Z-\theta) = -\frac{\exx_{\ddist}\dev_{\sigma}^{\prime}(Z-\theta)}{\sigma}, \qquad \frac{d}{d\theta}\exx_{\ddist}\dev_{\sigma}^{\prime}(Z-\theta) = -\frac{\exx_{\ddist}\dev_{\sigma}^{\prime\prime}(Z-\theta)}{\sigma}
\end{align}
for any $0 < \sigma < \infty$, and for the special case of $\sigma=\infty$, we have
\begin{align}\label{eqn:leibniz_inf}
\frac{d}{d\theta}\exx_{\ddist}\dev_{\infty}(Z-\theta) = -2(\exx_{\ddist}Z-\theta), \qquad \frac{d}{d\theta}\exx_{\ddist}\dev_{\infty}^{\prime}(Z-\theta) = -2.
\end{align}
These equalities hold for any $\theta \in \RR$.
\end{lem}
\begin{proof}[Proof of Lemma \ref{lem:integration}]
Non-negativity of $\dev_{\sigma}$ implies $0 \leq \exx_{\ddist}\dev_{\sigma}(Z-\theta)$ for all $\sigma$. Regarding finiteness, starting with $\sigma = 0$ we have $\exx_{\ddist}\dev_{0}(Z-\theta) = \exx_{\ddist}|Z-\theta| < \infty$, which follows from H\"{o}lder's inequality and $\ddist$-integrability of $Z^{2}$. For $0 < \sigma < \infty$, first note that $\atan^{\prime}(u) \leq 1$ for all $u \in \RR$, and thus since $\atan(0)=0$, we have $|\atan(u)| \leq |u|$ and $u\atan(u) \leq u^{2}$ for all $u$. In particular, this means $\dev_{\sigma}(Z-\theta) \leq (Z-\theta)^{2}/\sigma^{2}$, and thus square-$\ddist$-integrability of $Z$ implies $\exx_{\ddist}\dev_{\sigma}(Z-\theta) < \infty$. The $\sigma = \infty$ case follows identically.

Moving to $\dev_{\sigma}^{\prime}(\cdot)$, for $0 < \sigma < \infty$ we have that $|\dev_{\sigma}^{\prime}(u)| < \pi/2$ for all $u \in \RR$, and thus $\exx_{\ddist}|\dev_{\sigma}^{\prime}(Z-\theta)| < \infty$. The exact same argument holds for $\dev_{\sigma}^{\prime\prime}(\cdot)$, since $0 < \dev_{\sigma}^{\prime\prime}(u) \leq 1$ for all $u \in \RR$. The $\sigma = \infty$ case follows analogously.

For the Leibniz property, let $(a_k)$ be any real sequence such that $a_k \to 0$. Using the fact that $\dev_{\sigma}^{\prime}$ is bounded, the dominated convergence theorem lets us deduce the following:
\begin{align*}
\lim\limits_{k \to \infty} \left[ \frac{\exx_{\ddist}\dev_{\sigma}(Z-(\theta+a_k)) - \exx_{\ddist}\dev_{\sigma}(Z-\theta)}{a_k} \right]& = \lim\limits_{k \to \infty} \exx_{\ddist}\left[ \frac{\dev_{\sigma}(Z-(\theta+a_k)) - \dev_{\sigma}(Z-\theta)}{a_k} \right]\\
& = \exx_{\ddist}\left[\lim\limits_{k \to \infty} \frac{\dev_{\sigma}(Z-(\theta+a_k)) - \dev_{\sigma}(Z-\theta)}{a_k} \right]\\
& = -\frac{\exx_{\ddist}\dev_{\sigma}^{\prime}(Z-\theta)}{\sigma}.
\end{align*}
We note that the first equality just uses $\ddist$-integrability and linearity of the Lebesgue integral, the second equality uses boundedness and integrability of the derivative, plus dominated convergence (e.g., \citet[Thm.~1.6.9]{ash2000a}). The last equality is just the chain rule applied to the differentiable function $\dev_{\sigma}(\cdot)$. Since the sequence $(a_k)$ was arbitrary, we conclude that the first equality of (\ref{eqn:leibniz_sigma}) holds. The second equality of (\ref{eqn:leibniz_sigma}), as well as both equalities in (\ref{eqn:leibniz_inf}) hold via an identical argument.
\end{proof}

\begin{proof}[Proof of Proposition \ref{prop:well_defined}]
First, note that the convexity of $\dev_{\sigma}(\cdot)$ implies that $\theta \mapsto \jrisk_{\sigma}(Z,\theta)$ is convex. We start by showing that $\jrisk_{\sigma}$ is also coercive, namely that $|\theta| \to \infty$ implies $\jrisk_{\sigma}(Z,\theta) \to \infty$. For all cases $\sigma \in [0,\infty]$, the non-negativity of $\dev_{\sigma}$ and $\eta$ trivially implies that $\theta \to \infty$ implies $\jrisk_{\sigma}(Z,\theta) \to \infty$, and thus we need only consider the negative direction, where $\theta \to -\infty$.

For the case of $\sigma = 0$, note that
\begin{align*}
\theta + \eta \exx_{\ddist}|Z-\theta| \geq \theta - \eta|\theta| + \eta\exx_{\ddist}|Z|.
\end{align*}
Clearly, with $\eta > 1$ the right-hand side grows without bound as $\theta \to -\infty$. For the case of $0 < \sigma = \infty$, writing $Z_{\theta} \defeq (Z-\theta)/\sigma$ for readability, the joint risk can be written conveniently as
\begin{align}\label{eqn:well_defined_1}
\jrisk_{\sigma}(Z,\theta) = \theta\left(1 - \frac{\eta}{\sigma}\exx_{\ddist}\atan\left(Z_{\theta}\right)\right) + \frac{\eta}{\sigma}\exx_{\ddist}Z\atan\left(Z_{\theta}\right) - \frac{\eta}{2}\exx_{\ddist}\log\left(1 + Z_{\theta}^{2}\right).
\end{align}
Since $\atan(\cdot)$ is monotonic (increasing) on $\RR$, bounded as $|\atan(\cdot)| < \pi/2$, and $\atan(u) \to \pi/2$ as $u \to \infty$, we have that $\exx_{\ddist}\atan(Z_{\theta}) \to \pi/2$ as $\theta \to -\infty$, by monotone convergence.\footnote{\citet[Thm.~1.6.7]{ash2000a}.} Thus, taking $\eta > 2\sigma/\pi$ ensures that eventually as $\theta \to \infty$, the first term on the right-hand side of (\ref{eqn:well_defined_1}) will become positive. Since this term grows linearly, it dominates the other unbounded term (which is logarithmic), and thus we have shown that $\jrisk_{\sigma}$ is coercive whenever $\sigma \geq 0$. Convexity and coercivity together imply that $\theta \mapsto \jrisk_{\sigma}(Z,\theta)$ takes its minimum on $\RR$; see \citet[Sec.~B.3.2]{bertsekas2015ConvexOpt} or \citet[Thm.~2.11]{barbu2012ConvOptBanach} for standard references.

The case of $\sigma = \infty$ is easy, since by direct inspection we can write
\begin{align*}
\jrisk_{\infty}(Z,\theta) = \theta(1 - 2\exx_{\ddist}Z) + \eta\theta^{2} + \eta\exx_{\ddist}Z^{2}.
\end{align*}
Since the sum of a strongly convex function and an affine function is strongly convex, we have that $\theta \mapsto \jrisk_{\infty}(Z,\theta)$ has a unique minimum on $\RR$.

It only remains to prove the uniqueness of $\theta_{Z}$ in the proposition statement for the case of $0 < \sigma < \infty$. The most direct way of doing this is to use the Leibniz property (\ref{eqn:leibniz_sigma}) proved in our helper Lemma \ref{lem:integration}, which in particular tells us that
\begin{align*}
\frac{d^{2}}{d\theta^{2}}\jrisk_{\sigma}(Z,\theta) = \frac{\eta}{\sigma}\exx_{\ddist}\dev^{\prime\prime}\left(\frac{Z-\theta}{\sigma}\right) > 0,
\end{align*}
where positivity follows from the fact that $\dev^{\prime\prime}(\cdot) = 1/(1+(\cdot)^{2}) > 0$. This implies strict convexity, and thus that the minimizer $\theta_{Z}$ is unique.
\end{proof}

\begin{proof}[Proof of Proposition \ref{prop:risk_props_axioms}]
We take the points in the statement of the proposition in order, one at a time. To being, the (joint) convexity of $\jrisk_{\sigma}$ follows from direct inspection, using the convexity of $\dev_{\sigma}$ for any $\sigma \in [0,\infty]$. With this fact in mind, note that the convexity of $\risk_{\sigma}$ can be checked easily as follows. For any $Z_1,Z_2 \in \ZZ$ and $\theta_1,\theta_2 \in \RR$, the definition and convexity of $\jrisk_{\sigma}$ immediately implies that for any $\alpha \in (0,1)$ we have
\begin{align*}
\risk_{\sigma}(\alpha Z_1 + (1-\alpha)Z_2) & \leq \jrisk_{\sigma}(\alpha Z_1 + (1-\alpha)Z_2, \alpha\theta_1 + (1-\alpha)\theta_2)\\
& \leq \alpha\jrisk_{\sigma}(Z_1,\theta_1) + (1-\alpha)\jrisk_{\sigma}(Z_2,\theta_2).
\end{align*}
Using the notation (and statement) of Proposition \ref{prop:well_defined}, we can set $\theta_1 = \theta_{Z_1} \in \RR$ and $\theta_2 = \theta_{Z_2} \in \RR$, and plugging this in to the above inequalities, we obtain
\begin{align*}
\risk_{\sigma}(\alpha Z_1 + (1-\alpha)Z_2) \leq \alpha\risk_{\sigma}(Z_1) + (1-\alpha)\risk_{\sigma}(Z_2),
\end{align*}
and thus both $\jrisk_{\sigma}$ and $\risk_{\sigma}$ are convex for any $\sigma \in [0,\infty]$. Note that this does not require the minima $\theta_{Z_1}$ and $\theta_{Z_2}$ to be unique, and thus holds for the $\sigma = 0$ case without issue. From Lemma \ref{lem:integration}, we also have that $|\jrisk_{\sigma}(\cdot,\cdot)| < \infty$ and $|\risk_{\sigma}(\cdot)| < \infty$, so both functions are proper convex. As for continuity, note that from Lemma \ref{lem:dev_lsc} and the continuity of $\dev_{\sigma}(\cdot)$ for all $\sigma \in [0,\infty]$, we can immediately infer that $\jrisk_{\sigma}$ is LSC. It is well-known that on Banach spaces, any proper convex LSC function is continuous and sub-differentiable on the interior of the effective domain.\footnote{Actually, via \citet[Prop.~2.16]{barbu2012ConvOptBanach}, this holds for every point in the algebraic interior of its effective domain; the fact stated follows as the algebraic interior contains the interior.} Since our integrability assumptions imply $\dom \jrisk_{\sigma} = \ZZ \times \RR$, the continuity and sub-differentiability of $\jrisk_{\sigma}$ is thus proved. To handle $\risk_{\sigma}$, take any sequence $(Z_k)$ converging to an arbitrarily chosen point $Z_{\ast} \in \ZZ$. Let $(\theta_k)$ be any sequence converging to $\theta_{Z_{\ast}} \in \RR$. Then by definition of $\risk_{\sigma}$ and continuity of $\jrisk_{\sigma}$, we have
\begin{align}\label{eqn:risk_props_axioms_0}
\limsup_{k \to \infty} \risk_{\sigma}(Z_k) \leq \limsup_{k \to \infty} \jrisk_{\sigma}(Z_k,\theta_k) = \lim\limits_{k \to \infty} \jrisk_{\sigma}(Z_k,\theta_k) = \risk_{\infty}(Z_{\ast},\theta_{Z_{\ast}}) = \risk_{\sigma}(Z_{\ast}).
\end{align}
The two ends of the inequality (\ref{eqn:risk_props_axioms_0}) imply that $\risk_{\sigma}$ is USC, via (\ref{eqn:lsc_liminf_condition}) and the relation of USC to LSC functions. On the effective domain of any convex USC function, the function is in fact continuous.\footnote{\citet[Prop.~3.2]{penot2012CWOD}.} Thus, we have that $\risk_{\sigma}$ is continuous. Furthermore, the sub-differentiability of $\risk_{\sigma}$ follows in the exact same fashion as for $\jrisk_{\sigma}$.

Next, for the monotonicity of the location term $Z \mapsto \theta_{Z}$ in (\ref{eqn:mrisk_loc_dev}), with $0 < \sigma \leq \infty$, recall that we can utilize the Leibniz properties (\ref{eqn:leibniz_sigma})--(\ref{eqn:leibniz_inf}) from the integration Lemma \ref{lem:integration}. To start, we know that for any $Z$, the corresponding $\theta_{Z}$ must satisfy the following first-order optimality condition:
\begin{align}\label{eqn:risk_props_axioms_1}
\text{If $\sigma < \infty$:} \quad \exx_{\ddist}\dev_{\sigma}^{\prime}(Z-\theta_{Z}) = \frac{\sigma}{\eta}. \qquad \text{If $\sigma = \infty$:} \quad \theta_{Z} = \exx_{\ddist}Z - \frac{1}{2\eta}.
\end{align}
The desired monotonicity property is obvious for the $\sigma = \infty$ case using (\ref{eqn:risk_props_axioms_1}). As for the case of $0 < \sigma < \infty$, it is evident from the second equality of (\ref{eqn:leibniz_sigma}) that the function $\theta \mapsto \exx_{\ddist}\dev_{\sigma}(Z-\theta)$ is monotonically decreasing on $\RR$. Thus, if we assume $Z_1 \leq Z_2$ almost surely but $\theta_{Z_1} > \theta_{Z_2}$, the first order optimality combined with monotonicity implies
\begin{align*}
\frac{\sigma}{\eta} = \exx_{\ddist}\dev_{\sigma}(Z_1-\theta_{Z_1}) < \exx_{\ddist}\dev_{\sigma}(Z_1-\theta_{Z_2}) \leq \exx_{\ddist}\dev_{\sigma}(Z_2-\theta_{Z_2}) = \frac{\sigma}{\eta},
\end{align*}
which is a contradiction. Thus, $\theta_{Z_1} \leq \theta_{Z_2}$ as desired for the $0 < \sigma < \infty$ case as well.

The translation-equivariance property of $Z \mapsto \theta_{Z}$ follows from direct inspection using the condition (\ref{eqn:risk_props_axioms_1}), that is for $0 < \sigma < \infty$, we trivially have
\begin{align*}
\frac{\sigma}{\eta} = \exx_{\ddist}\dev_{\sigma}(Z-\theta_{Z}) = \exx_{\ddist}\dev_{\sigma}(Z+a-(\theta_{Z}+a)),
\end{align*}
and thus $\jrisk_{\sigma}(Z+a,\theta_{Z}+a) = \risk_{\sigma}(Z+a)$. Since Proposition \ref{prop:well_defined} guarantees that the minimizer of $\theta \mapsto \jrisk_{\sigma}(Z,\theta)$ is unique, we can safely write $\theta_{Z+a} = \theta_{Z}+a$. The proof for the $\sigma = \infty$ case is analogous.

Finally, to prove that $\risk_{\sigma}$ is not in general monotonic, we give a concrete example of $Z_1$ and $Z_2$ such that $Z_1 \leq Z_2$ almost surely, but $\risk_{\sigma}(Z_1) > \risk_{\sigma}(Z_2)$. For simplicity, consider the case of $\sigma = \infty$, where $\dev_{\infty}(\cdot) = (\cdot)^{2}$, and for any $\eta > 0$ direct inspection shows that
\begin{align}\label{eqn:mean_variance_special_case}
\risk_{\infty}(Z) = \exx_{\ddist}Z + \eta\vaa_{\ddist}Z - \frac{1}{4\eta}.
\end{align}
That is, the special case of $\sigma = \infty$ is equivalent to the mean-variance risk function of classical portfolio theory, dating back to \citet{markowitz1952a}. The random variables $Z_1$ and $Z_2$ are constructed as follows. Let $c_1$ and $c_2$ be the respective centers, $w_1$ and $w_2$ the respective widths, and $v_1 < w_{1}^{2}$ and $v_2 < w_{2}^{2}$ the respective scaling factors of $Z_1$ and $Z_2$, which are characterized as
\begin{align*}
\prr\left\{Z_j = c_j-w_j\right\} = \prr\left\{Z_j = c_j+w_j\right\} = \frac{v_j}{2w_{j}^{2}}, \qquad \prr\{Z_j = c_j\} = 1-\prr\{Z_j \neq c_j\}
\end{align*}
for each $j \in \{1,2\}$. Note that our assumptions imply $0 < \prr\{Z_j = c_j\} < 1$, and direct inspection shows that $\exx Z_j = c_j$ and $\vaa Z_j = v_j$, again for each $j \in \{1,2\}$. As a simple concrete example, note that setting $c_2 = c_1 + w_1 + w_2$ guarantees $Z_1 \leq Z_2$ with probability $1$. From the equality (\ref{eqn:mean_variance_special_case}) given above, the difference in risks can be written as
\begin{align*}
\risk_{\infty}(Z_1)-\risk_{\infty}(Z_2) & = \exx Z_1 - \exx Z_2 + \eta\left( \vaa Z_1 - \vaa Z_2 \right)\\
& = -(w_1+w_2) + \eta(v_1 - v_2).
\end{align*}
Thus, $\risk_{\infty}(Z_1) > \risk_{\infty}(Z_2)$ holds whenever $v_1 - v_2 > w_1 + w_2$. For concreteness, say for some $\varepsilon > 0$, we fix the variance factors to $v_1 = w_{1}^{2} - \varepsilon$ and $v_2 = w_{2}^{2} - \varepsilon$ respectively. Then the condition simplifies to $w_{1}^{2} > w_1 + w_2 + w_{2}^{2}$. As an example, setting $w_1=2$ and $w_2 = 1/2$, the condition holds, implying $\risk_{\infty}(Z_1) > \risk_{\infty}(Z_2)$, despite the fact that $Z_1 \leq Z_2$. This gives us a simple but intuitive example where monotonicity of $\risk_{\sigma}$ does not hold, and concludes the proof.
\end{proof}

\subsection{Proofs for section \ref{sec:learning}}\label{sec:proofs_learning}

Recall that our basic probabilistic setup for the learning problem has an underlying probability space $(\Omega,\FF,\ddist)$, a hypothesis class $\HH$, and a random loss $\loss(h)$ indexed by $\HH$. That is, we consider any $\FF$-measurable function $\loss(h;\cdot): \Omega \to \RR$ as a loss. When a particular realization $\omega \in \Omega$ is important, we will write $\loss(h;\omega)$, but otherwise, for readability we will typically write $\loss(h) \defeq \loss(h;\cdot)$. Our basic integrability assumption, carried over from section \ref{sec:mrisk}, is that of square-$\ddist$-integrability, which in the context of losses is written explicitly as
\begin{align*}
\exx_{\ddist}|\loss(h)|^{2} = \int_{\Omega}|\loss(h;\omega)|^{2} \, \ddist(\dif\omega) < \infty
\end{align*}
for all $h \in \HH$. This requirement is made in assumption \ref{asmp:background} in the main text. It follows that $\{\loss(h): h \in \HH\} \subset \LL_{2}(\Omega,\FF,\ddist)$. Thus the map $h \mapsto \loss(h)$ takes us from $\HH$ to $\LL_{2}(\Omega,\FF,\ddist)$.

\paragraph{Loss-specific terminology}
To ensure our use of formal terms is clear, we apply the definitions of section \ref{prop:bg_prelim} to losses here. We shall typically suppress the dependence on $\omega \in \Omega$ in directional derivatives and gradients, writing $\loss_{r}^{\prime}(h;g) \defeq \loss_{r}^{\prime}(h;g,\cdot)$, $\loss^{\prime}(h;g) \defeq \loss^{\prime}(h;g,\cdot)$, and $\loss^{\prime}(h) \defeq \loss^{\prime}(h;\cdot)$. Let $H \subset \HH$ be an open set. We say that \term{$\loss$ is radially differentiable} at $h \in H$ if the radial derivative $\loss_{r}^{\prime}(h;g)$ exists for all directions $g \in \HH$, $\ddist$-almost surely. We say that \term{$\loss$ is directionally differentiable} at $h \in H$ if the directional derivative $\loss^{\prime}(h;g)$ exists for all directions $g \in \HH$, $\ddist$-almost surely. On this ``good'' event of probability $1$, if the map $g \mapsto \loss_{r}^{\prime}(h;g)$ is linear and continuous, we say \term{$\loss$ is Gateaux differentiable} at $h$, and if the map $g \mapsto \loss^{\prime}(h;g)$ is linear and continuous, we say \term{$\loss$ is Hadamard differentiable} at $h$. When we say that \term{$\loss$ is (Fr\'{e}chet) differentiable} at $h \in H$, we mean that there exists a function $\loss^{\prime}(h)(\cdot):\HH \to \RR_{+}$ that is linear, continuous, and which satisfies (\ref{eqn:deriv_frechet}) $\ddist$-almost surely.\footnote{In our particular setting with losses here, the norm used in the numerator of (\ref{eqn:deriv_frechet}) will be the $\LL_{2}$ norm.} We say that \term{$\loss$ is $\lambda$-Lipschitz at $h \in \HH$} if there exists a $\delta > 0$ such that $\|h-h^{\prime}\| < \delta \implies \|\loss(h)-\loss(h^{\prime})\| \leq \lambda\|h-h^{\prime}\|$. With the running assumption about second moments, this amounts to requiring
\begin{align}\label{eqn:loss_local_lipschitz}
\|h-h^{\prime}\| < \delta \implies \exx_{\ddist}|\loss(h)-\loss(h^{\prime})|^{2} \leq \lambda^{2}\|h-h^{\prime}\|^{2}.
\end{align}
We say that \term{$\loss$ is weakly $\lambda$-smooth at $h \in H$} if $\loss$ is Gateaux differentiable and the map $h \mapsto \loss_r^{\prime}(h;\cdot)$ is $\lambda$-Lipschitz $\ddist$-almost surely at $h$. That is, if for small enough $\delta > 0$ we have
\begin{align}\label{eqn:loss_smooth}
\|h-h^{\prime}\| < \delta \implies \| \loss_r^{\prime}(h;\cdot)-\loss_r^{\prime}(h^{\prime};\cdot) \| \leq \lambda\|h-h^{\prime}\|.
\end{align}
Note that the norm used here is the operator norm applied to the linear map $\loss_r^{\prime}(h;\cdot)-\loss_r^{\prime}(h^{\prime};\cdot)$.

\subsubsection{Weak convexity of joint composition function}

The joint risk function $\jrisk_{\sigma}(\loss(h),\theta)$ can be written as a simple composition $(h,\theta) \mapsto (\loss(h),\theta) \mapsto \theta + \eta \exx_{\ddist}\dev_{\sigma}(\loss(h)-\theta)$. For any $0 \leq \sigma < \infty$ and any smooth loss, using the preliminary results established section \ref{sec:bg_composition}, it is straightforward to show the weak convexity of this composite function.
\begin{prop}\label{prop:learn_weak_convexity}
Let the hypothesis class $\HH$ be Banach. Let the loss $\loss$ be locally Lipschitz and weakly $\lambda^{\prime}$-smooth on $\HH$.  Then, for any $0 \leq \sigma < \infty$, defining a $\sigma$-dependent factor $\lambda_{\sigma}$ as
\begin{align*}
\lambda_{\sigma} \defeq
\begin{cases}
1, & \text{if } \sigma = 0\\
\pi/(2\sigma), & \text{if } 0 < \sigma < \infty
\end{cases}
\end{align*}
we have that $(h,\theta) \mapsto \jrisk_{\sigma}(\loss(h),\theta)$ is $\gamma$-weakly convex with $\gamma = (1+\eta\lambda_{\sigma})\max\{1,\lambda^{\prime}\}$.
\end{prop}
\begin{proof}
Recall the generic result given in Proposition \ref{prop:weak_convexity_comp} for the weak convexity of generic composite functions. Our proof here amounts to checking that the assumptions of Proposition \ref{prop:weak_convexity_comp} are satisfied for the composition $f \defeq \jrisk_{\sigma} \circ f_1$, with $f_1: \HH \times \RR \to \LL_{2}(\Omega,\FF,\ddist) \times \RR$ defined as $f_1(h,\theta) \defeq (\loss(h),\theta)$.

To start, let us consider the properties of $f_1$. Since $\loss$ is locally Lipschitz and Gateaux differentiable, it follows that $\loss$ is also Hadamard differentiable.\footnote{\citet[Prop.~2.25]{penot2012CWOD}.} Since the map $h \mapsto \loss_r^{\prime}(h;\cdot) = \loss^{\prime}(h;\cdot)$ is continuous (by weak smoothness), it follows that $\loss$ is (Fr\'{e}chet) differentiable.\footnote{\citet[Prop.~2.51]{penot2012CWOD}.} Since $(h,\theta) \mapsto f_1(h,\theta)=(\loss(h),\theta)$ just passes $\theta$ through the identity, trivially the second component is also differentiable, and the differentiability of both components thus implies $f_1$ is differentiable.\footnote{\citet[Prop.~2.52]{penot2012CWOD}.} Furthermore, the local Lipschitz property of $\loss$ is clearly retained by $f_1$. Evaluating the gradients we have $f_1^{\prime}(h,\theta)(g,r) = (\loss^{\prime}(h)(g),r)$, and thus using a typical product space norm we have
\begin{align*}
\|f_1^{\prime}(h_1,\theta_1)-f_1^{\prime}(h_2,\theta_2)\| = \|\loss^{\prime}(h_1)-\loss^{\prime}(h_2)\| + |\theta_1-\theta_2|.
\end{align*}
By weak smoothness of $\loss$, it follows that $f_1$ is $\max\{1,\lambda^{\prime}\}$-smooth $\ddist$-almost surely, where smoothness is in the sense defined in section \ref{sec:bg_composition}.

Next, let us look at properties of $\jrisk_{\sigma}$. Note that $\dev_{\sigma}(\cdot)$ is trivially $1$-Lipschitz in the case of $\sigma = 0$, and for $0 < \sigma < \infty$, the $\pi/2$-Lipschitz property of $\dev$ defined in (\ref{eqn:dev_handy}) implies that $\dev_{\sigma}(\cdot)$ is $\pi/(2\sigma)$-Lipschitz. With these basic facts in place, it follows that
\begin{align*}
|\jrisk_{\sigma}(Z_1,\theta_1)-\jrisk_{\sigma}(Z_2,\theta_2)| \leq \left(1 + \eta\lambda_{\sigma}\right)|\theta_1-\theta_2| + \eta\lambda_{\sigma}\exx_{\ddist}|Z_1-Z_2|,
\end{align*}
where the $\sigma$-dependent Lipschitz coefficient $\lambda_{\sigma}$ is as defined in the statement of the desired result. To obtain bounds in terms of the correct norm, note that
\begin{align*}
\exx_{\ddist}|Z_1 - Z_2| \leq \ddist(\Omega)\sqrt{\exx_{\ddist}|Z_1 - Z_2|^{2}} = \sqrt{\exx_{\ddist}|Z_1 - Z_2|^{2}},
\end{align*}
which follows from the fact that $\ddist$ is a probability, and a simple application of H\"{o}lder's inequality.\footnote{\citet[Sec.~2.4]{ash2000a}.} Plugging this into the previous inequality, and noting that it holds for any choice of $Z_1,Z_2 \in \LL_2(\Omega,\FF,\ddist)$ and $\theta_1,\theta_2 \in \RR$, it follows that $\jrisk_{\sigma}$ is $(1+\eta\lambda_{\sigma})$-Lipschitz on $\LL_2(\Omega,\FF,\ddist) \times \RR$, and $f_1(\HH \times \RR) \subset \dom f_2$. Furthermore, from Proposition \ref{prop:risk_props_axioms}, we have that $\jrisk_{\sigma}$ is convex.

Taking the above points together, if we consider the good event of probability $1$ where $f_1$ satisfies the desired properties, direct application of Proposition \ref{prop:weak_convexity_comp} to the map $(h,\theta) \mapsto (\jrisk_{\sigma} \circ f_1)(h,\theta)$ yields the desired result.
\end{proof}
\begin{rmk}
The result in the preceding Proposition \ref{prop:learn_weak_convexity} is rather useful, and it does not require the loss to be convex. When the loss is convex, the analysis becomes somewhat simpler and stronger arguments are naturally possible; composite risks under convex losses and convex, monotonic risk functions is the setting considered by \citet[Sec.~3.2]{ruszczynski2006a}, for example.
\end{rmk}

We have established conditions under which the intermediate joint objective $\jrisk_{\sigma}(\loss(h),\theta)$ is weakly convex, and characterized this weak convexity with respect to properties of the underlying risk function and data distribution. Since the data distribution $\ddist$ is unknown, we can never actually compute $\jrisk_{\sigma}(\loss(h),\theta)$. Any learning algorithm will only have access to feedback of a stochastic nature which provides incomplete, noisy information. Our next task is to establish conditions under which the feedback available to the learner is ``good enough'' to ensure reasonable performance guarantees.

\subsubsection{Unbiased stochastic feedback}

In considering stochastic feedback, recall that $\jrisk_{\sigma}(\loss(h),\theta) = \exx_{\ddist}(f_2 \circ F_1)(h,\theta)$, with $F_1$ and $f_2$ given by (\ref{eqn:jrisk_composition}) in the main text. For each $h \in \HH$ and $\theta \in \RR$, the value $F_1(h,\theta)$ returned by $F_1$ is a random vector. We shall assume that for any $h \in \HH$, the learner can obtain independent random samples of the loss $\loss(h)$ and the associated gradient $\loss^{\prime}(h)$. Since $F_{1}^{\prime}(h,\theta) = (\loss^{\prime}(h),1)$ by which we mean $F_{1}^{\prime}(h,\theta)(g,r) = (\loss^{\prime}(h)(g),r)$ for all $g \in \HH$ and $r \in \RR$, clearly the learner can also independently sample from $F_1(h,\theta)$ and $F_{1}^{\prime}(h,\theta)$. Sub-differentiability is already guaranteed by Proposition \ref{prop:risk_props_axioms}, and since $\dev_{\sigma}$ and $\eta$ are by design known to the learner, they can readily acquire an element from $\partial f_2(u,\theta)$. Thus if $(\loss(h),\theta) \in \inter(\dom f_2)=\RR^{2}$ $\ddist$-almost surely, it follows that the learner can sample from $\partial f_2(\loss(h),\theta) \circ F_{1}^{\prime}(h,\theta)$. This is the stochastic \term{feedback} available to the learner, and when we ask that it be ``good enough,'' this means we require it to be an unbiased estimator of the (Clarke) sub-differential of $\jrisk_{\sigma}$. The following result gives mild conditions under which this is achieved.
\begin{prop}\label{prop:learn_unbiased}
Under the conditions of Proposition \ref{prop:learn_weak_convexity}, for any $h \in \HH$ and $\theta \in \RR$, as long as $\exx_{\ddist}|\loss^{\prime}(h;\cdot)| < \infty$, the stochastic sub-differential is an unbiased estimator in that
\begin{align*}
\exx_{\ddist}\left[ \partial f_2(\loss(h),\theta) \circ F_{1}^{\prime}(h,\theta) \right] \subset \partial_{\clarke}\jrisk_{\sigma}(\loss(h),\theta),
\end{align*}
and this holds for any choice of $0 \leq \sigma < \infty$.
\end{prop}
\begin{proof}
Using the weak smoothness of $\loss$, with probability $1$, the map $(h,\theta) \mapsto F_{1}^{\prime}(h,\theta)$ is continuous, plus $F_1$ is locally Lipschitz and $\partial_{\clarke}F_{1}(h,\theta) = \{F_{1}^{\prime}(h,\theta)\}$.\footnote{\citet[Prop.~5.6]{penot2012CWOD}.} Furthermore, since $F_{1}^{\prime}(h,\theta)(g,r) = (\loss^{\prime}(h)(g),r)$, the linearity of $\loss^{\prime}(h)(\cdot)$ implies that $F_{1}^{\prime}(h,\theta)(\HH \times \RR) = \RR^{2}$ $\ddist$-almost surely. Since $F_1$ and $f_2$ are (locally) Lipschitz, the facts we have just laid out imply a strong chain rule.\footnote{\citet[Prop.~5.13]{penot2012CWOD}.} That is, it holds $\ddist$-almost surely that
\begin{align}
\nonumber
\partial_{\clarke}(f_2 \circ F_1)(h,\theta) & = \partial_{\clarke}f_2(\loss(h),\theta) \circ F_{1}^{\prime}(h,\theta)\\
\label{eqn:joint_unbiased_0}
& = \partial f_2(\loss(h),\theta) \circ F_{1}^{\prime}(h,\theta),
\end{align}
where the second equality follows from the convexity of $f_2$.

Next, the Lipschitz property of $\dev_{\sigma}$ implies that $f_2$ is $\lambda$-Lipschitz for some $\lambda>0$; the actual value is not important for this proof. Using the Lipschitz continuity of $f_2$ and the fact that $\partial_{\clarke}F_{1}(h,\theta) = \{F_{1}^{\prime}(h,\theta)\}$ implies $\partial_{\clarke}\loss(h) = \{\loss^{\prime}(h)\}$, we have
\begin{align*}
|(f_2 \circ F_1)_{\clarke}^{\prime}(h,\theta)(g,r)| & \leq \limsup_{(\alpha,h^{\prime},\theta^{\prime}) \to (0_{+},h,\theta)} \frac{\left| (f_2 \circ F_1)(h^{\prime}+\alpha g,\theta^{\prime}+\alpha r)-(f_2 \circ F_1)(h^{\prime},\theta^{\prime}) \right|}{\alpha}\\
& \leq \lambda \left( \limsup_{(\alpha,h^{\prime}) \to (0_{+},h)}\frac{\left|\loss(h^{\prime}+\alpha g)-\loss(h^{\prime})\right|}{\alpha} + |r| \right)\\
& = \lambda \left( |\loss^{\prime}(h)(g)| + |r| \right).
\end{align*}
Thus, using $\jrisk_{\sigma}(\loss(h),\theta) = \exx_{\ddist}(f_2 \circ F_1)(h,\theta)$ and $\exx_{\ddist}|\loss^{\prime}(h)(g)| < \infty$ for all $h,g \in \HH$, we have
\begin{align}
\label{eqn:joint_unbiased_1}
\exx_{\ddist} \left( f_2 \circ F_1 \right)_{\clarke}^{\prime}(h,\theta;\cdot) = (\jrisk_{\sigma})_{\clarke}^{\prime}(\loss(h),\theta;\cdot)
\end{align}
by a direct application of dominated convergence.\footnote{\citet[Thm.~1.6.9]{ash2000a}.}

To conclude, taking $G(h,\theta) \in \partial f_2(\loss(h),\theta) \circ F_{1}^{\prime}(h,\theta)$, by (\ref{eqn:joint_unbiased_0}) it follows that we have $G(h,\theta) \in \partial_{\clarke}(f_2 \circ F_1)(h,\theta)$, and thus by definition of the Clarke sub-differential, monotonicity of the integral, and finally (\ref{eqn:joint_unbiased_1}), we obtain
\begin{align}
\label{eqn:joint_unbiased_2}
\exx_{\ddist}G(h,\theta)(\cdot) \leq \exx_{\ddist} (f_2 \circ F_1)_{\clarke}^{\prime}(h,\theta;\cdot) = (\jrisk_{\sigma})_{\clarke}^{\prime}(\loss(h),\theta;\cdot).
\end{align}
Linearity of $\exx_{\ddist}G(h,\theta)(\cdot)$ follows from the linearity of both $G(h,\theta)(\cdot)$ and the integral. Finally, applying (\ref{eqn:joint_unbiased_2}) we have
\begin{align*}
\sup_{\|(u,r)\|=1} \exx_{\ddist}G(h,\theta)(g,r) \leq \sup_{\|(g,r)\|=1} (\jrisk_{\sigma})_{\clarke}^{\prime}(\loss(h),\theta;u,r) < \infty,
\end{align*}
where finiteness holds because $\jrisk_{\sigma}$ is locally Lipschitz.\footnote{\citet[Prop.~5.2(b)]{penot2012CWOD}.} Thus $\exx_{\ddist}G(h,\theta) \in (\HH \times \RR)^{\ast}$, and with (\ref{eqn:joint_unbiased_2}) we have $\exx_{\ddist}G(h,\theta) \in \partial_{\clarke}\jrisk_{\sigma}(\loss(h),\theta)$ as desired.
\end{proof}
\begin{rmk}
The validity of interchanging the operations of (sub-)differentiation and expectation is a topic of fundamental importance in stochastic optimization and statistical learning theory. A useful, modern reference on this topic is included in \citet[Ch.~2]{ruszczynski2003a}. A classical reference is \citet{rockafellar1982a}; see also \citet{rockafellar1968a} for a look at measurability of convex integrands. The interchangeability problem appears in various places in the literature over the years, see for example \citet{shapiro1994a} as well as \citet[Ch.~3, Rmk.~2.2]{kall2005SLP}. See also \citet[Eqn.~(3.9)]{ruszczynski2006a}, who refer to generalized versions of a classic result due to \citet{strassen1965a}.
\end{rmk}

\subsubsection{Proximity to a nearly-stationary point}

\begin{proof}[Proof of Theorem \ref{thm:learn_stationary}]
With all the results established thus far, this proof has just two simple parts. First, we need to show that the objective function of interest is weakly convex, and that we have access to unbiased estimates of the sub-differential; this is done here. This is done using the critical preparatory results in Propositions \ref{prop:learn_weak_convexity} and \ref{prop:learn_unbiased}. Once this has been established, the remaining part just has us applying recent results from the literature for non-asymptotic control of the envelope gradient norm.

To begin, the assumptions of Proposition \ref{prop:learn_weak_convexity} are satisfied by \ref{asmp:background}, which ensures that $\jrisk_{\sigma}$ is $\gamma$-weakly convex for $\gamma = (1+\eta\lambda_{\sigma})\max\{1,\lambda\}$. Furthermore, the $\ddist$-integrability assumption on $\loss^{\prime}(h)$ lets us use Proposition \ref{prop:learn_unbiased} to ensure that feedback drawn from (\ref{eqn:sgd_feedback}) is such that $\exx_{\ddist}G_t \in \partial_{\clarke}\jrisk_{\sigma}(h_t,\theta_t)$ for all $t$. Furthermore, using \ref{asmp:feedback_cond_indep} implies that $\exx[G_t \cond G_{[t-1]}] \in \partial_{\clarke}\jrisk_{\sigma}(h_t,\theta_t)$ for all $t$, since Algorithm \ref{algo:sgd} uses $G_t$ sampled via (\ref{eqn:sgd_feedback}).

The desired result follows from an application of \citet[Thm.~3.1]{davis2019b}, where their objective function $f$ corresponds to our $\jrisk_{\sigma}$.\footnote{It also relies on the observation that a \emph{proximal} stochastic sub-gradient update using the indicator function of $C$ as a regularizer is equivalent to the \emph{projected} sub-gradient update we do here.} While their proof is given for the case of $\HH = \RR^{d}$, using assumption \ref{asmp:hilbert}, if we leverage our characterization of weak convexity (Proposition \ref{prop:weak_convexity_characterized}), and replace their Lemma 2.2 with our (\ref{eqn:env_gradient_form}), it is straightforward to see that their insights extend to arbitrary Hilbert spaces using the usual norm induced by the inner product. Thus with the moment bound \ref{asmp:feedback_mnt_bd} in hand, the generalized result can be applied to Algorithm \ref{algo:sgd}, for objective function $\jrisk_{\sigma}$, which has just been proved to be $\gamma$-weakly convex. The desired result follows immediately.
\end{proof}

\section{Helper results}\label{sec:helpers}

In this section, we provide some standard results that are leveraged in the main paper.

\subsection{Useful results based on Lipschitz properties}

Let $\XX$ be a normed linear space, and let $f:\XX \to \overbar{\RR}$ be convex and $\lambda$-Lipschitz. If $f$ is sub-differentiable at a point $x$, then using the definition of the sub-differential, we have that
\begin{align*}
|\langle x^{\prime}-x, \partial f(x) \rangle| \leq |f(x^{\prime})-f(x)| \leq \lambda\|x^{\prime}-x\|.
\end{align*}
It immediately follows that
\begin{align}\label{eqn:lip_convex_subgrad_bound}
\|\partial f(x)\| \leq \lambda.
\end{align}
That is, all sub-gradients of $f$ at $x$ have norm no greater than the Lipschitz coefficient $\lambda$.

Let $\XX$ and $\YY$ be Banach spaces, and let $f:\XX \to \YY$ be differentiable on $U \subset \XX$, an open set. Further, assume that the derivative is $\lambda$-Lipschitz on $U$, that is, for each $x, x^{\prime} \in U$, we have $\|f^{\prime}(x)-f^{\prime}(x^{\prime})\| \leq \lambda\|x-x^{\prime}\|$. First-order Taylor approximations have direct analogues in this general setting, as the following result shows.\footnote{See for example \citet[Sec.~7.3, Prop.~2--3]{luenberger1969Book} and \citet[Ch.~1]{nesterov2004ConvOpt}.}
\begin{prop}\label{prop:smooth_taylor_1st}
Let $f:\XX \to \YY$ be differentiable on an open set $U \subset \XX$, with $\XX$ and $\YY$ assumed to be Banach. If $f^{\prime}(\cdot)$ is $\lambda$-Lipschitz on $U$, then for any $x,u \in U$ such that $x+u \in U$, we have
\begin{align*}
\|f(x+u) - f(x) - f^{\prime}(x)(u)\| \leq \frac{\lambda}{2}\|u\|^{2}.
\end{align*}
\end{prop}

\subsection{Radial derivatives of convex functions}

Say a function $f: \VV \to \overbar{\RR}$ is convex. Take any $u, v \in \dom f$, and any scalar $c \geq 0$ such that $v + c(v-u) \in \dom f$. Then writing $u^{\prime} \defeq v + c(v-u)$, note that
\begin{align*}
v = \frac{1}{1+c}\left(u^{\prime}+cu\right) = (1-\beta)u^{\prime} + \beta u,
\end{align*}
where $\beta \defeq c/(1+c) \in [0,1)$. By convexity we have $f((1-\beta)u^{\prime} + \beta u) \leq (1-\beta)f(u^{\prime}) + \beta f(u)$. Filling in definitions and rearranging we have
\begin{align}\label{eqn:convexity_chord_slopes}
f(v + c(v-u)) - f(v) \geq c(f(v)-f(u)).
\end{align}
Note that this can be done for any pair of $u,v$ and scalar $c$ that keeps the relevant points on the domain. Clearly this property is necessary for convexity, but it is in fact also sufficient.\footnote{For example, see \citet[Thm.~3.1.1]{nesterov2004ConvOpt}.}

For any function $f:\VV \to \overbar{\RR}$ and open set $U \subset \XX$, fix a point $x \in U$. We denote the difference quotient of $f$ at $x$, incremented in the direction $u$, modulated by scalar $\alpha \neq 0$ as
\begin{align}\label{eqn:diffquot_defn}
q(\alpha) \defeq q(\alpha;f,x,u) \defeq \frac{f(x + \alpha u)-f(x)}{\alpha}.
\end{align}
Consider the map $g(t) \defeq f(x + t u)-f(x)$, with all elements but $t \geq 0$ fixed. When $f$ is convex, direct inspection immediately shows that $t \mapsto g(t)$ is convex. For any $0 \leq t_{1} < t_{2}$, take some $t^{\prime} \in (t_1,t_2)$. Clearly, there exists a $\beta \in (0,1)$ such that $t^{\prime} = \beta t_1 + (1-\beta)t_2$. Then, we have
\begin{align*}
\frac{g(t^{\prime})-g(t_1)}{t^{\prime}-t_1} = \frac{g(\beta t_{1} + (1-\beta)t_2)-g(t_1)}{(1-\beta)(t_2-t_1)} & \leq \frac{\beta g(t_1) + (1-\beta)g(t_2)-g(t_1)}{(1-\beta)(t_2-t_1)} = \frac{g(t_2)-g(t_1)}{t_2-t_1},
\end{align*}
where the inequality follows from convexity of $g$. If we use this inequality in the special case of $t_1 = 0$, alongside the basic relation $q(\alpha) = (g(\alpha)-g(0))/\alpha$, it immediately follows that $\alpha \mapsto q(\alpha)$ is monotonic (non-decreasing) on the positive reals. Furthermore, the set $\{q(\alpha): \alpha > 0\}$ is bounded below. To see this, take some $\gamma > 0$ small enough that $x - \gamma u \in \dom f$, and note that by direct application of convexity and the basic property (\ref{eqn:convexity_chord_slopes}), it follows that
\begin{align*}
\left(\frac{\alpha}{\gamma}\right)(f(x)-f(x-\gamma u)) \leq f\left( x + \frac{\alpha}{\gamma}(x-(x-\gamma u)) \right) - f(x) = f(x + \alpha u) - f(x).
\end{align*}
That is, dividing both sides by $\alpha$, we have
\begin{align}\label{eqn:convex_diffquot_lowerbd}
\left(\frac{1}{\gamma}\right)(f(x)-f(x-\gamma u)) \leq \frac{f(x + \alpha u) - f(x)}{\alpha} = q(\alpha;f,u).
\end{align}
Since the choice of $\gamma > 0$ depends only on $x$ and $u$, and is free of $\alpha$, it follows that the set $\{q(\alpha): \alpha > 0\}$ is bounded below, as desired. Using this boundedness alongside the monotonicity of $\alpha \mapsto q(\alpha)$, we have that the infimum is finite. Thus, recalling the definition (\ref{eqn:ddv_gateaux}) of the radial derivative of $f$ at $x$ in the direction $u$, since we have
\begin{align}\label{eqn:convex_dder_upper_defn}
f_{r}^{\prime}(x;u) = \lim\limits_{\alpha \to 0_{+}} \frac{f(x + \alpha u)-f(x)}{\alpha} = \inf\left\{ q(\alpha;f,x,u) : \alpha > 0 \right\},
\end{align}
it follows immediately that the radial derivative always exists (i.e., $f^{\prime}(x;u) \in \RR$). Note also that using convexity, direct inspection shows that for all $u$ we have
\begin{align}\label{eqn:convex_dder_upperbd}
f(u) - f(x) \geq f^{\prime}(x;u-x).
\end{align}
Furthermore, it is easily verified that whenever $x \in \dom f$, the map $u \mapsto f^{\prime}(x;u)$ is sub-additive and positively homogeneous, i.e., a sub-linear functional.\footnote{This means that the Hahn-Banach theorem can be applied to construct a linear functional $g$ bounded above as $g(u) \leq f^{\prime}(x;u)$, for all $u$. See for example \citet[Sec.~5.4]{luenberger1969Book} or \citet[Thm.~3.4.2]{ash2000a}. This $g$ is not necessarily a sub-gradient of $f$ at $x$, since it need not be continuous in general; such functions are sometimes called \term{algebraic sub-gradients} \citep[Sec.~3]{ruszczynski2006a}.} The basic facts of interest here are summarized in the following proposition.
\begin{prop}[Difference quotients for convex functions]\label{prop:convex_diffquot_dder}
Let $\VV$ be a vector space. If function $f: \VV \to \overbar{\RR}$ is proper and convex, then it is radially differentiable on $\inter(\dom f)$.
\end{prop}
\begin{proof}
The desired result follows immediately from previous discussion leading up to (\ref{eqn:convex_dder_upper_defn}), and the fact that if $x$ is an interior point of the effective domain of $f$, it follows that for any $u \in \VV$, we can find a $\gamma > 0$ small enough that $x - \gamma u \in \dom f$, which means we can apply the lower bound of (\ref{eqn:convex_diffquot_lowerbd}) to the difference quotients $q(\alpha;f,x,u)$ indexed by $\alpha > 0$.
\end{proof}

\subsection{Loss example}\label{sec:theory_examples}

\begin{ex}
While stated with a somewhat high degree of abstraction, let us give a concrete example to emphasize that the assumptions of Proposition \ref{prop:learn_weak_convexity} are readily satisfied under natural and important learning settings. Consider the regression problem, where we observe random pairs $(X,Y) \sim \ddist$, assuming that $X$ is a finite-dimensional real-valued random vector, and $Y$ is a real-valued random variable, related to the inputs by the relation $Y = h^{\ast}(X) + \epsilon$, where $\epsilon$ is a zero-mean random noise term. For simplicity, let $h^{\ast}$ be a continuous linear map, and let $\HH$ be the set of all continuous linear maps on the space that $X$ is distributed over. Finally, let the loss by the squared error, such that
\begin{align*}
\loss(h) & = (h(X)-Y)^{2} = (\langle X, h-h^{\ast} \rangle - \epsilon)^{2}\\
\loss^{\prime}(h)(u) & = 2(\langle X, h-h^{\ast} \rangle - \epsilon)\langle u, X\rangle.
\end{align*}
Since we make almost no assumptions on the nature of the underlying noise distribution, clearly both the losses and the ``gradients'' can be unbounded and heavy-tailed. Fix any $h_0 \in \HH$, and note that for any $h \in \HH$, we have
\begin{align*}
\loss(h)-\loss(h_0) & = \langle X, h-h^{\ast} \rangle^{2} - \langle X, h_0-h^{\ast} \rangle^{2} - 2\epsilon \langle X, h-h_0 \rangle\\
& = \langle X, h-h_0 \rangle \langle X, (h-h^{\ast}) + (h_0-h^{\ast}) \rangle - 2\epsilon \langle X, h-h_0 \rangle.
\end{align*}
Absolute values can be bounded above as
\begin{align*}
|\loss(h)-\loss(h_0)| \leq \|X\|^{2}\|h-h_0\|\left( \|h-h^{\ast}\| + \|h_0-h^{\ast}\|\right) + 2|\epsilon|\|X\|\|h-h_0\|.
\end{align*}
It follows immediately that as long as $\exx_{\ddist}\|X\|^{4} < \infty$, we have that the local Lipschitz property (\ref{eqn:loss_local_lipschitz}) of the loss is satisfied, for arbitrary choice of $h_0$.

As for the weak smoothness requirement on the loss, note that
\begin{align*}
\|\loss^{\prime}(h)-\loss^{\prime}(h_0)\| = \sup_{\|u\|=1} \langle u, \loss^{\prime}(h)-\loss^{\prime}(h_0) \rangle = \sup_{\|u\|=1} 2(\langle X, h-h_0 \rangle)\langle u, X\rangle \leq 2\|X\|^{2}\|h-h_0\|.
\end{align*}
Thus, if the random inputs $X$ are $\ddist$-almost surely bounded, the desired smoothness condition (\ref{eqn:loss_smooth}) holds. Note that this does not preclude heavy-tailed losses and gradients since no additional assumptions have been made on the noise term.
\end{ex}

\section{Empirical supplement}\label{sec:empirical_supp}

Due to limited space, we could only include key details and a few representative results in section \ref{sec:empirical} of the main text. Here we fill in those additional details. To begin, all our numerical experiments have been implemented entirely in Python (v.~3.8) using the following additional open-source software: Jupyter notebook (for interactive demos),\footnote{\url{https://jupyter.org/}} matplotlib (v.~3.4.1, for all visuals),\footnote{\url{https://matplotlib.org/}} PyTables (v.~3.6.1, for dataset handling),\footnote{\url{https://www.pytables.org/}} NumPy (v.~1.20.0, for almost all computations),\footnote{\url{https://numpy.org/}} and SciPy (v.~1.6.2, for random variable statistics and special functions). In the following paragraphs, we provide information about the benchmark datasets used, as well as several figures including additional experimental results.

\paragraph{Dataset description}

The real-world benchmark datasets used in our classification tests are as follows: \texttt{adult},\footnote{\url{https://archive.ics.uci.edu/ml/datasets/Adult}} \texttt{australian},\footnote{\url{https://archive.ics.uci.edu/ml/datasets/statlog+(australian+credit+approval)}} \texttt{cifar10},\footnote{\url{https://www.cs.toronto.edu/~kriz/cifar.html}} \texttt{cod\_rna},\footnote{\url{https://www.csie.ntu.edu.tw/~cjlin/libsvmtools/datasets/binary.html}} \texttt{covtype},\footnote{\url{https://archive.ics.uci.edu/ml/datasets/covertype}} \texttt{emnist\_balanced},\footnote{\url{https://www.nist.gov/itl/products-and-services/emnist-dataset}} \texttt{fashion\_mnist},\footnote{\url{https://github.com/zalandoresearch/fashion-mnist}} and \texttt{mnist}.\footnote{\url{http://yann.lecun.com/exdb/mnist/}} See Table \ref{table:datasets} for a summary. Further background on all datasets is available at the URLs provided in the footnotes. Dataset size reflects the size after removal of instances with missing values, where applicable. For all datasets with categorical features, the ``input features'' given in Table \ref{table:datasets} represents the number of features after doing a one-hot encoding of all such features. The ``model dimension'' is just the product of the number of input features and the number of classes, since we are using multi-class logistic regression (one linear model for each class). All categorical features are given a one-hot representation. All input features are standardized to take values on the unit interval $[0,1]$. As mentioned in the main text, we have prepared a GitHub repository that includes code for both re-creating the empirical tests and pre-processing the data, that will be made public following the review phase.

\begin{table}[t!]
\begin{center}
\begin{tabular}{|l|l|l|l|l|}
\hline
Dataset & Size & Input features & Number of classes & Model dimension \\
\hline\hline
\texttt{adult} & 45,222 & 105 & 2 & 210\\
\hline
\texttt{australian} & 690 & 43 & 2 & 86\\
\hline
\texttt{cifar10} & 60,000 & 3,072 & 10 & 30,720\\
\hline
\texttt{cod\_rna} & 331,152 & 8 & 2 & 16\\
\hline
\texttt{covtype} & 581,012 & 54 & 7 & 378\\
\hline
\texttt{emnist\_balanced} & 131,600 & 784 & 47 & 36,848\\
\hline
\texttt{fashion\_mnist} & 70,000 & 784 & 10 & 7,840\\
\hline
\texttt{mnist} & 70,000 & 784 & 10 & 7,840\\
\hline
\end{tabular}
\end{center}
\caption{A summary of the benchmark datasets used for performance evaluation.}
\label{table:datasets}
\end{table}

\paragraph{Additional results}

In Figures \ref{fig:real_data_supp_trajectories}--\ref{fig:real_data_supp_hist_2}, we give additional results that complement Figure \ref{fig:real_data} in the main text. The trends in terms of the histograms of test loss distributions are essentially uniform across this wide variety of datasets. We also see that a sharply-concentrated logistic (test) loss tends to correlate with better classification error (average zero-one error), with \texttt{cifar10} being the only exception to this trend. As another point not raised in the main text, intuitively we would hope that Algorithm \ref{algo:sgd} performs well in terms of the risk $\risk_{\sigma}$ corresponding to its particular $\sigma$ setting; we have found this to be true across the benchmark datasets studied here. See Figure \ref{fig:real_data_supp_levels} for an example from the \texttt{adult} dataset. Moving from top to bottom, the order of colors shows a rather clear reversal. Very similar trends can be observed on the other datasets as well. In estimating $\risk_{\sigma}$ on the test set, we use an empirical mean estimate of $\jrisk_{\sigma}$, and then minimize with respect to $\theta$ using the \texttt{minimize\_scalar} function of the SciPy (v.~1.6.2) \texttt{optimize} module.

\begin{figure}[t]
\centering
\includegraphics[width=0.7\textwidth]{legend_plot_test_values}\\
\includegraphics[width=0.25\textwidth]{avesmall_adult_linreg_multi_SGD_Ave_zero_one}\includegraphics[width=0.25\textwidth]{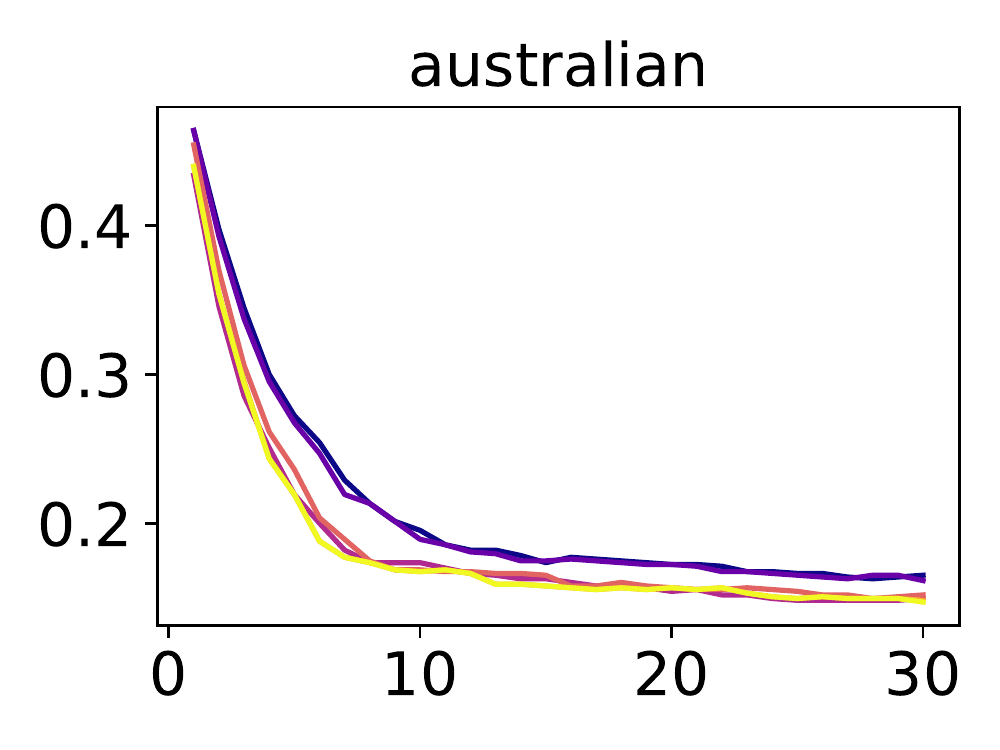}\includegraphics[width=0.25\textwidth]{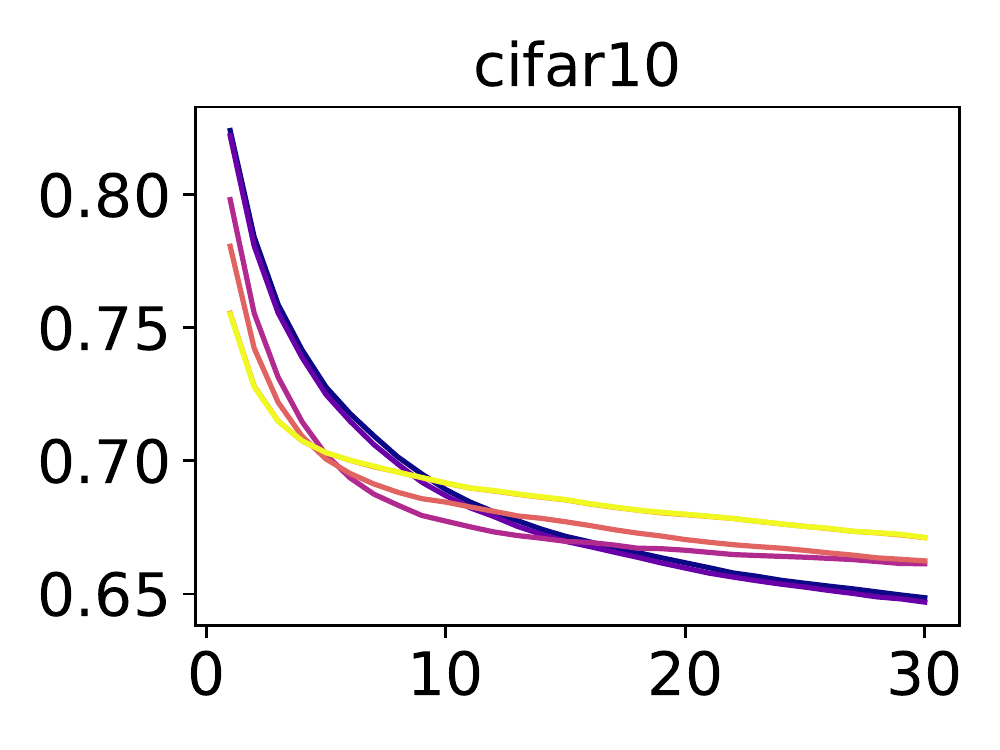}\includegraphics[width=0.25\textwidth]{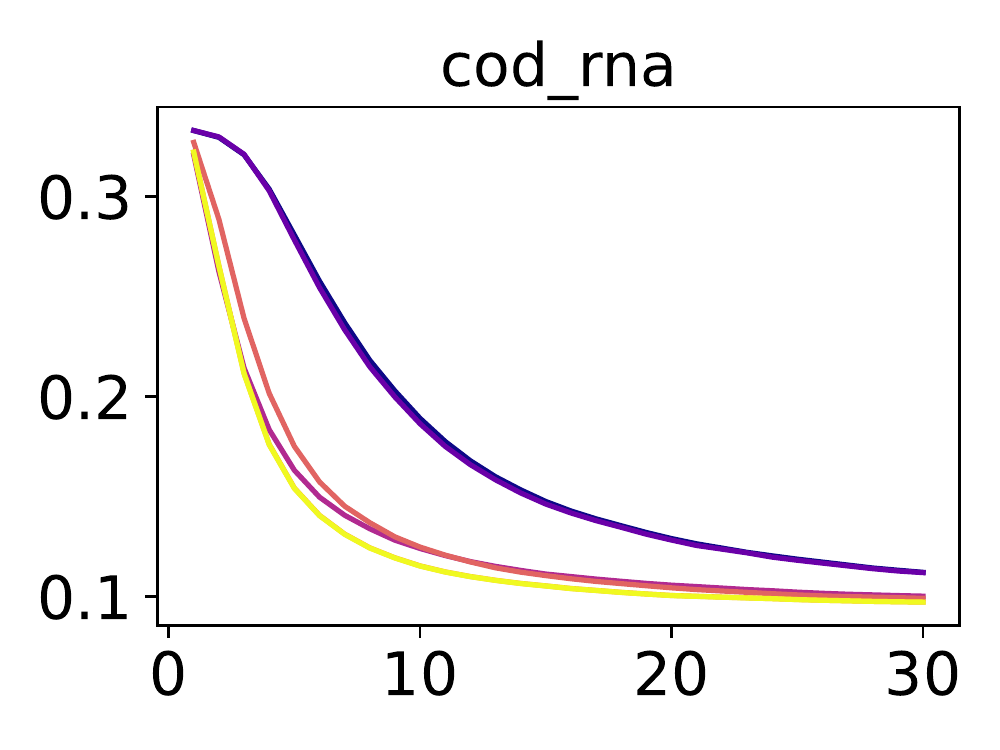}\\
\includegraphics[width=0.25\textwidth]{avesmall_covtype_linreg_multi_SGD_Ave_zero_one}\includegraphics[width=0.25\textwidth]{avesmall_emnist_balanced_linreg_multi_SGD_Ave_zero_one}\includegraphics[width=0.25\textwidth]{avesmall_fashion_mnist_linreg_multi_SGD_Ave_zero_one}\includegraphics[width=0.25\textwidth]{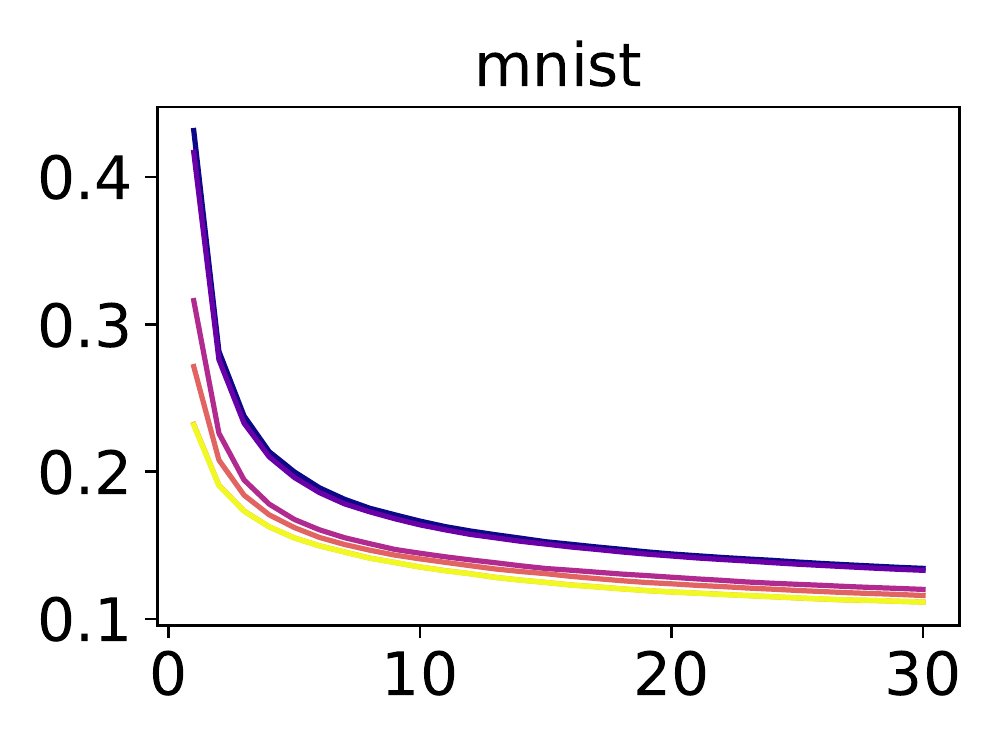}
\caption{Additional average classification error trajectories over epochs.}
\label{fig:real_data_supp_trajectories}
\end{figure}

\begin{figure}[t]
\centering
\includegraphics[width=1.0\textwidth]{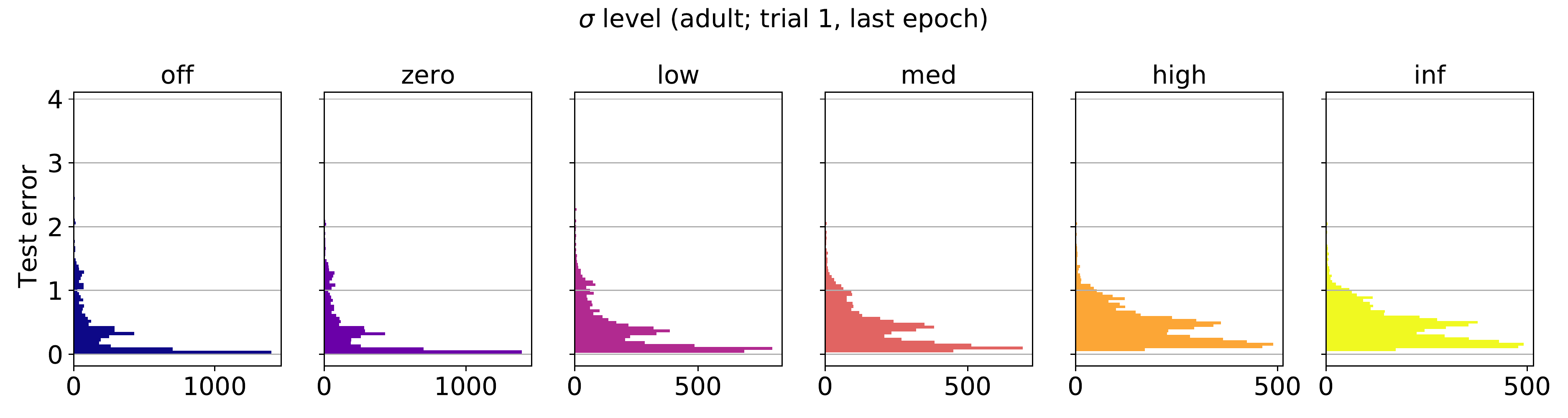}\\
\includegraphics[width=1.0\textwidth]{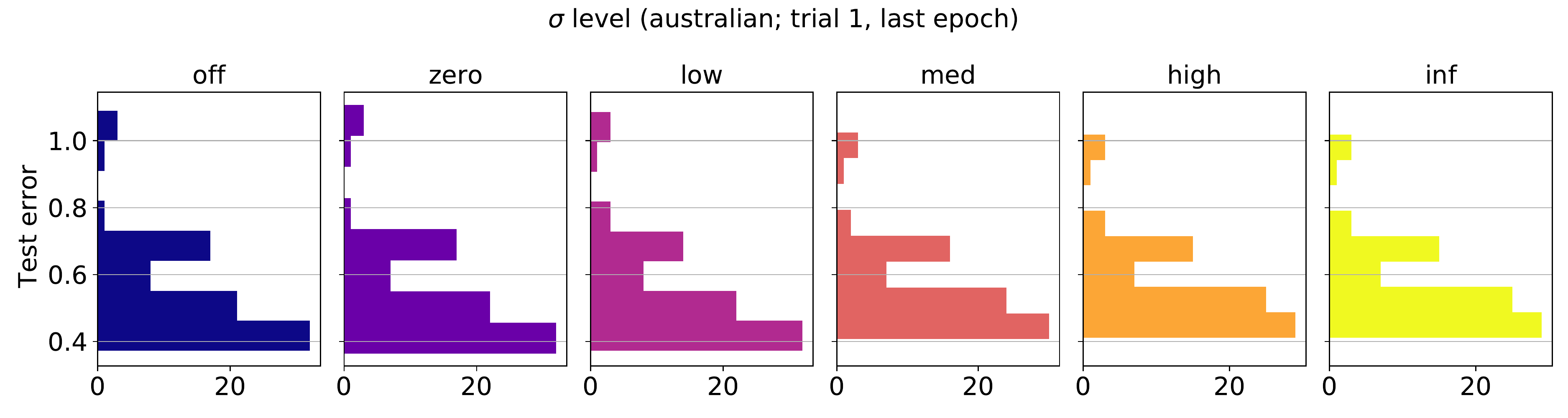}\\
\includegraphics[width=1.0\textwidth]{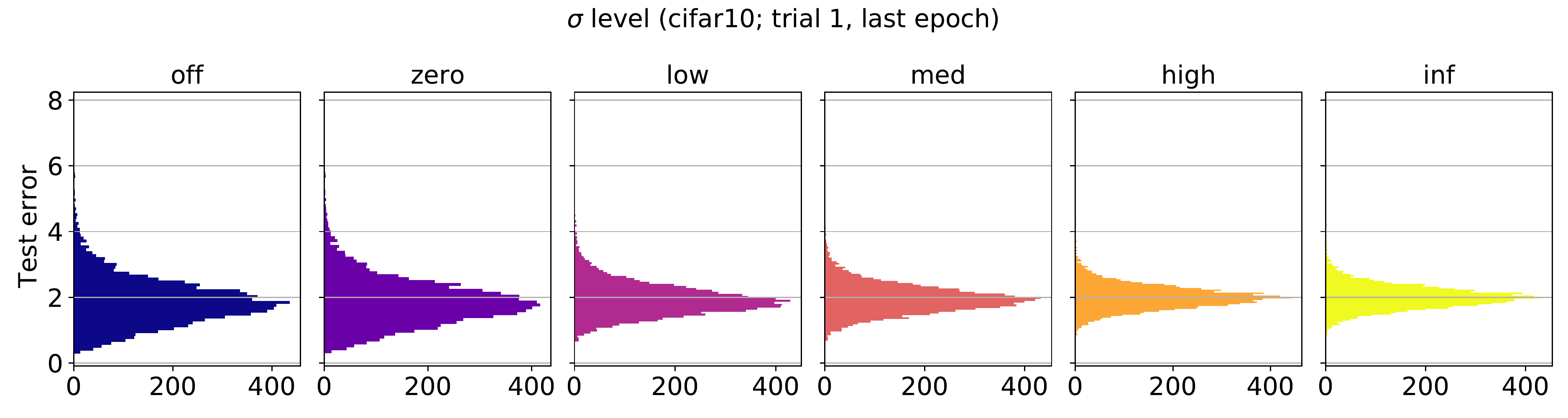}\\
\includegraphics[width=1.0\textwidth]{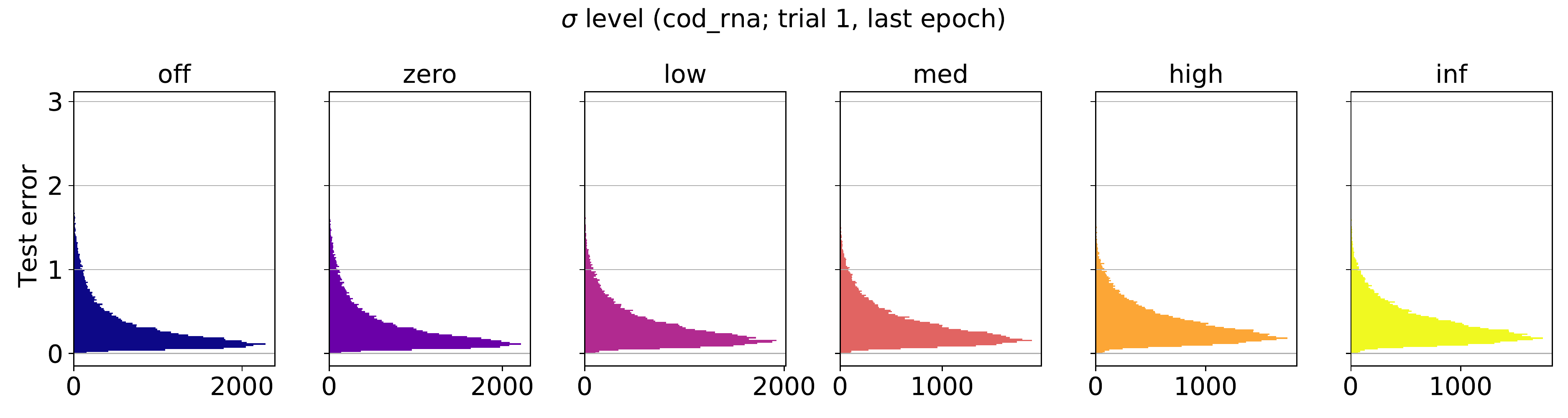}
\caption{Additional test error histograms.}
\label{fig:real_data_supp_hist_1}
\end{figure}

\begin{figure}[t]
\centering
\includegraphics[width=1.0\textwidth]{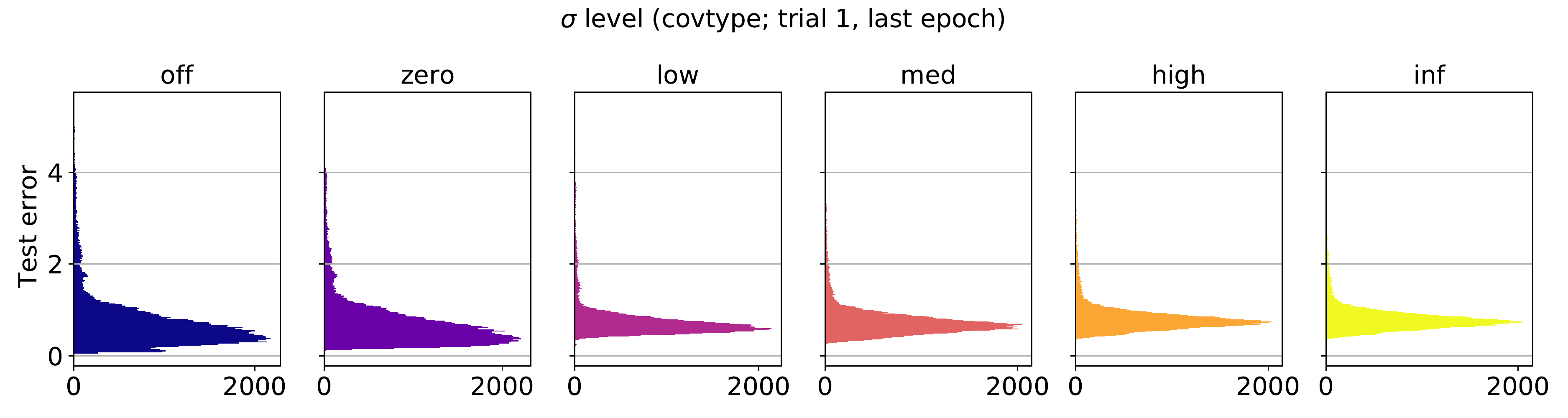}\\
\includegraphics[width=1.0\textwidth]{hist_emnist_balanced_linreg_multi_SGD_Ave_logistic}\\
\includegraphics[width=1.0\textwidth]{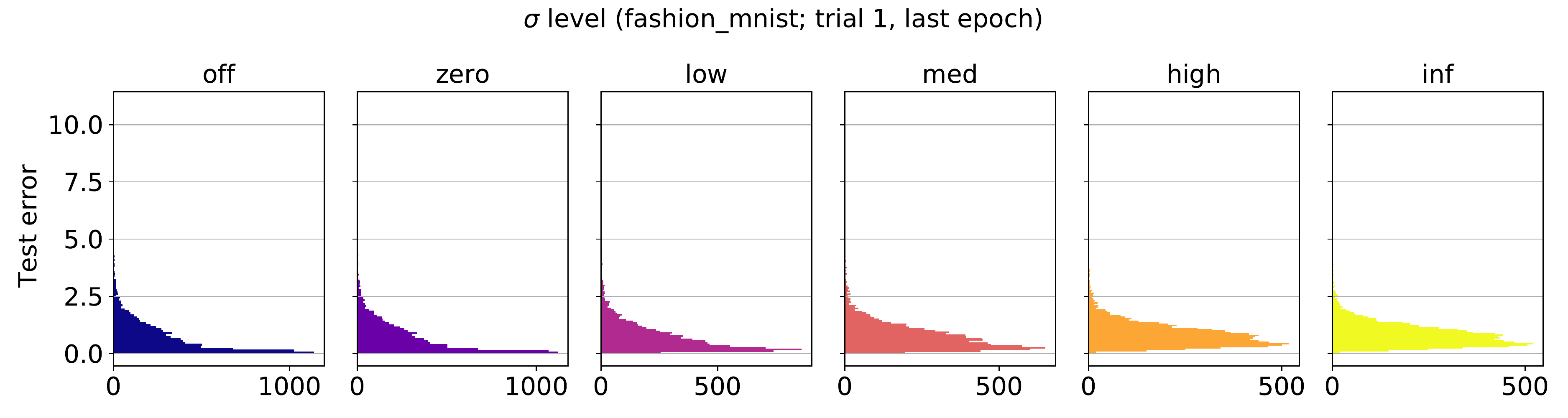}\\
\includegraphics[width=1.0\textwidth]{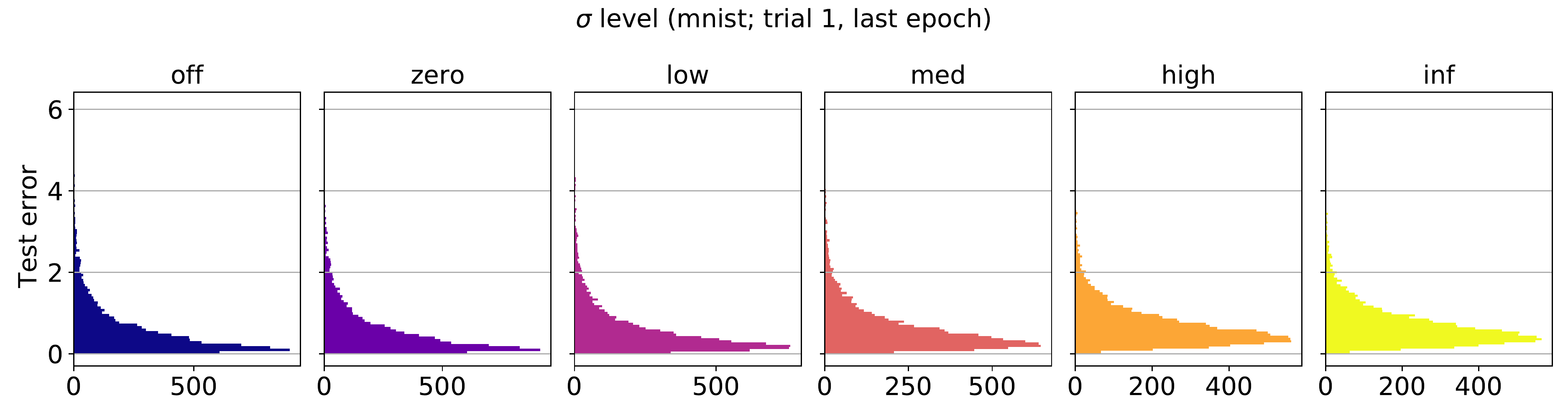}
\caption{Additional test error histograms.}
\label{fig:real_data_supp_hist_2}
\end{figure}

\begin{figure}[t]
\centering
\includegraphics[width=0.6\textwidth]{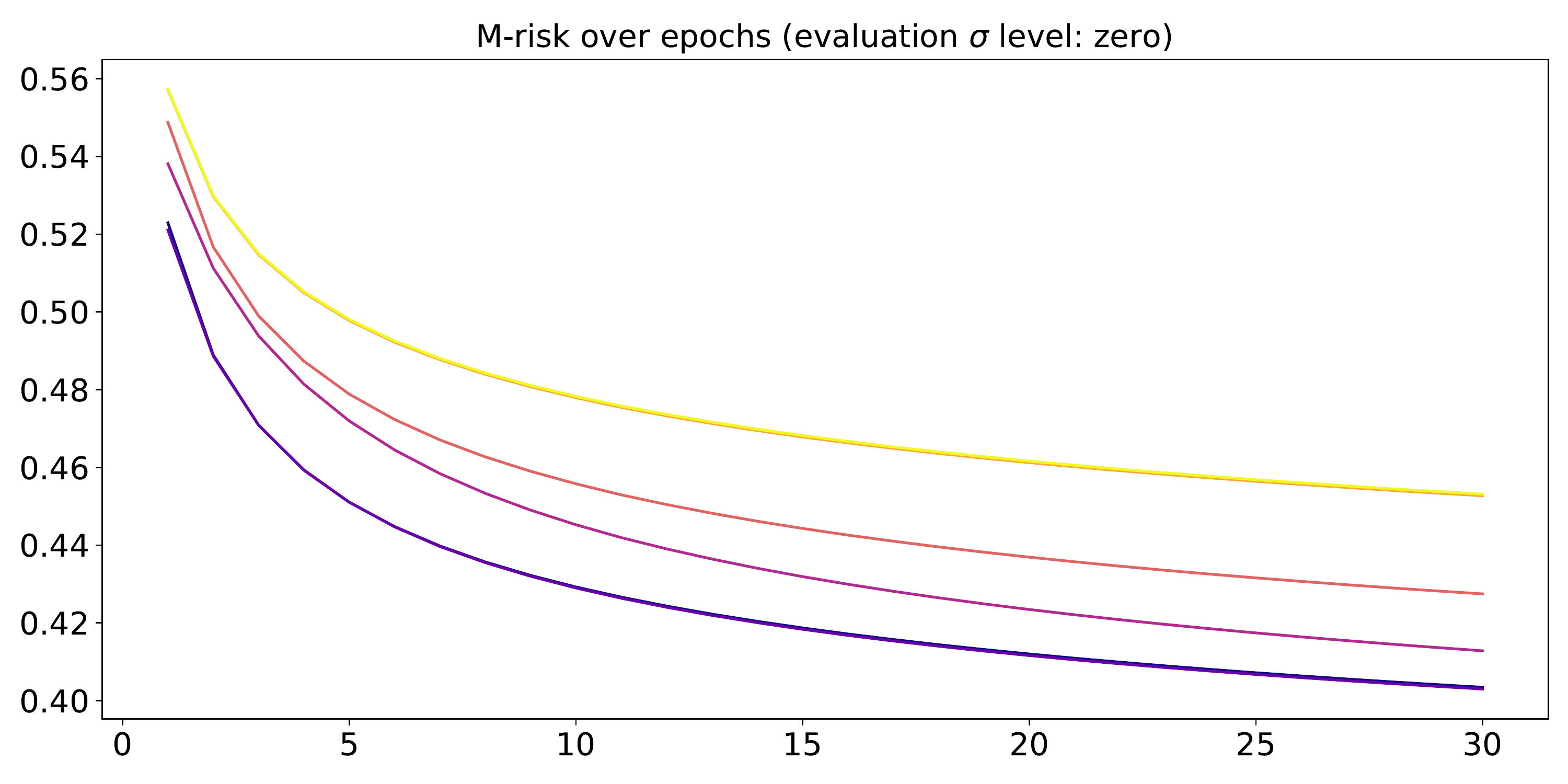}\\
\includegraphics[width=0.6\textwidth]{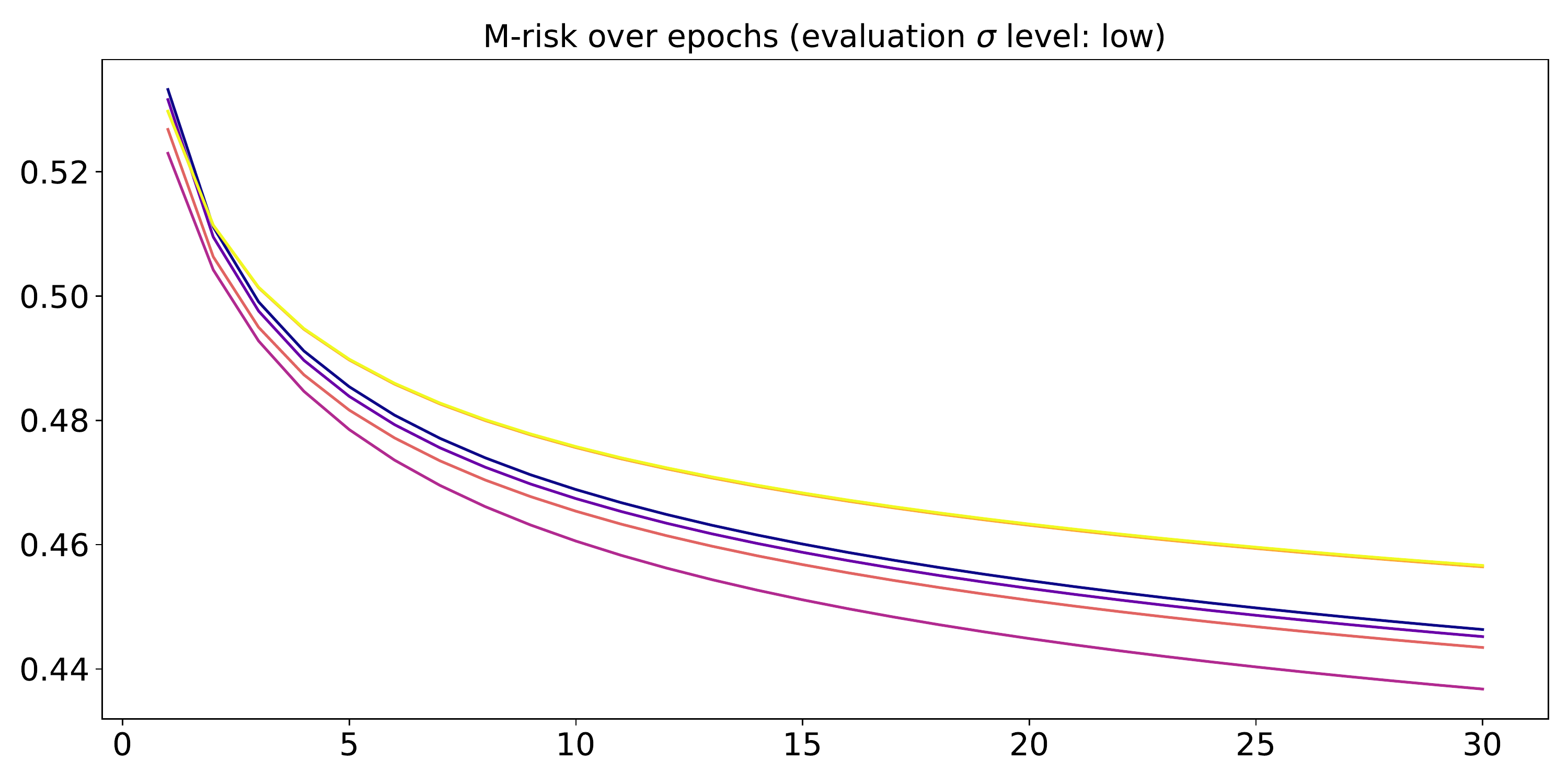}\\
\includegraphics[width=0.6\textwidth]{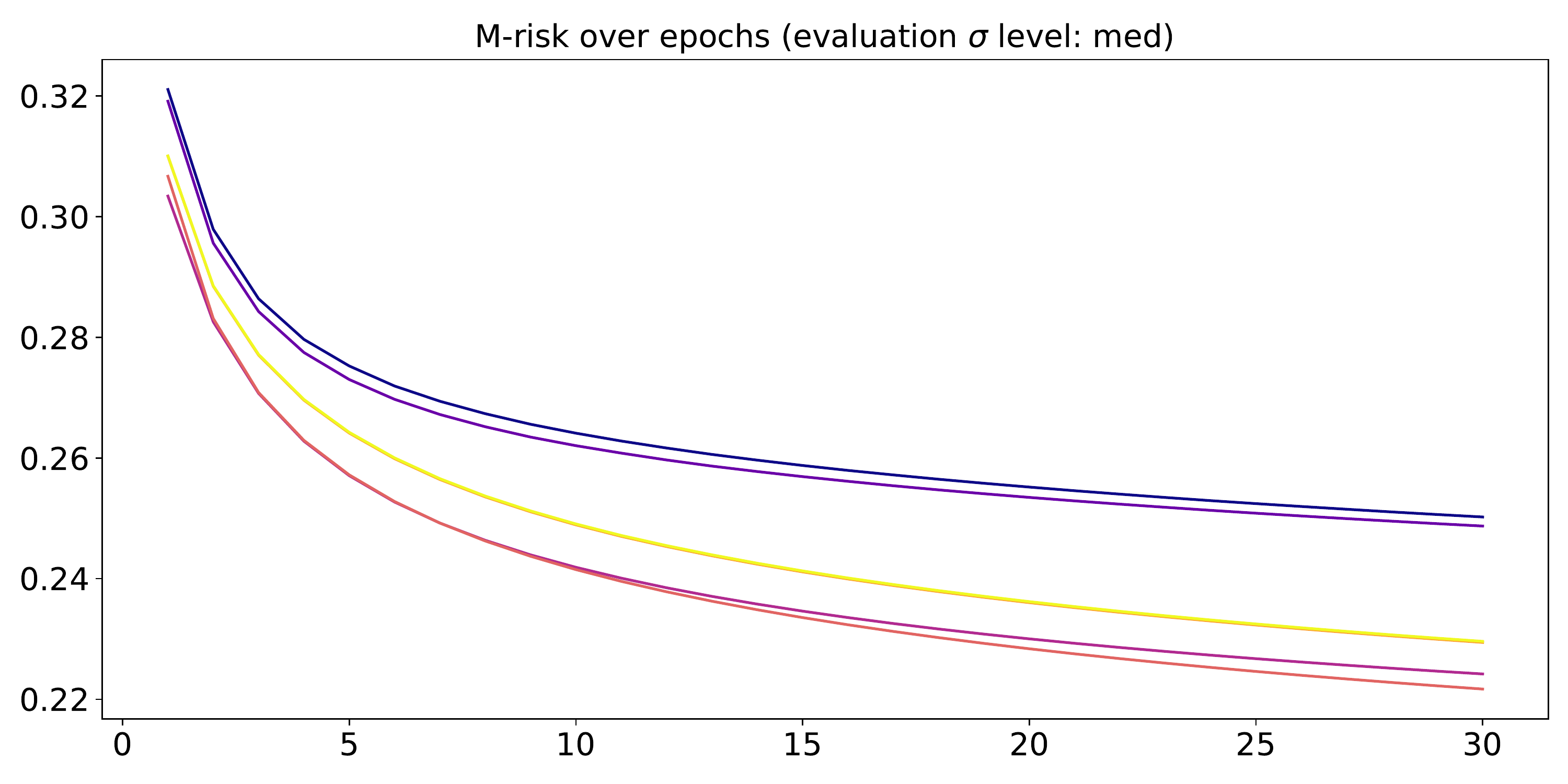}\\
\includegraphics[width=0.6\textwidth]{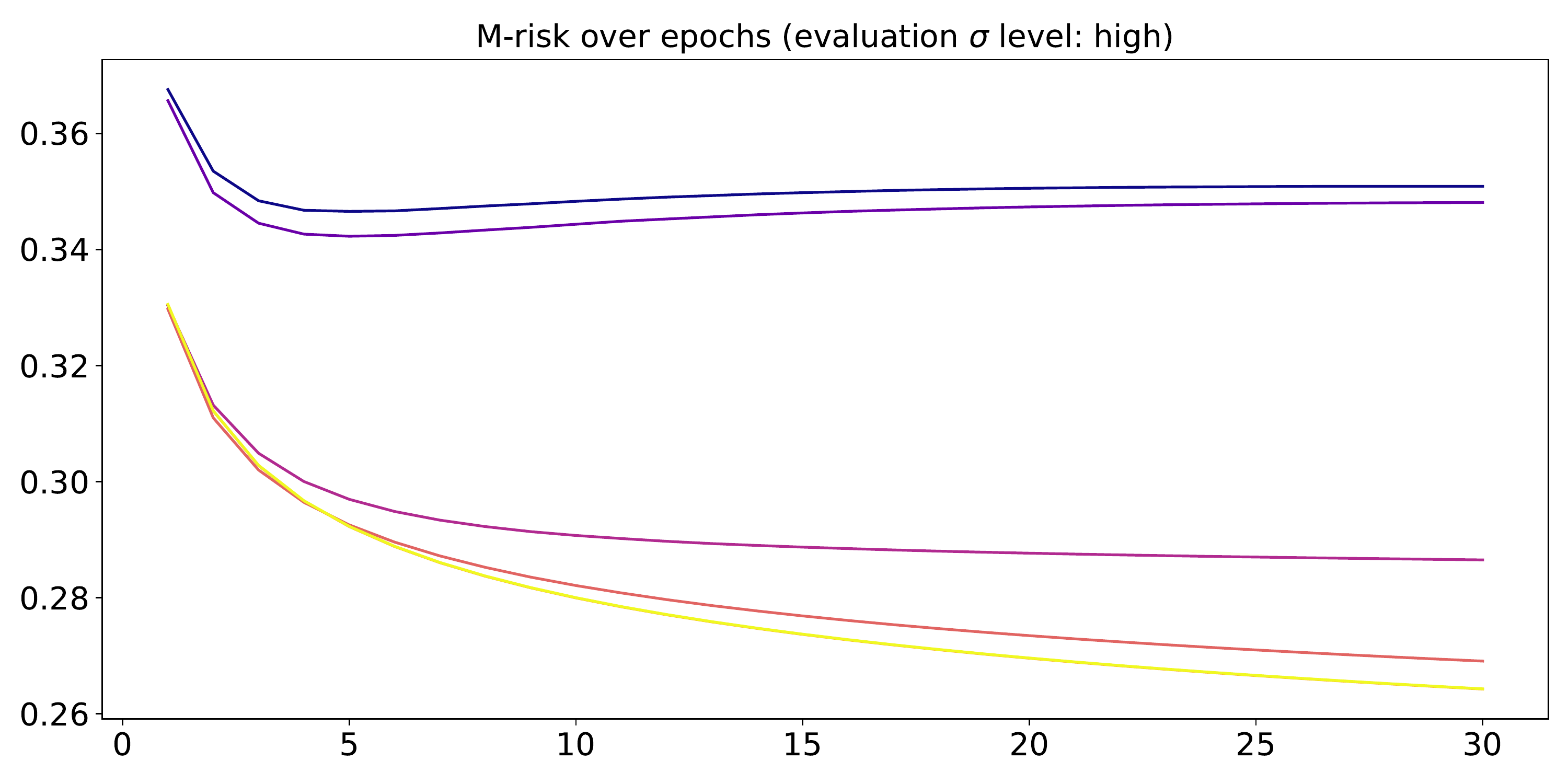}\\
\includegraphics[width=0.6\textwidth]{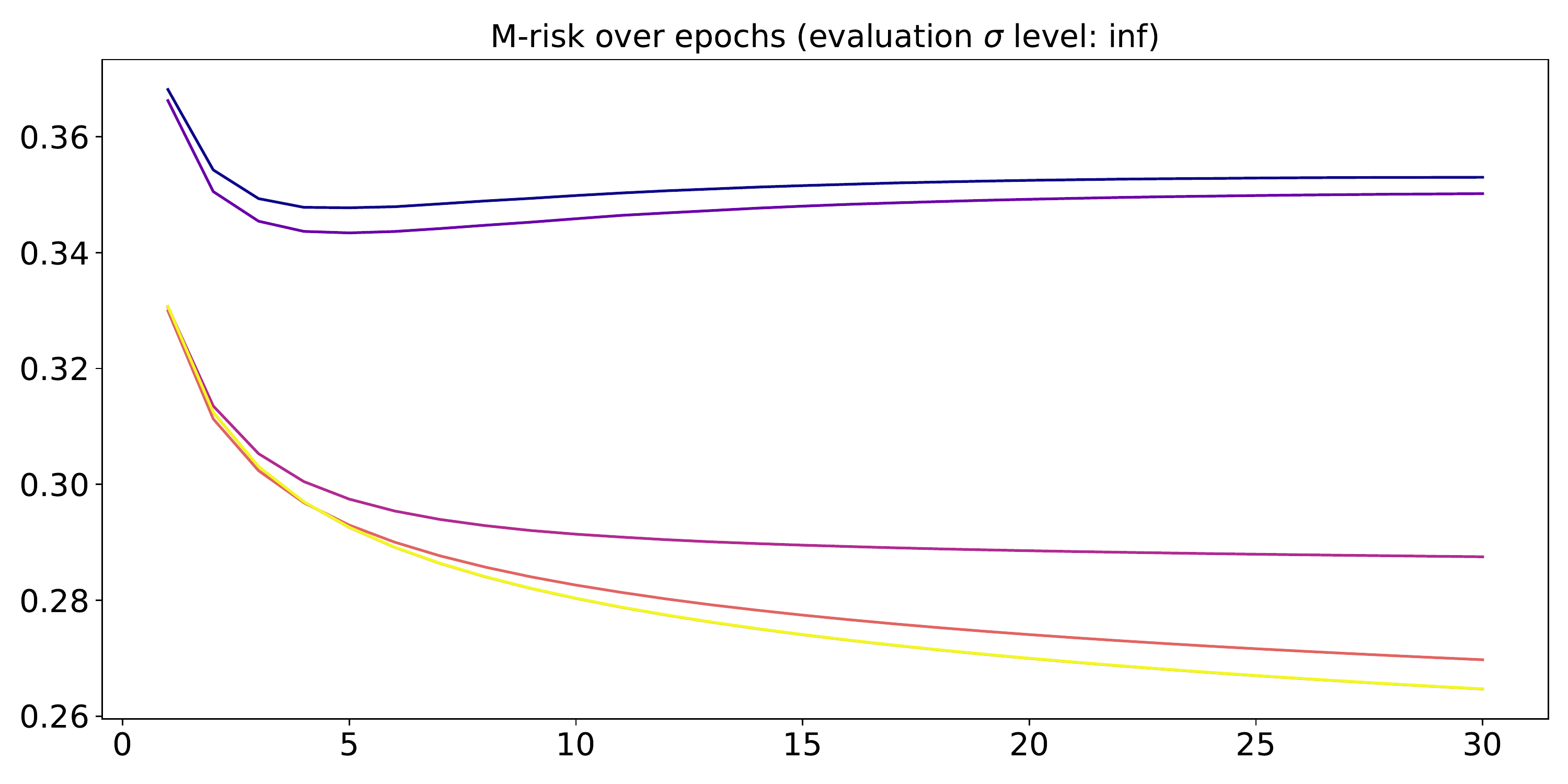}
\caption{Empirical mean estimates of $\risk_{\sigma}$ for a variety of $\sigma \in [0,\infty]$ settings. The colored curves correspond to different $\sigma$ settings in running Algorithm \ref{algo:sgd}, just as in previous plots, whereas the distinct plots correspond to the different $\sigma$ used in \emph{evaluation}.}
\label{fig:real_data_supp_levels}
\end{figure}

\bibliographystyle{../refs/apalike}
\bibliography{../refs/refs.bib}

\end{document}